\documentclass[review,onefignum,onetabnum]{siamart190516}

\usepackage{bm}        
\usepackage{lipsum}
\usepackage{amsfonts}
\usepackage{graphicx}
\usepackage{epstopdf}
\usepackage{tabularx}
\usepackage{booktabs}
\usepackage{comment}
\usepackage{subcaption} 
\usepackage{algpseudocode}
\usepackage{placeins}
\usepackage{xr-hyper}

\DeclareMathOperator{\Tr}{Tr}

\DeclareMathOperator{\Diag}{Diag}

\newcommand{\rmd}{\mathrm{d}}

\usepackage[utf8]{inputenc}
\usepackage[margin=1in]{geometry}
 
\usepackage[]{epsfig}

\usepackage{graphicx}
\usepackage{epstopdf}
\usepackage{comment}
\usepackage{array}
\usepackage{algorithm}
\usepackage{url}
\usepackage{bm}
\usepackage{dsfont}
\usepackage{bbm}

\usepackage{lscape}
\usepackage{algpseudocode}
\usepackage{setspace}
\usepackage{multicol}
\usepackage{multirow}
\usepackage{color}
\usepackage{colortbl}
\usepackage{xcolor}
\usepackage{mathtools}
\usepackage{amssymb}

\usepackage{placeins}

\usepackage[
     font={small,sf},
    labelfont=bf,
    format=hang,    
    format=plain,
    margin=0pt,
    width=0.8\textwidth,
]{caption}

\usepackage[list=true]{subcaption}

\usepackage{enumerate}
\usepackage{tikz}  
\usepackage{tikz-3dplot} 
\usepackage{mathtools}

\usepackage{xifthen}
\usepackage{graphicx}
\usepackage{caption,subcaption}

\usepackage{pgfplots}
\pgfplotsset{compat=newest}
\usetikzlibrary{plotmarks}
\usetikzlibrary{arrows.meta}
\usepgfplotslibrary{patchplots}
\usepackage{grffile}
\usepackage{amsmath}

\usepackage{epstopdf}

\newcommand{\widetildeC}{\widetilde{C}}
\newcommand{\widetildeL}{\widetilde{L}}
\newcommand{\widetildeK}{\widetilde{K}}

\newcommand{\calR}{\mathcal{R}}

\newcommand{\calC}{\mathcal{C}}
\newcommand{\bma}{\bm{\alpha}}
\newcommand{\bmb}{\bm{\beta}}
\newcommand{\bmx}{\bm{x}}
\newcommand{\bmu}{\bm{u}}
\newcommand{\bmxp}{\bm{x}'}
\newcommand{\sfy}{\mathsf{y}}
\newcommand{\sfx}{\mathsf{x}}
\newcommand{\sfz}{\mathsf{z}}
\newcommand{\sfu}{\mathsf{u}}
\newcommand{\sfe}{\mathsf{\epsilon}}
\newcommand{\Vol}{\mathrm{Vol}}

\newtheorem{example}{Example}

\ifpdf
  \DeclareGraphicsExtensions{.eps,.pdf,.png,.jpg}
\else
  \DeclareGraphicsExtensions{.eps}
\fi

\newsiamremark{remark}{Remark}
\newsiamremark{hypothesis}{Hypothesis}
\crefname{hypothesis}{Hypothesis}{Hypotheses}
\newsiamthm{claim}{Claim}

\headers{Implicit Regularization of Semi-parametric Models}{ M. Fanuel, J. Schreurs
and J.A.K. Suykens}

\title{Determinantal Point Processes Implicitly Regularize Semi-parametric Regression Problems\thanks{\funding{EU: The research leading to these results has received funding from the European Research Council under the European Union's Horizon 2020 research and innovation program / ERC Advanced Grant E-DUALITY (787960). This paper reflects only the authors' views and the Union is not liable for any use that may be made of the contained information.
Research Council KU Leuven: Optimization frameworks for deep kernel machines C14/18/068 Flemish Government: FWO: projects: GOA4917N (Deep Restricted Kernel Machines: Methods and Foundations), PhD/Postdoc grant
This research received funding from the Flemish Government (AI Research Program). Ford KU Leuven Research Alliance Project KUL0076 (Stability analysis and performance improvement of deep reinforcement learning algorithms) EU H2020 ICT-48 Network TAILOR (Foundations of Trustworthy AI - Integrating Reasoning, Learning and Optimization). Leuven.AI Institute.}}}

\author{Micha\"el Fanuel\thanks{KU Leuven, Department of Electrical Engineering (ESAT),
STADIUS Center for Dynamical Systems, Signal Processing and Data Analytics,
Kasteelpark Arenberg 10, B-3001 Leuven, Belgium. email: michael.fanuel@kuleuven.be}
\and Joachim Schreurs\footnotemark[2]
\and Johan A.K. Suykens\footnotemark[2]}

\usepackage{amsopn}
\DeclareMathOperator{\diag}{diag}

\ifpdf
\hypersetup{
  pdftitle={Implicit Regularization of Semi-parametric Models},
  pdfauthor={M. Fanuel, J. Schreurs, J.A.K. Suykens}
}
\fi

\begin{document}
\begin{nolinenumbers}
\maketitle

\begin{abstract}
Semi-parametric regression models are used in several applications which require comprehensibility without sacrificing accuracy. Typical examples are
    spline interpolation in geophysics, or non-linear time series problems, where the system includes a linear and non-linear component. We discuss here the use of a finite Determinantal Point Process (DPP) for approximating semi-parametric models.
    Recently, Barthelm\'e,  Tremblay,  Usevich,  and Amblard introduced a novel representation of some finite DPPs. These authors formulated \emph{extended $L$-ensembles} that can conveniently represent partial-projection DPPs and suggest their use for optimal interpolation. With the help of this formalism, we derive a key identity illustrating the implicit regularization effect of determinantal sampling for semi-parametric regression and interpolation. Also, a novel \emph{projected} Nystr\"om approximation is defined and used to derive a bound on the expected risk for the corresponding approximation of semi-parametric regression. This work naturally extends similar results obtained for kernel ridge regression.
\end{abstract}



\section{Introduction}
Kernel methods provide a theoretically grounded framework for non-parametric regression and have been able to achieve excellent performance~\cite{BLESS,rudi2015less} in the last years. In applications that require more explainability, a parametric component, usually a polynomial, is added to the kernel regressor. This semi-parametric model has the best of both worlds, a parametric component that is understandable for the user and a non-parametric kernel component that boosts the accuracy of the prediction. Full-size kernel regression problems do no scale well with the size of data sets, for that reason several approximations have been studied. In particular, in the case of massive data sets, smart sampling and sketching methods have allowed to scale up kernel ridge regression~\cite{meanti2020kernel}, while preserving its statistical guarantees. Not only to reduce memory requirements, sampling methods are interesting to reduce the number of parameters of such models for enhancing prediction speed, for instance, in the context of embedded applications. In this paper, we consider the specific setting of semi-parametric regression which generalizes and improves the interpretability of kernel ridge regression for the applications where a parametric (e.g., polynomial) estimator can be an educated guess. We combine this semi-parametric approach with a custom sampling scheme based on Determinantal Point Processes, thereby allowing to obtain subsets of important and diverse points. This leads to similar results to the ones obtained for kernel ridge regression~\cite{fanuel2020diversity}, that is to say, DPP sampling implicitly regularizes (semi-parametric) kernel regression problems. 

\paragraph{Sampling with a determinantal point process} Discrete Determinantal Point Processes (DPPs) provide elegant ways to sample  random subsets $\calC\subseteq \{1,\dots , n\}$, sometimes called `coresets'~\cite{JMLR:v20:18-167}, so that the selected items are diverse. In a word, discrete DPPs are represented by a marginal kernel, that is, a $n\times n$ matrix $P$ with eigenvalues within $[0,1]$, giving the inclusion probabilities: if $\calC$ is a random subset distributed according to a DPP with marginal kernel $P$, then the inclusion probabilities are
\[
\Pr\left(\mathcal{E}\subseteq\calC \right) = \det ( P_{\mathcal{E}\mathcal{E}}),
\]
where $P_{\mathcal{E}\mathcal{E}}$ is the square submatrix obtained by selecting the rows and columns of $P$ indexed by $\mathcal{E}$. The off-diagonal entries of the marginal kernel are interpreted as similarity scores. Thus, a subset with a large probability is diverse. This can intuitively be seen thanks to the interpretation of the determinant of a positive definite matrix in terms of squared volume. In general, the expression of the probability for sampling a given subset $\Pr(\calC)$ is known but non trivial. Therefore, it is often very convenient to work with a $L$-ensemble, which is a DPP, denoted here by $ DPP_L(L)$, such that
\begin{equation}
    \Pr(\calC) = \det ( L_{\calC\calC})/\det(\mathbb{I}+L),\label{eq:Lensemble}
\end{equation}
 where $L$ is a $n\times n$ positive semi-definite matrix. The marginal kernel of an $L$-ensembles has the following simple expression
\[
P = L(L+\mathbb{I})^{-1},
\]
which is a matrix encountered in kernel rigde regression as we explain hereafter.
In the context of \emph{sketched} kernel ridge regression, sampling with $L$-ensemble DPPs yields very simple theoretical guarantees displaying an implicit regularization effect~\cite{fanuel2020diversity}. Before discussing semi-parametric regression, we briefly outline the known results about kernel ridge regression.

\paragraph{Sketching Kernel Ridge Regression}  Given input-output pairs $(\bm{x}_i,y_i)\in \mathbb{R}^d\times \mathbb{R}$ for $1\leq i\leq n$, kernel ridge regression (KRR) estimates a function of the form $f(\bmx) = \sum_{i=1}^n \alpha_i k(\bmx,\bmx_i)$ with the help of a positive semi-definite kernel function $k(\bmx, \bmxp)$ defined on $\mathbb{R}^d\times \mathbb{R}^d$, such as the Gaussian kernel $k(\bmx,\bmx') = \exp(-\|\bmx-\bmx'\|_2^2/\sigma^2)$.
Classically, the numerical solution of KRR relies on a $n\times n$ positive semi-definite kernel matrix
\[
K = [k(\bmx_i,\bmx_j)]_{1\leq i,j\leq n},
\]
constructed from the input data. Such a matrix can be potentially large if the size of the data set is large. Therefore, low rank approximations of $K$ have been developed~\cite{williams2001using,rudi2015less,fanuel2020diversity,fanuel2019nystr}.  In particular, an $L$-ensemble DPP can be used to select subsets $\calC =  \{c_1,\dots , c_{k}\}$ of $\{1,\dots , n\}$ in order to sample a subset of entries of $K$.  In this context, a natural choice is $L = K/\lambda$ for some $\lambda>0$. Then, it is customary to approximate $K$ thanks to the low rank Nystr\"om method which uses its submatrices, such as the square $k\times k$ submatrix $K_{\calC\calC}$. Sampling is conveniently done with the use of a $n\times k$ sampling matrix, that is obtained by selecting the columns of the identity matrix indexed by $\calC$ as follows
$
C = (\bm{e}_{c_1}  \cdots  \bm{e}_{c_k}),
$
where $\bm{e}_i$ denotes the $i$-th element of the canonical basis. In the case of the Nystr\"om approximation, one considers the pseudo-inverse of the \emph{sparse} matrix $C K_{\calC\calC} C^\top$ (see below), which is a $n\times n$ matrix whose entry $(i,j)$ is $K_{ij}$ if $i,j\in \calC$ and zero otherwise.  Explicitly, the \emph{common} Nystr\"om approximation of $K$ is defined as 
\begin{align}
K(C K_{\calC\calC} C^\top)^+ K = KC K_{\calC\calC}^{+} C^\top K, \tag{\text{Nystr\"om}}
\end{align}
where $(\cdot)^{+}$ denotes the Moore-Penrose pseudo-inverse\footnote{See Lemma~\ref{lem:pinv} in Appendix for a formal statement concerning the pseudo-inverse of this kind of matrices.}.
 This subsampling with $L$-ensembles has an implicit regularization effect~\cite{fanuel2020diversity}, which is based on the following expectation formula, also independently shown by~\cite{Mutny}:
\begin{equation}
    \mathbb{E}_\calC\left(CK_{\calC\calC}C^\top\right)^+= (K+\lambda\mathbb{I})^{-1},\label{eq:Exppinv}
\end{equation}
where $\calC\sim DPP_L(L)$  with $L = K/\lambda$ and $\lambda>0$. Varying $\lambda$ allows to vary the expected size of the subset and the amount of regularization. A similar identity has been revisited in the context of fixed-size $L$-ensemble DPP in~\cite{Schreurs2020diversity}. Albeit the exact sampling of a $L$-ensemble DPP has a time complexity $\mathcal{O}(n^3)$, the obtained expected error for Nystr\"om approximation $\mathbb{E}[K-KC K_{\calC\calC}^{+} C^\top K] = \lambda K(K+\lambda\mathbb{I})^{-1}$ for $\calC\sim DPP_L(L)$ provides a generalization of error bounds obtained with Ridge Leverage Score sampling~\cite{ElAlaouiMahoney,MuscoMusco,BLESS}. To the best of our knowledge, such results have not been obtained for semiparametric regression problems generalizing KRR.

\paragraph{Partial projection DPPs and semi-parametric regression} There are DPPs which are not $L$-ensembles, for instance, the projection DPPs for which the marginal kernel is a projector, i.e., a symmetric matrix such that $P^2= P$. In practice, it is often convenient to have a simple formula for the probability that a subset is sampled, i.e., $\Pr(\calC)$. Therefore, the elegant framework of `extended $L$-ensembles' has been introduced in~\cite{Barthelme2020} which provides a handy formula generalizing~\eqref{eq:Lensemble}. This formalism is used extensively in this paper to deal with partial-projection DPPs. In Layman's Terms, both partial projection DPPs and semi-parametric regression rely on mathematical expressions involving a sum of two objects living in orthogonal subspaces. This analogy is the main motivation to consider approximations of semi-parametric regression models with partial projection DPPs.
For a partial-projection DPP, denoted by $DPP(L,V)$, the marginal kernel is of the following form
\begin{equation}
    P = \mathbb{P}_V + \widetildeL (\widetildeL +\mathbb{I})^{-1},\label{eq:marginalKernel}
\end{equation}
where the matrix $V$ is a $n\times p$ matrix with full column rank and $\mathbb{P}_V = V (V^\top V)^{-1} V^\top$ is the projection on its column space, while
the matrix $\widetildeL = \widetildeK/\lambda$ with $\lambda>0$ is defined thanks to the $n\times n$ projected kernel 
\[
\widetildeK \triangleq \mathbb{P}_{V^\perp} K \mathbb{P}_{V^\perp} \text{ with } \mathbb{P}_{V^\perp} = \mathbb{I} - \mathbb{P}_V.
\]
In what follows, such a projected quantity is denoted by using a tilde. It is natural to assume $\widetildeL$ to be positive semi-definite so that the inverse matrix in~\eqref{eq:marginalKernel} is well-defined.

Extended $L$-ensembles, that we describe below, represent  partial-projection DPPs (see~\eqref{eq:marginalKernel}) by giving a convenient formula for $\Pr(\calC)$.
The reference~\cite{Barthelme2020} also points out a connection between extended $L$-ensembles and optimal interpolation in Section 2.8.2. This remark has motivated the following case study: the Nystr\"om approximation of semi-parametric regression problem.  
The problem consists in recovering a function from noisy function values   
\[y_i = z_i + \epsilon_i \quad \text{ with }  z_i = f(\bmx_i),\text{ and } 1\leq i\leq n, \]
where $\epsilon_i$ denotes i.i.d. $\mathcal{N}(0,\sigma^2)$ noise. 
We consider here the semi-parametric model (see Figure~\ref{fig:toy_semi} for an illustration)
\[
f(\bmx) = \sum_{i=1}^n \alpha_i k(\bmx,\bmx_i) + \sum_{m=1}^p \beta_m p_m(\bmx),
\]
where the first term is the non-parametric component associated to a \emph{conditionally} positive semi-definite kernel $k(\bmx, \bmxp)$, while the second term is the parametric component that is typically given by polynomials. The estimation problem amounts to solve
$
    \hat{f} = \arg\min_{f}\frac{1}{n} \sum_{i=1}^n \left(y_i -f(\bmx_i)\right)^2 +\gamma J(f),
$
for a suitable penalization functional $J(f)$ defined in Section~\ref{sec:thin-plate}, and given $\gamma>0$.
Then, the marginal kernel~\eqref{eq:marginalKernel} appears interestingly in the known formula for the in-sample estimate of a semi-parametric $\gamma$-regularized least squares problem,
\begin{equation*}
    \hat{\bm{z}} = P \bm{y} \quad \text{ with } \quad P = \mathbb{P}_V + \widetildeL (\widetildeL +\mathbb{I})^{-1},
\end{equation*}
with $\widetildeL = \widetildeK/(n\gamma)$ and where $\hat{z}_i$ denotes the estimated function value $z_i = f(\bmx_i)$ for $1\leq i \leq n$ and $\gamma>0$ is a regularization parameter.
In this setting, the $V$ and $K$ matrices used in~\eqref{eq:marginalKernel} are identified with the following matrices obtained from the parametric and non-parametric components:  
\[
V = [p_m(\bmx_i)]_{1\leq i\leq n, 1\leq m\leq p} \text{ and } K = [k(\bmx_i,\bmx_j)]_{1\leq i,j\leq n}. \]
\begin{figure}[h]
		\centering
	\begin{subfigure}[b]{0.4\textwidth}
			\includegraphics[width=\textwidth, height= 0.8\textwidth]{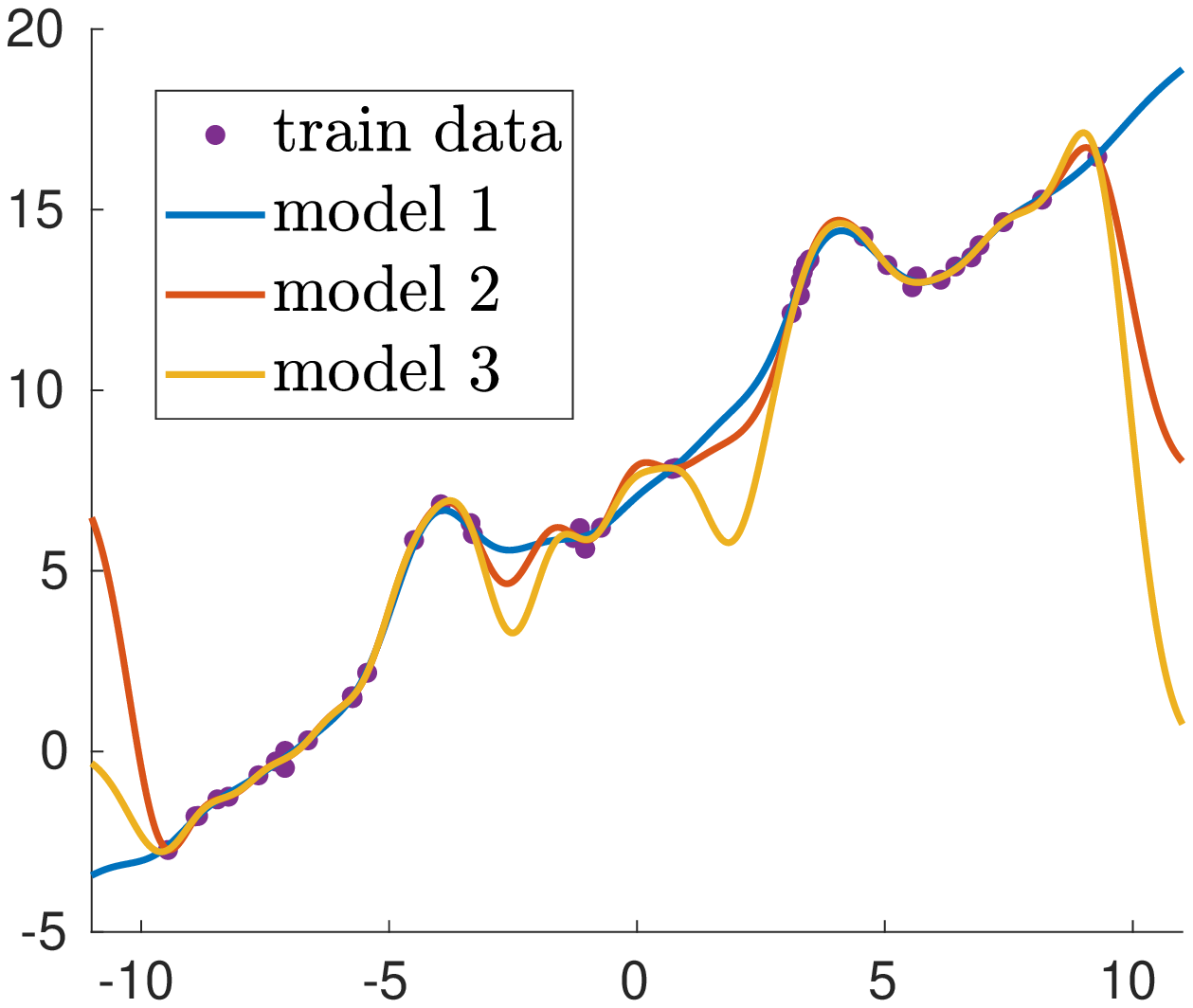}
			\caption{Function prediction}
			\label{fig:toy_semi_1}
		\end{subfigure}
		\begin{subfigure}[b]{0.4\textwidth}
			\includegraphics[width=\textwidth, height= 0.8\textwidth]{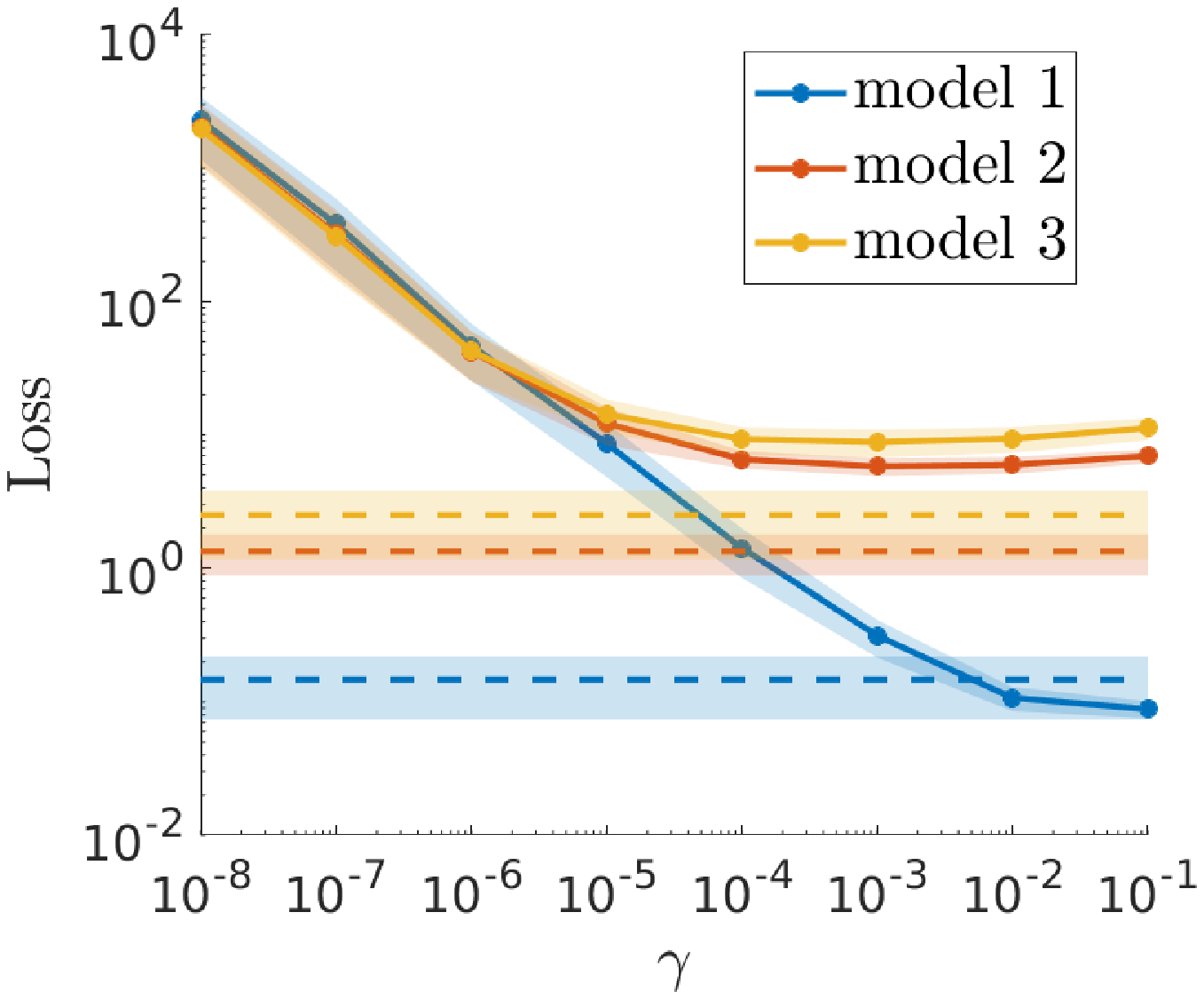}
			\caption{Expected loss (MSE)}
			\label{fig:toy_semi_2}
		\end{subfigure}
	\caption{A Toy example of semi-parametric regression with a Gaussian kernel ($d=1$). Figure \ref{fig:toy_semi_1} shows the training points and the estimated function with $\sigma = 1$ and best performing $\gamma$. Only the semi-parametric model, i.e. \texttt{model 1}: $\hat{f}(x) = \beta_1 + \beta_2 x + \sum_i \alpha_i k(x,x_i)$, predicts the linear trend in low density regions as well outside the training interval $[-10,10]$, contrary to \texttt{model 2} (LS-SVM, i.e. $\beta_2 =0$) and \texttt{model 3} (KRR, i.e. $\beta_1 = \beta_2 =0$). The MSE of each model with bandwidth $\sigma = 1$ is visualized as a function of the regularization parameter in Figure \ref{fig:toy_semi_2}. The dashed line shows the best performance when cross-validating over both $\gamma$ and $\sigma$. See Section~\ref{sec:Gaussian+Poly} for more details. }\label{fig:toy_semi}
\end{figure}
We now outline the contributions of this paper.
\subsection{Contributions}
\paragraph{Sketched semi-parametric regression}
First and foremost, a \emph{key} contribution of this paper is a formula analogous to~\eqref{eq:Exppinv} involving a sampling with a custom partial-projection DPP. As it is explained above, full-fledged semi-parametric regression involves two orthogonal components, associated to the matrices $V$ and $\mathbb{P}_{V^\perp} K \mathbb{P}_{V^\perp}$ respectively. To preserve this orthogonal decomposition for the sketched problem, we begin by defining \emph{a sampling of the rows} of the matrix $V$--which stores the non-parametric component of the regression problem--and a sampling of rows and columns of $K$ as follows
\[
V_\calC = [p_m(\bm{x}_i)]_{i\in\calC, 1\leq m\leq p} \text{ and } K_{\calC\calC} = [K_{ij}]_{ i,j\in \calC}.
\]
We analyse a sketched regression problem which is constructed so that the orthogonality between the parametric and non-parametric components is preserved. A key ingredient to analyse this problem is the projector $\mathbb{P}_{V^\perp_\calC}$ onto the orthogonal of the column space of $V_\calC = C^\top V$. 
Then, we address the following question:
\begin{center}
{\it 
What is the relationship between the sketched regression problem associated to $V_\calC$ and $\mathbb{P}_{V^\perp_\calC} K_{\calC\calC} \mathbb{P}_{V^\perp_\calC}$, and the full regression problem associated to $V$ and $\mathbb{P}_{V^\perp} K \mathbb{P}_{V^\perp}$?
}
\end{center}
Our main result, Theorem~\ref{thm:implicit_reg} given hereafter, implies the following identity for the expectation of the pseudo-inverse\footnote{Technically, we require here that $K$ is conditionally positive semi-definite definite with respect to $V$, i.e., $\mathbb{P}_{V^\perp} K \mathbb{P}_{V^\perp}$ is positive semi-definite.}
\begin{equation}
    \mathbb{E}_\calC\left(C \mathbb{P}_{V^\perp_\calC} K_{\calC\calC} \mathbb{P}_{V^\perp_\calC} C^\top\right)^+= \mathbb{P}_{V^\perp}(\widetildeK + \lambda \mathbb{I})^{-1}\mathbb{P}_{V^\perp}\quad \text{ with }\quad \widetildeK = \mathbb{P}_{V^\perp} K \mathbb{P}_{V^\perp},\label{eq:ExtendedExpectedpinv}
\end{equation}
where $\calC$ is sampled according to the partial-projection $DPP(K/\lambda,V)$. We emphasize that there is no trivial connection between $\mathbb{P}_{V^\perp}$ and $\mathbb{P}_{V^\perp_\calC}$, which are  the projectors onto the orthogonal of the column spaces of $V$ and $V_\calC$ respectively. 
The implicit regularization in~\eqref{eq:ExtendedExpectedpinv} is `conditional' since it occurs only within the subspace orthogonal to $V$. As in the case of $L$-ensembles, the real number $\lambda>0$ influences the expected subset size and the amount of regularization.
The identity~\eqref{eq:ExtendedExpectedpinv} is novel to the best of our knowledge and is also instrumental to derive two key contributions of this paper.
\paragraph{Projected Nystr\"om approximation} In Section~\ref{sec:LargeScaleRegression}, we define a projected Nystr\"om approximation $\widetilde{L(\mathcal{C})}$ of the projected kernel matrix  $\widetildeK$ under the assumption that $\widetildeK$ is positive semi-definite:
    \begin{align*}
                \widetilde{L(\mathcal{C})} \triangleq \widetildeK S(\calC)\left(S(\calC)^\top \widetilde{K}S(\calC)\right)^{+} S(\calC)^\top \widetildeK, \tag{\text{Projected Nystr\"om}}
    \end{align*}
    where the sketching matrix is $S(\calC) = CB(\calC)\in \mathbb{R}^{n\times (k-p)}$,
    with $B(\calC)\in \mathbb{R}^{k\times (k-p)}$ a matrix whose columns are an orthonormal basis of the orthogonal of the column space of $V_\calC$. The projected Nystr\"om  naturally extends the common Nystr\"om approximation to semi-parametric regression problems and is essential for scaling the model to larger data sets. The low rank approximation of the projected kernel matrix can be constructed conveniently with submatrices of the original kernel $K$. Indeed, it is not necessary to construct explicitly the sketching matrix $S(\calC)$. In comparison with the common Nystr\"om approximation, the sketching matrix involves here a projection since $B(\calC)B(\calC)^\top = \mathbb{P}_{V^\perp_\calC}$. Importantly, we give an expected error formula for the projected Nystr\"om approximation in Corollary~\ref{corol:Nyst_error},
    \[
    \mathbb{E}_\calC[\widetildeK - \widetilde{L(\mathcal{C})}] = \lambda \widetildeK (\widetildeK + \lambda\mathbb{I})^{-1}, \text{ where } \calC\sim DPP(K/\lambda,V).    
    \]
     Notice that the expected subset size of $\calC\sim DPP(K/\lambda,V)$ is given by
    \[
    \mathbb{E}_{\calC}[|\calC|] = p + d_{\rm eff}(\widetilde{K}/\lambda), \text{ with } d_{\rm eff}(\widetilde{K}/\lambda) = \Tr\left(\widetilde{K} (\widetilde{K}+ \lambda\mathbb{I}_n)^{-1} \right),
    \]
    where $p$ is the number of columns of $V$ and with $\widetildeK = \mathbb{P}_{V^\perp} K \mathbb{P}_{V^\perp}$.
    The interpretation of the above identities is that a small $\lambda>0$ yields a large number of samples and a small error on expectation. This extends similar results obtained independently in~\cite[Corollary 2]{fanuel2020diversity} and~\cite{ImprovedNys}.
    
    \paragraph{Stability result} In Section~\ref{sec:BoundRisk}, we give an expected risk bound for the estimator $\hat{\bm{z}}_{N}$ of the $\gamma$-regularized semi-parametric regression obtained with the projected Nystr\"om approximation, that is,
    \[
    \mathbb{E}_{\calC}\left[\sqrt{\frac{\mathcal{R}(\hat{\bm{z}}_{N})}{\mathcal{R}(\hat{\bm{z}})}}\right]\leq 1+ \frac{\lambda}{n\gamma} d_{\rm eff}(\widetilde{K}/\lambda),
    \text{ with }     \calC\sim DPP(K/\lambda,V),
    \]
where the expected risk of the estimator $\hat{\bm{z}}$ is 
$\mathcal{R}(\hat{\bm{z}}) \triangleq \mathbb{E}_\epsilon\|\hat{\bm{z}}-\bm{z}\|_2^2,
$ as it is detailed in Theorem~\ref{thm:Bound} hereafter. This stability result indicates that the estimation thanks to the Nystr\"om approximation cannot be arbitrarily worse than the estimation obtained without approximation.


Two different applications are considered within the penalized kernel regression framework. 1) The first case occurs  when the output values (the $y_i$'s) are initially unknown to the user and costly to retrieve. This is for example the case in an active learning approach where the data points have to be manually labelled or when measurements are expensive. 
This application is known as `discrete' experimental design and was previously studied, e.g., for linear regression in~\cite{derezinski2017unbiased,derezinski2020bayesian,derezinski2019minimax}. In this setting, one interpolates on a small number of selected landmark points to minimize the number of necessary labeled points. The question now poses itself: what is a good way of selecting points such that the performance is maintained together with a good conditioning of the linear system? In this paper, we propose a determinantal design approach. 2) The user has knowledge of the full response vector $\bm{y}$, but the (embedded) application requires a number of parameters smaller than $n+p$, or the number of data points $n$ is too large to solve the corresponding linear system. 
\paragraph{Random design regression} Incidentally, we provide in Section~\ref{sec:randomDesign} a discrete random design method for parametric problems of the type $ \min_{\bm{\beta}}\|V \bm{\beta} - \bm{y}\|_2^2$. Essentially, a partial projection DPP is used in order to sample a subset $\calC$ so that the estimator $ \hat{\bm{\beta}}(\calC) =\arg\min_{\bm{\beta}}\|V_\calC \bm{\beta} - \bm{y}_\calC\|_2^2$ is unbiased. These result are analogous to those of~\cite{JMLR:Warmuth} although the sampling algorithm is different. Our analysis directly follows from the main result in Theorem~\ref{thm:implicit_reg}, and provides an alternative method generalizing volume sampling which might be of independent interest.
\subsection{Related work}
While it is currently  an active topic of research in the context of deep neural networks, implicit regularization has been studied already previously in \cite{Mahoney12,MahoneyOrecchia}.
Recently, DPPs and implicit regularization also appeared in the context of double descent phenomena \cite{DerezinskiLiangManhoney}, while we refer to~\cite{DerezinskiManhoney} for a review. DPPs are useful methods to sample diverse subsets that have been applied in machine learning in variety of tasks, such as  diverse recommendations, summarizing text or search tasks~\cite{KuleszaTaskar}.  Implicit regularization is not specific to sampling, since it is also observed with Gaussian and Rademacher sketches of Gram matrices in~\cite{DerezinskiRP}, although a closed-form formula can  be advantageously derived with $L$-ensemble sampling.

As it was mentioned earlier in the introduction, large scale KRR has been successfully solved thanks to the Nystr\"om method combined with smart approximations of ridge leverage score (RLS) sampling in ~\cite{meanti2020kernel,BLESS} for data sets of several million points. Sampling with RLSs can be interpreted as an approximation of $L$-ensemble DPP sampling, where negative dependence is neglected~\cite{DerezinskiManhoney,fanuel2020diversity,Schreurs2020diversity}. In this spirit, new bounds on the Nystr\"om approximation errors with $L$-ensemble sampling, naturally generalize the bounds obtained with RLS sampling~\cite{ElAlaouiMahoney}. The results of the aforementioned papers have been extended to Column Subset Selection Problems (CSSP) in~\cite{ImprovedNys}.

Semi-parametric models are useful tools when some domain knowledge exists about the function to be estimated (e.g., a user wants to correct the data for a linear trend) or more understandability is required from a model~\cite{ruppert2003semiparametric,pmlr-v33-huang14}. These models combine a parametric part which is easy to understand and non-parametric term to improve performance. Semi-parametric models are used in some critical applications where a user wants to have an understandable model, without sacrificing accuracy~\cite{EspinozaConf,EspinozaJournal,smola1999semiparametric}.  

A natural application of conditionally positive semi-definite matrices is radial basis function interpolation~\cite{mouat2001fast}, which is an attractive method for interpolating and smoothing function values on scattered points in the plane. However, a potential problem is that their computation involves the solution of a linear system that is often ill-conditioned for large data sets.  Importantly, the paper~\cite{BeatsonLightBillings} studies a slightly different question, which mainly concerns the case of thin-plate splines. Thin-plate spline basis functions have been shown to be very accurate in medical imaging~\cite{carr1997surface}, surface reconstruction~\cite{carr2001reconstruction}, as well as other engineering applications~\cite{biancolini2017fast}. Given a fixed set of landmark points, the authors of~\cite{BeatsonLightBillings} propose an elegant method for choosing a suitable basis of the function space so that the linear system under study has an improved condition number. This strategy is also described in~\cite{WENDLAND2006231}.

Random design regression has been studied recently from the viewpoint of repulsive point processes, whereas optimal design has received already a lot of attention in statistics (see for example~\cite{GauthierPronzato} or~\cite{Pronzato} for a review). Prominent recent works using random designs involve volume sampling~\cite{JMLR:Warmuth,VolumeRescaled} or use the Bayesian perspective~\cite{derezinski2020bayesian}. Several extensions of DPP and volume sampling have been developed recently such as in~\cite{DerezinskiLiangManhoney} or in the generalization to Polish spaces in~\cite{Poinas}.

\subsection{Organization of the paper} In Section~\ref{sec:ExtendedL}, we introduce basic definitions. Then, the penalized semi-parametric regression and interpolation problems that we study in this paper are discussed in Section~\ref{sec:Regression}. There, we consider two case studies: thin-plate splines regression and Gaussian kernel semi-parametric regression. Next, in Section~\ref{sec:design}, we explain the implicit regularization effect of a determinantal design for optimal interpolation. In Section~\ref{sec:LargeScaleRegression}, a large scale semi-parametric regression problem and its approximation thanks to a determinantal sampling are discussed together with a custom Nystr\"om approximation. We also provide the stability result for the expected risk of this approximation, which was announced in the inroduction. Technical proofs and results are deferred to Appendix.
\subsection{Notations}
Matrices ($A$,$B$,$K$,\dots) are denoted by upper-case letters, whereas (column) vectors are denoted by bold lower-case letters, e.g., $\bmx = [x_{1 } , \dots, x_{d }]^\top \in \mathbb{R}^{d\times 1}$. As mentioned above, the canonical basis of $\mathbb{R}^n$ is written $\bm{e}_i$ for $1\leq i \leq n$. The constant $n\times 1$ vector of ones is $\bm{1}_n$. We write $A\succeq 0$ (resp. $A\preceq 0$)  if $A$ (resp. $-A$) is positive semi-definite (\emph{psd}), while $A\succeq B$ indicates that $A-B$ is \emph{psd}. The $n\times n$ identity matrix is written $\mathbb{I}_n$. The Moore-Penrose pseudo-inverse of a matrix $A$ is denoted here by $A^+$, and is given by $A^+ = (A^\top A)^{-1} A^\top$ if $A$ has full column rank.
We use calligraphic letters to denote sets, with the following exception $[n] = \{1,\dots , n\}$. In this paper, we consider the problem of sampling subsets $\calC\subseteq  [n]$. Then, it is convenient to use a sampling matrix $C\in\mathbb{R}^{n\times k}$ obtained by selecting the columns of the identity matrix corresponding to $\calC$.
The sampled submatrices are denoted as follows: $V_\calC = C^\top V$ where $V\in \mathbb{R}^{n\times p}$ and $A_{\calC\calC} = C^\top A C$ with $A\in \mathbb{R}^{n\times n}$. The characteristic function of a set $X$ is written $\mathbbm{1}(X)$. 

\section{Extended $L$-ensembles and definitions \label{sec:ExtendedL}}
To begin, we recall some definitions that were introduced in~\cite{Barthelme2020}. 
First, an essential element is the non-negative pair, which is used to define the parametric and non-parametric components of the regressors.
\begin{definition}[non-negative pair \cite{Barthelme2020}]\label{def:NNP}
Let $V\in \mathbb{R}^{n\times p}$ be a matrix with full column rank, and let $A\in \mathbb{R}^{n\times n}$ be a conditionally positive semi-definite matrix with respect to $V$, i.e., a matrix satisfying $\widetilde{A}\succeq 0$ where
\[
\widetilde{A} \triangleq \mathbb{P}_{V^\perp} A  \mathbb{P}_{V^\perp},
\]
with $\mathbb{P}_{V^\perp}$ the linear projector on the orthogonal of the space spanned by the columns of $V$. Then, the couple $(A,V)$ is called a Non-Negative Pair (NNP). 
\end{definition}
The NNPs are closely related to a certain class of kernel functions, called conditionally positive semi-definite, whose kernel matrices are \emph{psd} only within the orthogonal of a subspace. Several examples of these functions have been used in the context of optimal interpolation of scattered data. 
\begin{definition}[conditionally positive semi-definite kernels]
A kernel function $k(\bmx,\bmx')$ is conditionally positive semi-definite with respect to $\{p_m(\bmx)\}_{m=1}^p$, if for all finite sets $\{\bm{x}_i\}_{i=1}^n$ so that $V_{ij} = [p_j(\bmx_i)]$  is full column rank, the matrix $K_{ij} = [k(\bmx_i,\bmx_j)]_{i,j}$ is such that $\mathbb{P}_{V^\perp} K  \mathbb{P}_{V^\perp} \succeq 0$.
\end{definition}
The following example of NNP is classical. Other examples will be considered in the context of thin-plate spline interpolation.
\begin{example}
The kernel $k(\bmx,\bmx') = -\|\bmx-\bmx'\|_2^2$ is a conditionally positive semi-definite kernel with respect to the constant function. Indeed, let $\bmx_i\in \mathbb{R}^d$ with $1\leq i\leq n$. Then, the Gram matrix $K= [k(\bmx_i,\bm\bmx_j)]_{i,j}$ satisfies $\bm{v}^\top K \bm{v}\geq 0$ for all $\bm{v}\in \mathbb{R}^d$ such that $\bm{1}^\top_n \bm{v} = 0$. Then, $(K,\bm{1}_n)$ is a NNP. Other examples are given by the generalized multiquadrics mentioned in~\cite{Barthelme2020}.
\end{example}
Equipped with these definitions, a DPP associated to a NNP can be formally defined thanks to the following representation called extended $L$-ensemble.
\begin{definition}[extended $L$-ensemble \cite{Barthelme2020}]\label{def:ExtendedL}
Let $(L,V)$ be a NNP and $\widetildeL = \mathbb{P}_{V^\perp} L \mathbb{P}_{V^\perp}$. Then, an extended $L$-ensemble $\mathcal{Y}\sim DPP(L,V)$ satisfies
\begin{equation*}
    \Pr(\mathcal{Y}= \calC) = N^{-1} \times \det
\begin{pmatrix}
L_{\calC\calC} & V_\calC\\
V^\top_\calC  & 0\\
\end{pmatrix},
\end{equation*}
 where the normalization writes  $N = (-1)^p\det(\widetildeL+ \mathbb{I}) \det ( V^\top V)$.
\end{definition}
The connection between the marginal kernel of a partial projection DPP given in~\eqref{eq:marginalKernel} and extended $L$-ensembles is discussed in~\cite{Barthelme2020}. It is worth emphasizing that extended $L$-ensemble elegantly combine the probability mass functions of projection DPPs and $L$-ensemble (cfr.~\eqref{eq:Lensemble}). Indeed, the determinant of the block matrix in Definition~\ref{def:ExtendedL} can be conveniently calculated (see Lemma~\ref{lemma:Qperp} in Appendix), so that an equivalent expression writes
\begin{equation}
    \Pr(\mathcal{Y}= \calC) = \frac{\det\left(B(\calC)^\top \widetildeL_{\calC\calC}B(\calC)\right)}{\det(\mathbb{I} + \widetildeL)}\frac{\det(V_\calC^\top V_\calC)}{\det(V^\top V)},\label{eq:LensembleBis}
\end{equation}
where, as defined in the introduction, $B(\calC)\in \mathbb{R}^{k\times (k-p)}$ is a matrix whose columns form an orthonormal basis of the orthogonal of the column space of $V_\calC$ and conveniently can be calculated using the QR decomposition.
\begin{remark}[limit case]\label{remark:ProjDPP}
Consider the partial projection process $DPP(t\mathbb{I}, V)$ with $t>0$. Its probability mass function converges pointwisely to 
\[
\Pr(\mathcal{Y} =\calC) = \mathbbm{1}(|\calC| = p)\frac{\det(V_\calC^\top V_\calC)}{\det(V^\top V)},
\]
as $t\to 0$. Thus, the limit process is actually a projection DPP, i.e., a DPP with the following marginal kernel $P = \mathbb{P}_V \triangleq V(V^\top V)^{-1} V^\top$.
\end{remark}

The subset size of an extended $L$-ensemble is also a random variable. Explicitly, the expected size of $\calC\sim DPP(K/\lambda,V)$ is
\[
\mathbb{E}\left[|\calC|\right] = \Tr\left( \widetildeK(\widetildeK + \lambda \mathbb{I})^{-1}\right) + p.
\]
The parameter $\lambda>0$ allows to vary the sample size, i.e., a small $\lambda$ value yields a large sample size on expectation and conversely.  This can be shown thanks to the marginal kernel of an extended $L$-ensemble, which was given in~\eqref{eq:marginalKernel}.
\begin{remark}[sampling]
The sampling algorithm used in this paper relies on Algorithm~3 in~\cite{JMLR:v20:18-167} (see also~\cite{KuleszaTaskar}). In all the numerical simulations, we use a fixed-size DPP sampling. A fixed-size DPP is a DPP conditioned on a fixed subset size. Our choice is motivated by the asymptotic equivalence between DPPs and fixed-size DPPs~\cite{bernouilli2019}. Importantly, the number of operations to exactly sample a DPP is $\mathcal{O}(n^3)$. This cost can be reduced if the marginal kernel has a low rank structure. Nonetheless, several approximation techniques have been published recently~\cite{DPPwithoutLooking,DPP-FX} in order to alleviate the cost of sampling fixed-size DPPs, especially if the subset size is small.
\end{remark}
\begin{remark}[leverage scores]
Leverage scores~\cite{drineas} and $\lambda$-ridge leverage scores~\cite{ElAlaouiMahoney} have been designed for randomized matrix approximations with i.i.d. sampling. Ridge leverage score (RLS) sampling can be seen as an approximation of DDP sampling by neglecting the negative dependence~\cite{DerezinskiManhoney, fanuel2020diversity}.
In regards to the marginal kernel~\eqref{eq:marginalKernel}, the marginal probabilities of a partial-projection $\mathcal{Y} \sim DPP(\widetildeK/\lambda,V)$ are the sum of a leverage score and a $\lambda$-ridge leverage score:
\begin{equation*}
    \Pr(i\in \mathcal{Y}) = \underbrace{\bm{e}_i^\top V (V^\top V)^{-1}V^\top\bm{e}_i}_{\text{leverage score}} + \underbrace{\bm{e}_i^\top\widetildeK (\widetildeK +\lambda\mathbb{I})^{-1}\bm{e}_i}_{\text{ridge leverage score}}.
\end{equation*}
Hence, partial projection DPP sampling provide a generalization of leverage score sampling. To the best of our knowledge, the above combination of leverage scores and $\lambda$-ridge leverage score sampling has received up to now little interest in the literature.
\end{remark}
%
\section{Basics of penalized semi-parametric regression\label{sec:Regression}}
We begin by introducing the framework of semi-parametric regression with a \emph{psd} kernel while an example with a conditionally positive semi-definite kernel is given below.
\subsection{Semi-parametric regression with semi-positive definite kernels}

Let $(\mathcal{H}_1,\langle \cdot , \cdot \rangle_1)$ be a Reproducing Kernel Hilbert Space (RKHS) with kernel $k(\bmx,\bmx')$. Also, let $\mathcal{H}_0$ be a Hilbert space of dimension $p<\infty$ with a basis\footnote{Any finite dimensional vector space can be endowed with a suitable scalar product so that is also a RKHS. Let $(\mathcal{H}_0,\langle \cdot , \cdot \rangle_0)$ be a RKHS. Then, $\mathcal{H}_0$ is the orthogonal complement of $\mathcal{H}_1$ with respect to the following inner product: $\langle f_0 + f_1, g_0 + g_1\rangle \triangleq  \langle f_0, g_0\rangle_0 +  \langle f_1, g_1\rangle_{1}$.} given by $p_j(\bmx)$ for $1\leq j\leq p$ and such that that $\mathcal{H}_0 \cap \mathcal{H}_1 = \{0\}$.

The function space that we consider is the direct sum $\mathcal{H}_0 \oplus \mathcal{H}_1$.
By construction, every $f\in \mathcal{H}_0 \oplus \mathcal{H}_1$ can be decomposed uniquely as $f = f_0+ f_1$ with $f_0 \in \mathcal{H}_0$ and $f_1 \in \mathcal{H}_1$. Hence, we define the penalty functional 
\[
J(f) = \langle f_1, f_1\rangle_{1},
\]
so that its null space is naturally $\mathcal{N}_J \triangleq \{f\in \mathcal{H}_0 \oplus \mathcal{H}_1: J(f) = 0\} = \mathcal{H}_0$. The penalized least-squares~\eqref{eq:pen_reg} problem reads
\begin{equation}
    \min_{f\in\mathcal{N}_J\oplus \mathcal{H}_1}\frac{1}{n} \sum_{i=1}^n \left(y_i -f(\bmx_i)\right)^2 +\gamma J(f)\tag{$\text{PLS}$}.\label{eq:pen_reg}
\end{equation}
When $\mathcal{N}_J = \{0\}$, \eqref{eq:pen_reg} reduces to Kernel Ridge Regression (KRR).
By a classical argument\footnote{A `representer theorem', see, e.g.~\cite[Section 2.3.2]{GuBook} or~\cite{NosedalSanchez2012ReproducingKH}.}, the solutions of~\eqref{eq:pen_reg} are in the semi-parametric form
$
f(\bmx) = \sum_{i=1}^n \alpha_i k(\bmx,\bmx_i) + \sum_{m=1}^p \beta_m p_m(\bmx),
$
where the first term includes only a \emph{finite} number of terms.
By plugging the above expression into the minimization problem~\eqref{eq:pen_reg}, we find the discrete minimization problem
\begin{equation}
    \min_{\bma,\bmb}\frac{1}{n}\|\bm{y}-V \bmb-K\bma\|_2^2 + \gamma \bma^\top K \bma,\label{eq:min_prob}
\end{equation}
with $V_{im} = [p_m(\bmx_i)]$ for $1\leq i\leq n$ and $1\leq m\leq p$.
Notice that $\frac{1}{n} \sum_{i=1}^n \left(y_i -f(\bmx_i)\right)^2$ is strictly convex on $\mathcal{N}_J$ if $V$ is full column rank. Furthermore, \eqref{eq:pen_reg} is strictly convex on $\mathcal{H}_0 \oplus \mathcal{H}_1$ if $V$ is full column rank.
Therefore, we assume in what follows that the data set is `unisolvent' with respect to $(p_j)_j$, i.e., such that $V$ is full column rank. In this case, the solution of~\eqref{eq:pen_reg} is unique in the light of Theorem~\ref{thm:existence} in Appendix (see~\cite{GuBook}).
The first order optimality condition of the optimization problem~\eqref{eq:min_prob} reads
 \begin{align*}
    K\left[\left(K+n\gamma\mathbb{I}\right)\bma + V \bmb -\bm{y}\right] &= 0\\
    V^\top \left(K\bma+V\bmb - \bm{y}\right)& = 0,
    \end{align*}
where $K$ is positive semi-definite. We assume first that $K$ is non-singular.
As it can be verified by a simple substitution of $\bma$ and $\bmb$ in the first order conditions, the unique solution of~\eqref{eq:min_prob} is obtained by solving
\begin{align}
\begin{pmatrix}
K+n\gamma\mathbb{I} & V\\
V^\top & 0\\
\end{pmatrix}
\begin{pmatrix}
\bma\\
\bmb\\
\end{pmatrix}
= \begin{pmatrix}
\bm{y}\\
 \bm{0}\\
\end{pmatrix}.\label{eq:LinearSystem}
\end{align}
Second, if $K$ is singular, then~\eqref{eq:min_prob} can have several solutions $\bma^\star$ and $\bmb^\star$, which yield the same in-sample estimator $\hat{\bm{z}} =K\bma^\star + V \bmb^\star$ of the true function values  $z_i =f(\bmx_i)$ for $1\leq i\leq n$. In that case, we select the solution corresponding to the coefficients obtained by solving~\eqref{eq:LinearSystem}.
\subsection{Case study: the Gaussian kernel RKHS does not contain polynomials\label{sec:Gaussian+Poly}}
Let $X\subset\mathbb{R}^d$ be any set with non-empty interior and let $\mathcal{H}_1$ be the RKHS of the Gaussian kernel $k(\bmx,\bmx') = \exp(-\|\bmx-\bmx'\|_2^2/\sigma^2)$ defined on $X\times X$.   Under these assumptions, Theorem 2 in~\cite{Minh2010SomePO} states  that $\mathcal{H}_1$ does  not  contain  any  polynomial  on $X$, including  the non-zero constant function. We can then solve the functional minimization problem~\eqref{eq:pen_reg} where $\mathcal{H}_0$ is a finite set of polynomials, since $\mathcal{H}_0\cap \mathcal{H}_1= \{0\}$. 

\begin{example}[LS-SVM with the Gaussian kernel]
In particular, the constant $p_1(\bmx) = 1$ is not part of the RKHS of the Gaussian kernel. Therefore, Least-Squares Support Vector Machine (LS-SVM)~\cite{suykens:worldsci2002} with the Gaussian kernel is also a particular case of the above discussion. Its dual optimization problem indeed reads
$
    \min_{\bma,b}\frac{1}{n}\|\bm{y}-K\bma - b\bm{1}_n\|_2^2 + \gamma \bma^\top K \bma,
$
where the real $b$ is the so-called bias term.
\end{example}

In Figure~\ref{fig:toy_semi}, we illustrate the use of semi-parametric regression with a Gaussian kernel on a toy example consisting of a linear trend with two Gaussian bumps, i.e., $f(x) = x + 7 + 4 \, \mathrm{exp}(-(x - 4)^2) -  4 \, \mathrm{exp}(-(x + 4)^2)$. The training points are sampled uniformly within the interval $[-10,10]$ and the function samples are $y_i  = f(x_i) +\sfe_i$ where $\sfe_i\sim \mathcal{N}(0, 0.2)$ for $1\leq i\leq n$ and $n = 40$. The test set consists of 1000 points sampled uniformly in the interval $[-11,11]$. This construction allows to assess the ability of the estimated function to capture the linear trend of the ground truth.

Let $\mathcal{H}_1$ be the RKHS of the Gaussian kernel. We compare the results obtained by different choices of the space of polynomials $\mathcal{H}_0$. Specifically, the following models are estimated: \texttt{Model 1}: $\hat{f}(x) = \beta_1 + \beta_2 x + \sum_i \alpha_i k(x,x_i)$ (semi-parametric), \texttt{Model 2}: $\hat{f}(x) = \beta_1 + \sum_i \alpha_i k(x,x_i)$ (LS-SVM) and \texttt{Model 3}: $\hat{f}(x) = \sum_i \alpha_i k(x,x_i)$ (KRR). For all the models drawn in Figure~\ref{fig:toy_semi_1}, the bandwidth is fixed to $\sigma =1$ to match the width of the two Gaussians of the ground truth and the regularization parameter $\gamma$ takes values in the set $\{10^{-j}\}_{ j\in\{1,\dots, 8\}}$. For completeness, we also include the performance of each model after parameter tuning in Figure~\ref{fig:toy_semi_2} (dashed line), where both the bandwidth $\sigma \in \{0.1,0.2,\dots,0.9,1,2,3,\ldots,10\}$ and regularization parameter are determined by using $10$ fold cross-validation, on the above-mentioned grid. The simulation is repeated $25$ times and the error bars show the $97.5\%$ confidence interval.
In Figure~\ref{fig:toy_semi_1}, the function prediction in low density regions of both \texttt{model 2} and \texttt{model 3} quickly moves to the bias. This is avoided by including a parametric linear part such as in \texttt{model 1}. We emphasize that the differences between the three models are reduced within the interval $[-10,10]$ if the number of training samples becomes larger.

\subsection{Case study: conditionally positive semi-definite kernels and thin-plate splines\label{sec:thin-plate}}
 A well-known choice of penalty functional yielding to a linear system with a conditionally positive semi-definite kernel is associated to thin-plate splines, and is given by
\[
J_p^d(f) \triangleq \langle f,f\rangle_1\text{ with } \langle f,g\rangle_1 =  \sum_{\alpha_1+\dots+\alpha_d = m}\frac{p!}{\alpha_1!\dots\alpha_d!}\int_{\mathbb{R}^d}\frac{\partial^{p}f}{\partial x_{1 }^{\alpha_1} \dots \partial x_{d}^{\alpha_d}} \frac{\partial^{p}g}{\partial x_{1}^{\alpha_1} \dots \partial x_{ d}^{\alpha_d}}\rmd x_{1} \dots \rmd x_{d},
\]
where $J_p^d(f)$ is a squared semi-norm on $\{f:J_p^d(f)<\infty\}$ for a large enough regularity index with respect to the dimension, i.e., for $2p> d$. Its null space $\mathcal{N}_J$ consists of polynomials of maximal total order equal to $p-1$.
Following section 4.3.2 of~\cite{GuBook}, the penality functional $J_p^d(f)$ satisfies
\[
J_p^d\left(\sum_{i=1}^n \alpha_i k(\bmx,\bmx_i)\right) = \sum_{i,j=1}^n \alpha_i \alpha_j k(\bmx_i,\bmx_j),
\text{ for all } \bma\in\mathbb{R}^n \text{ such that } V^\top \bma=0,
\]
where $V$ is assumed to be full column rank and where the thin-plate spline kernel reads
\[
k(\bmx,\bmx')=\begin{cases}
\|\bmx-\bmx'\|^{2p-d}_2 \log \|\bmx-\bmx'\|_2 \text{ for even } d\\
\|\bmx-\bmx'\|^{2p-d}_2 \text{ for odd } d.
\end{cases}
\]
The kernel $k(\bmx,\bmx')$ is conditionally  positive semi-definite, namely, $\sum_{i,j} \alpha_i \alpha_j k(\bmx_i,\bmx_j)\geq 0,$ for all vectors satisfying $V^\top \bm{\alpha}=0$. Again, $V$ is assumed here to be full column rank.
By an similar argument as in the previous section (see~\cite{GuBook}), the solution of the least-squares penalized regression
\[
\min_{J_p^d(f)<\infty}\frac{1}{n} \sum_{i=1}^n \left(y_i -f(\bmx_i)\right)^2 +\gamma J_p^d(f),
\]
is of the form
$f(\bmx) = \sum_{i=1}^n \alpha_i k(\bmx,\bmx_i) + \sum_{j=1}^p \beta_j p_j(\bmx) \text{ with } V^\top \bma = 0. $
This result is proved in~\cite[Theorem 4 bis]{Duchon1976SplinesMR} in the case of optimal interpolation (i.e., $\gamma\to 0$). 
The substitution of $f(\bmx)$ into the objective above yields a similar discrete minimization problem as~\eqref{eq:min_prob} with the extra condition $V^\top \bma = 0$. In particular, a solution of this minimization problem is given by the same system as \eqref{eq:LinearSystem}
with the exception that, here, $K$ is \emph{conditionally} positive semi-definite. Let $V_\perp$ a matrix with orthonormal columns such that $\mathbb{P}_{V_\perp} = V_\perp V_\perp^\top$.
The solution of this linear system is given as 
\begin{align*}
    \bma^\star &= V_\perp (V_\perp^\top K V_\perp + n\gamma \mathbb{I}_{n-p})^{-1}V_\perp^\top \bm{y}\\
    \bmb^\star &= (V^\top V)^{-1} V^\top (\bm{y}- K\bma),
\end{align*}
where we used the fact that $V$ is full column rank and  $\mathbb{P}_{V_\perp} K \mathbb{P}_{V_\perp}\succeq 0$.
Notice that the full in-sample estimator is
$
\hat{\bm{z}} = \widetilde{K} ( \widetilde{K} + n\gamma \mathbb{I})^{-1}  \bm{y} + \mathbb{P}_V \bm{y}.
$
The RKHS associated to the thin-plate splines and a \emph{psd} kernel built from the conditionally positive semi-definite kernel are determined in~\cite{GuBook}, where this problem is put in the form of~\eqref{eq:pen_reg}. Then, the domain $\{f: J_p^d(f)<\infty\}$ is shown to be the direct sum of a RKHS and the space of polynomials of maximal total degree $p-1$. 
For completeness, let us mention that a discussion of optimal interpolation with conditionally positive semi-definite kernels, within the framework of Hilbertian subspaces of L. Schwartz, can be found in the PhD thesis~\cite{phdthesisGauthier}.
Consider now the use of extended $L$-ensembles for obtaining designs in the context of optimal interpolations.
\FloatBarrier
\section{Implicit regularization of optimal interpolation with a determinantal design\label{sec:design}}
In the context of the applications mentioned in the introduction, we consider here the problem of interpolating function values given a small training data set.
In this section, we assume that obtaining the responses $y_i$ is expensive and therefore we look for a discrete design $(\bmx_{i_\ell},y_{i_\ell})$ for $i_\ell\in \calC$.
An interpolator is obtained by taking the `ridgeless' limit $\gamma\to 0$ of the regression problem~\eqref{eq:min_prob}.
Let $\bm{k}_{\bmx} = [k(\bmx, \bmx_1) \dots k(\bmx, \bmx_n)]^\top \in \mathbb{R}^n$ and $\bm{p}_{\bmx} = [ p_1(\bmx)  \dots  p_p(\bmx)]^\top\in \mathbb{R}^p$ for all $\bmx\in \mathbb{R}^d$. The estimated interpolator on a subset $\calC$ reads
\begin{equation}
\hat{f}_0(\bmx,\calC) = 
\begin{pmatrix}
\bm{k}_{\bmx}^\top C & \bm{p}^\top_{\bmx}
\end{pmatrix}
\begin{pmatrix}
K_{\calC\calC} & V_\calC\\
V^\top_\calC & 0\\
\end{pmatrix}^{-1}\begin{pmatrix}
 \bm{y}_\calC\\
 \bm{0}\\
\end{pmatrix}.\label{eq:subsampledInterpolation}
\end{equation}
If $\calC\sim DPP(K,V)$, notice that: (i)
 the square matrix on the RHS of~\eqref{eq:subsampledInterpolation}  is non-singular almost surely, and (ii)
 $V_\calC$ is full column rank almost surely, in the light of Lemma~\ref{lemma:Qperp}.
One of our main results is that the linear system in~\eqref{eq:LinearSystem} is regularized on expectation when the subsets $\calC$ are sampled according to a suitable DPP.
\begin{theorem}[implicit regularization on expectation]
\label{thm:implicit_reg}
Let $(K,V)$ be a NNP. Let $\bm{u}_0, \bm{v}_0 \in \mathbb{R}^n$ and $\bm{u}_1, \bm{v}_1 \in \mathbb{R}^p$. Then, we have the following identity
\[
\mathbb{E}_{\calC\sim DPP(K,V)}\left[ \begin{pmatrix}
\bm{u}_{0,\calC}\\
\bm{u}_1
\end{pmatrix}^\top\begin{pmatrix}
K_{\calC\calC} & V_\calC\\
V^\top_\calC & 0\\
\end{pmatrix}^{-1} \begin{pmatrix}
\bm{v}_{0,\calC}\\
\bm{v}_1
\end{pmatrix}\right]= \begin{pmatrix}
\bm{u}_0\\
\bm{u}_1
\end{pmatrix}^\top\begin{pmatrix}
K + \mathbb{I}  &  V\\
V^\top  & 0\\
\end{pmatrix}^{-1}\begin{pmatrix}
\bm{v}_0\\
\bm{v}_1
\end{pmatrix}.
\]
\end{theorem}
\begin{proof}
The proof of this result is given in Section~\ref{sec:deferred} in Appendix and mainly relies on the matrix determinant lemma.
\end{proof} Notice that only the upper left block is regularized in the above matrix inverse.
Also, this identity remains valid when $K$ is replaced by $\widetildeK = \mathbb{P}_{V_\perp} K \mathbb{P}_{V_\perp}$.
Similarly to~\eqref{eq:subsampledInterpolation}, the $\gamma$-regularized regressor on the full data set obtained by solving~\eqref{eq:pen_reg} is
\[
\hat{f}_\gamma(\bmx) = 
\begin{pmatrix}
\bm{k}_{\bmx}^\top & \bm{p}^\top_{\bmx}
\end{pmatrix}
\begin{pmatrix}
K+n\gamma\mathbb{I} & V\\
V^\top & 0\\
\end{pmatrix}^{-1}\begin{pmatrix}
 \bm{y}\\
 \bm{0}\\
\end{pmatrix}.
\]
The upshot is that the interpolator obtained with this determinantal design is actually regularized on expectation, as a direct consequence of Theorem~\ref{thm:implicit_reg}. This result is formalized in Corollary~\ref{corol:Extended_reg} which generalizes a similar result for KRR in the ridgeless limit given in~\cite{Schreurs2020diversity}.
\begin{corollary}[ensemble of interpolators]\label{corol:Extended_reg}
Let $\calC\sim DPP(K/(n\gamma),V)$. We have
$
\mathbb{E}_{\calC}[\hat{f}_0(\bmx,\calC)] = \hat{f}_\gamma(\bmx) 
$
 for all $\bmx\in \mathbb{R}^d$.
\end{corollary}
The interpretation of Corollary~\ref{corol:Extended_reg}  goes as follows: an average of interpolators obtained with this random design gives a regularized regressor on the full data set.
We refer to~\cite{derezinski2020bayesian} for another discussion of connections between Bayesian experimental design and DPPs.
In the next section, we illustrate the use of a discrete determinantal design in the context of optimal interpolation with thin-plate splines.
\subsection{Illustration of a discrete determinantal design for thin-plate spline interpolation}

We illustrate the effect of subsampling for thin-plate spline interpolation on Franke's function, which is frequently used to demonstrate radial basis function interpolation problems. Franke's function has two Gaussian peaks of different heights, and a smaller dip:
\begin{equation*}
\begin{aligned}
f(\mathbf{x}) =& 0.75 \exp \left(-\frac{\left(9 x_{ 1 }-2\right)^{2}}{4}-\frac{\left(9 x_{ 2}-2\right)^{2}}{4}\right)+0.75 \exp \left(-\frac{\left(9 x_{ 1}+1\right)^{2}}{49}-\frac{9 x_{2}+1}{10}\right) \\
&+0.5 \exp \left(-\frac{\left(9 x_{1}-7\right)^{2}}{4}-\frac{\left(9 x_{ 2}-3\right)^{2}}{4}\right)-0.2 \exp \left(-\left(9 x_{ 1 }-4\right)^{2}-\left(9 x_{ 2}-7\right)^{2}\right).
\end{aligned}
\end{equation*}
The full training set consists of 5000 points sampled uniformly at random within $[0 , 1]^2$, the test set consists of 10000 points sampled from the same domain. The full interpolation problem is solved by using~\eqref{eq:LinearSystem} with $\gamma = 0$ and regression function $f(\bmx) = \sum_{m=1}^2 b_m x_{ m } + b_0 + \sum_{i=1}^n \alpha_i\|\bmx - \bmx_i\|^2_2 \mathrm{log}\|\bmx - \bmx_i\|_2$. The subsampled interpolation problem is solved by using~\eqref{eq:subsampledInterpolation}. In this simulation, we compare uniform sampling to the associated partial-projection DPP. The simulation is repeated for an increasing subset size, and the  performance is measured by the mean squared error (MSE) on the test set. In the experiments, we sample each time a fixed size partial-projection DPP (with $|\calC| = k$) to finer control the number of sampled landmarks. Every sampling is repeated 10 times and the averaged results are visualized in Figure~\ref{fig:Franke}. Error bars correspond to the $97.5\%$ confidence interval. For a given subset size, the partial-projection DPP outperforms uniform sampling.

\begin{figure}[ht]
		\centering
	\begin{subfigure}[b]{0.42\textwidth}
			\includegraphics[width=\textwidth, height= 0.9\textwidth]{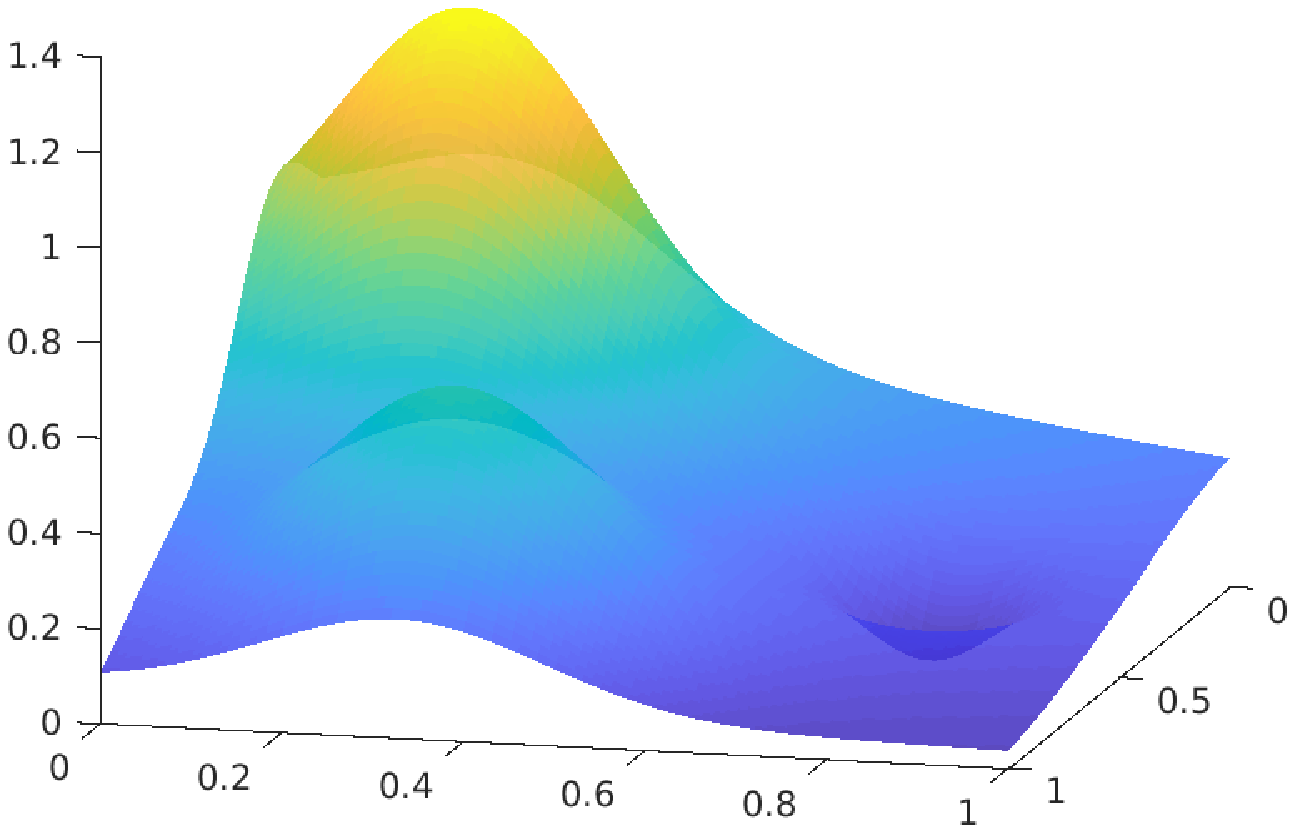}
			\caption{Franke's function}
			\label{fig:frank1}
		\end{subfigure}
		\begin{subfigure}[b]{0.4\textwidth}
			\includegraphics[width=\textwidth, height= 0.8\textwidth]{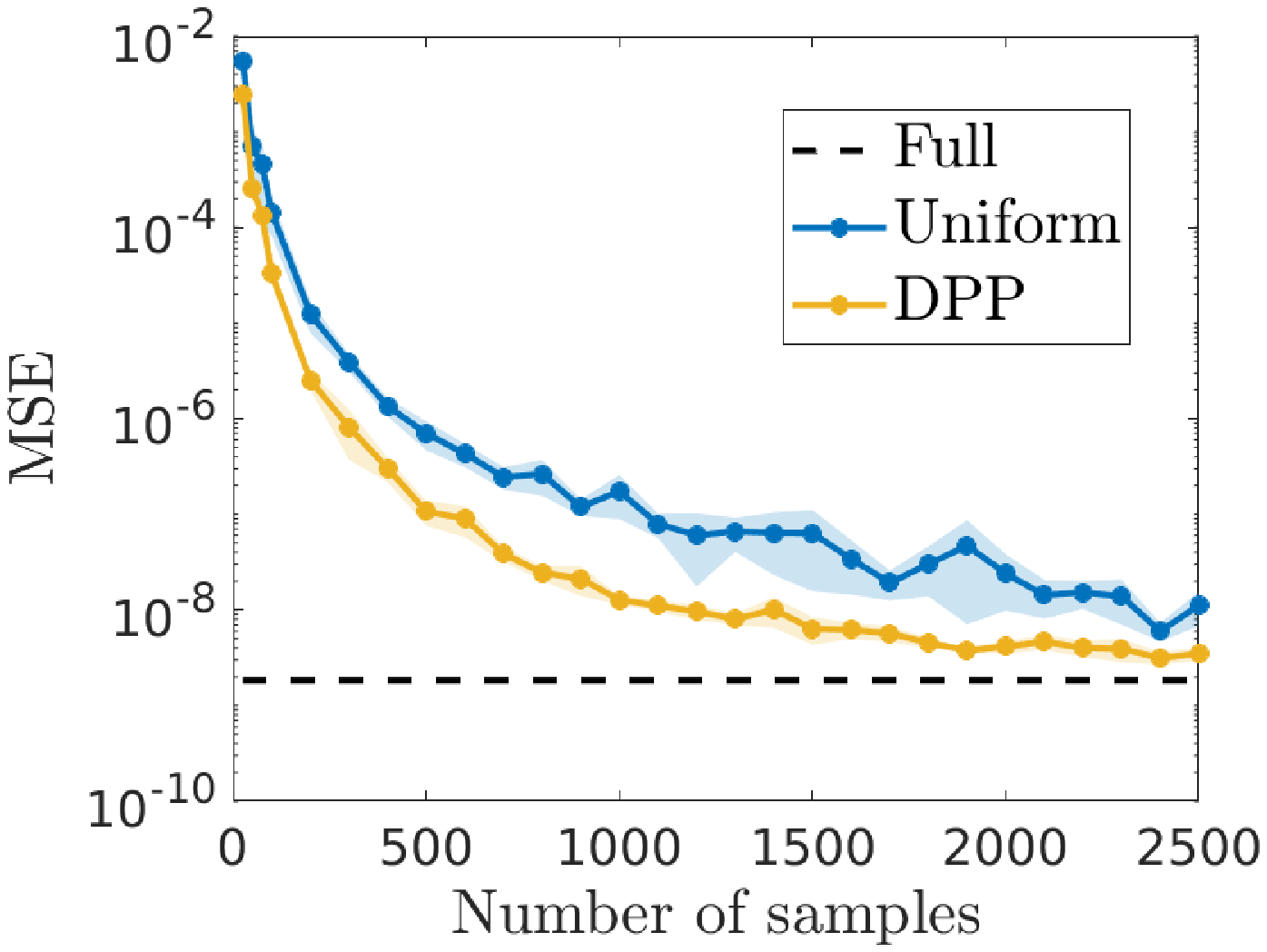}
			\caption{Performance}
			\label{fig:frank2}
		\end{subfigure}
	\caption{Figure \ref{fig:frank1} displays a mesh plot of the Franke's function. The MSE on the test set as a function of the subset size $|\calC|$ is given in Figure~\ref{fig:frank2}. }\label{fig:Franke}
\end{figure}

\subsection{Empirical results for Gaussian kernel interpolation\label{sec:EmpGaussian}}
We illustrate here the effect of extended $L$-ensemble sampling versus uniform sampling for subsampled interpolation on a number of UCI benchmark regression data sets: \texttt{Boston Housing}, \texttt{Abalone} and \texttt{Parkinson}. Both the regressors and response are standardized, afterwards the data set is split into a $50\%$ training and $50\%$ test set, and the performance is measured by the total MSE:  $\sum_{i=1}^n \|y_i - \hat{z}_i\|^2$. To obtain the regressor, we solve the system \eqref{eq:LinearSystem}. A Gaussian kernel $k(\bmx,\bmx')  = \exp(-\|\bmx-\bmx'\|_2^2/\sigma^2)$ is used with a linear regression component: $V = [X \enskip \bm{1}_n]$ where $X = [\bmx_1 \dots \bmx_n]^\top \in \mathbb{R}^{n\times d}$.
The squared bandwidth is determined by using the median heuristic~\cite{gretton2012kernel}, computed as: $
	 \hat{\sigma}^2 = \mathrm{median}\{\left\|x_{i}-x_{j}\right\|_2^{2} : 1 \leqslant i<j \leqslant n\}/2$.  Cross-validation is not possible as the full $\bm{y}$ is hidden from us in the experimental design setup. The simulation is repeated 25 times and the error bars show the $97.5\%$ confidence interval. The results are displayed in Figure~\ref{fig:optInt}. We observe that extended $L$-ensemble sampling improves the performance especially for smaller number of samples, compared to uniform sampling.  \\
	 
%
\begin{figure}[ht]
		\centering
	\begin{subfigure}[b]{0.32\textwidth}
			\includegraphics[width=\textwidth, height= 0.8\textwidth]{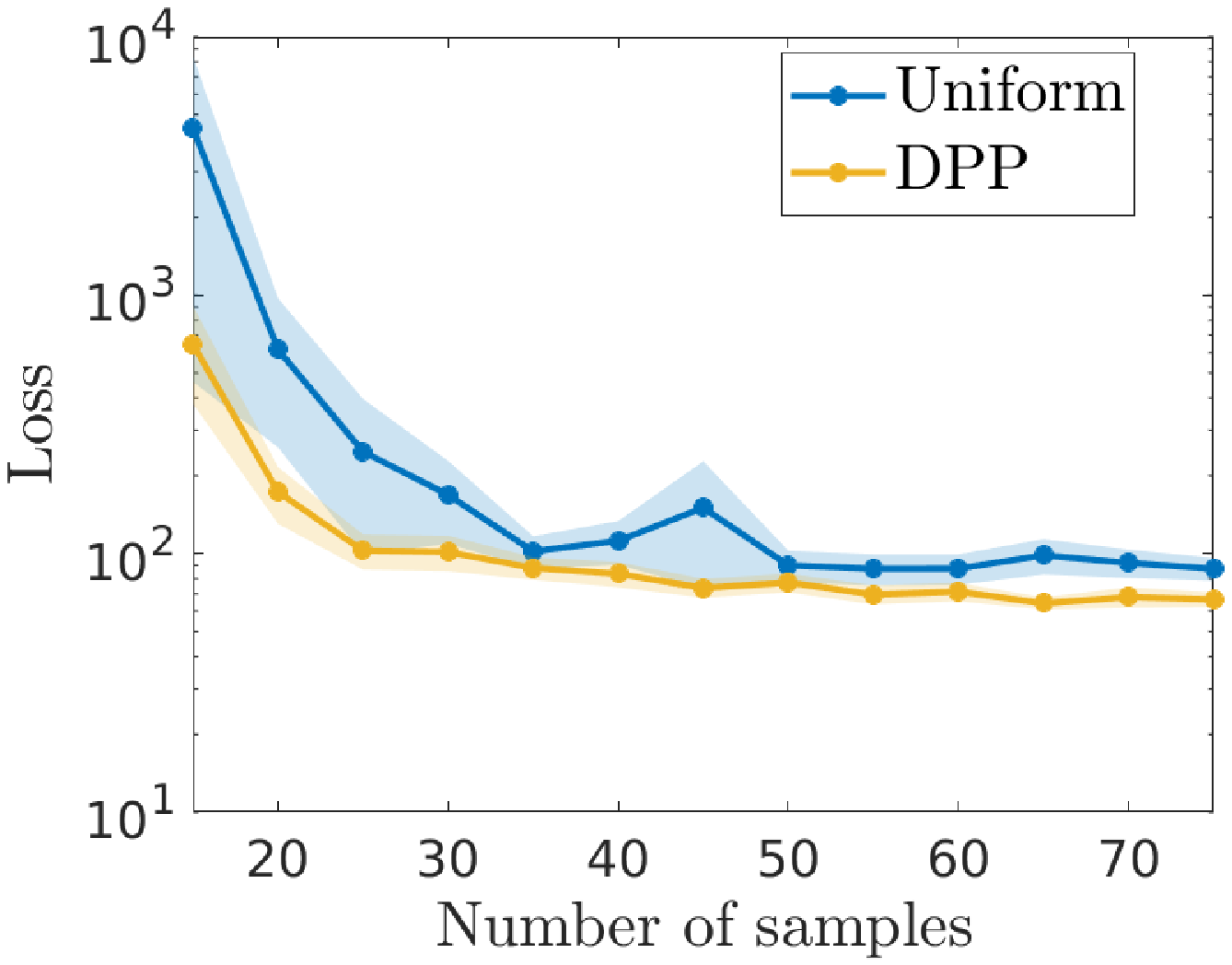}
			\caption{\texttt{Boston Housing}}
			\label{fig:optInt1}
		\end{subfigure}
		\begin{subfigure}[b]{0.32\textwidth}
			\includegraphics[width=\textwidth, height= 0.8\textwidth]{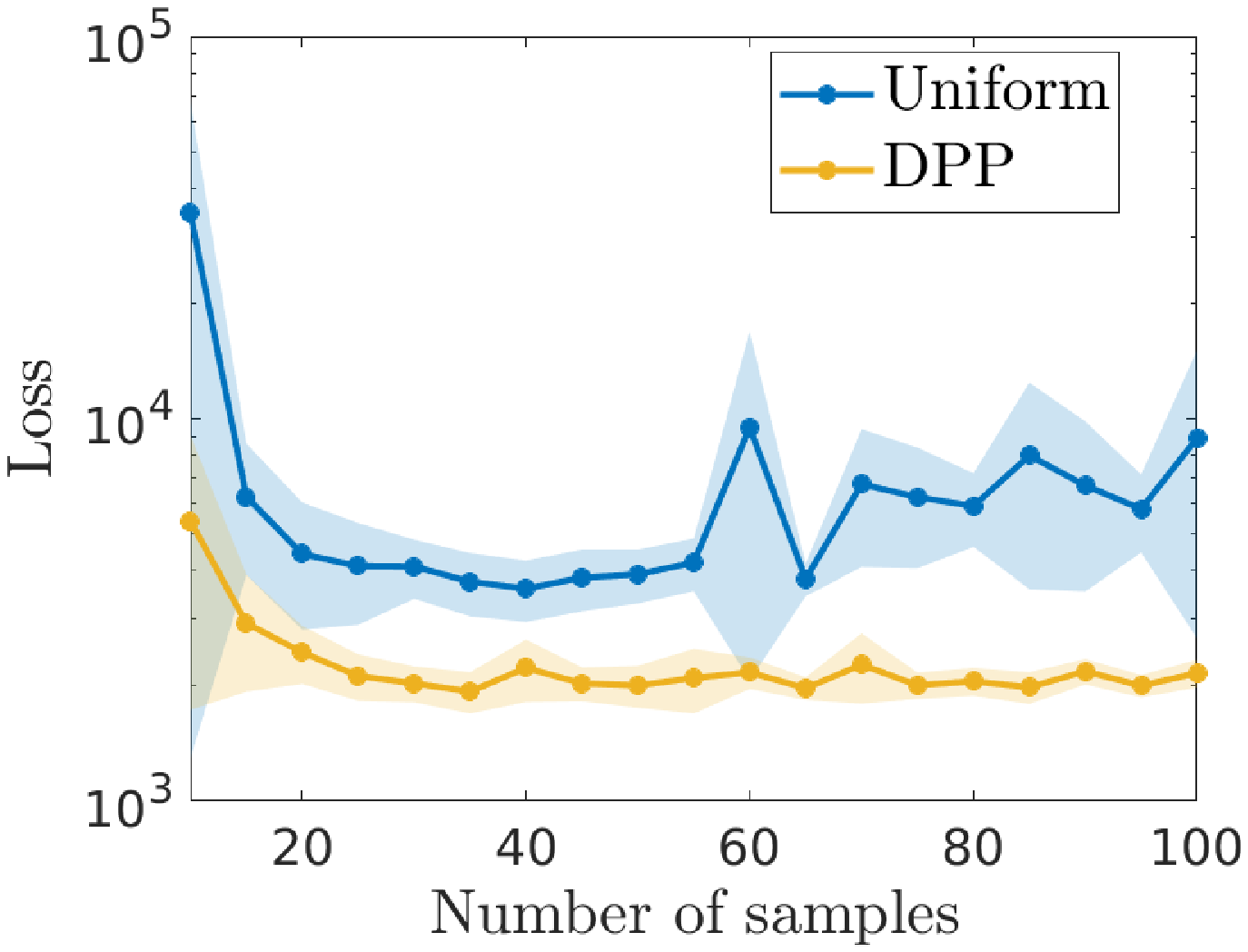}
			\caption{\texttt{Abalone}}
			\label{fig:optInt2}
		\end{subfigure}
		\begin{subfigure}[b]{0.32\textwidth}
			\includegraphics[width=\textwidth, height= 0.8\textwidth]{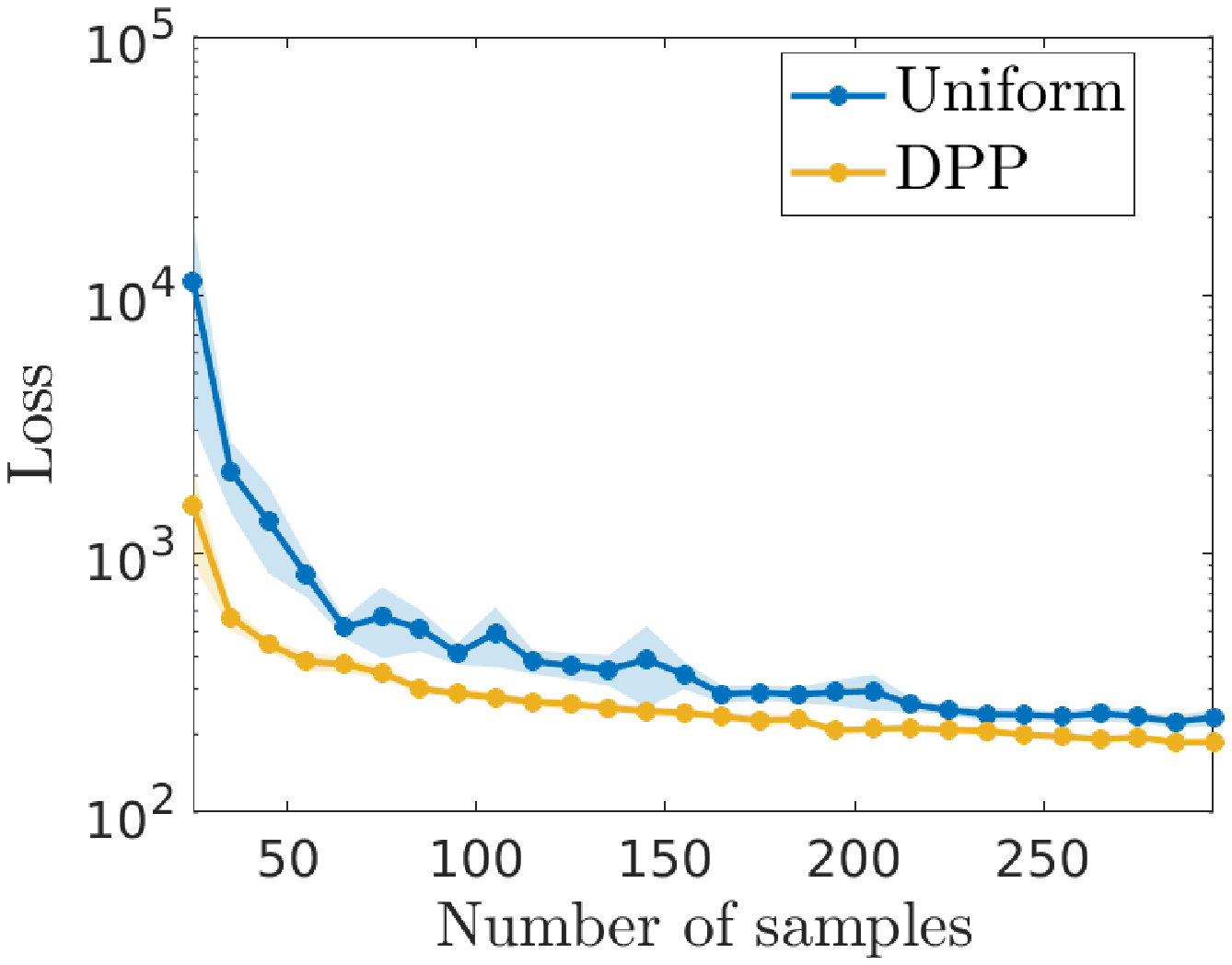}
			\caption{\texttt{Parkinson}}
			\label{fig:optInt3}
		\end{subfigure}
	\caption{The total loss (MSE) in function of the number of landmarks using uniform vs extended DPP sampling with the bandwidth estimated using the median heuristic.}\label{fig:optInt}
\end{figure}
\FloatBarrier
\section{Large scale regularized semi-parametric regression\label{sec:LargeScaleRegression}}
In this section, we consider the setting where the training points $(\bmx_i,y_i)$ with $1\leq i\leq n$ are abundant.
Penalized least-squares problems of the form of~\eqref{eq:pen_reg} have solutions which are determined by $n+p$ parameters where $n$ is the number of training points and $p$ is the number of functions used in the parametric component. As it was anticipated in the introduction, when $n$ is large, it can be interesting to reduce the number of parameters describing the estimated function, for instance, in order to allow for a faster out-of-sample prediction.  

\subsection{Preliminary results about implicit regularization}
For a $n\times n$ positive semi-definite kernel matrix $K$, the Nystr\"om approximation of $K$ associated to $\calC\subseteq [n]$ reads $K_\calC^\top K_{\calC\calC}^{+} K_\calC$. We extend here this definition to the case of conditionally positive semi-definite kernels.
First, we provide in Proposition~\ref{prop:ExtendedExpectedpinv} the expectation of an analogue of $CK_{\calC\calC}^{+}C^\top$ accounting for conditional positivity.
\begin{proposition}[implicit regularization of the projected kernel matrix]\label{prop:ExtendedExpectedpinv}
Let $\calC\sim DPP(K/\lambda,V)$. Then, we have
\begin{equation}
\mathbb{E}_\calC[I(\calC)]= (\widetildeK + \lambda\mathbb{P}_{V^\perp})^+ \text{ with  } I(\calC) = (C \mathbb{P}_{V^\perp_\calC} K_{\calC\calC} \mathbb{P}_{V^\perp_\calC} C^\top)^+.\label{eq:id0}
\end{equation}
Furthermore, it also holds that
\begin{align}
    &\mathbb{E}_\calC\left[(CV_\calC)^+  \left(\mathbb{I} - KI(\calC)\right)\right] = V^+\left(\mathbb{I} - K(\widetildeK +\lambda \mathbb{P}_{V^\perp})^+\right)\label{eq:id1}\\
      &\mathbb{E}_\calC\left[(CV_\calC)^+ \left(K - K I(\calC)K\right) (CV_\calC)^{+\top} \right] =  V^+ \left(K + \lambda\mathbb{I} - K (\widetildeK +\lambda \mathbb{P}_{V^\perp})^+ K \right) V^{+\top},\label{eq:id2}
\end{align}
\end{proposition}
The identity~\eqref{eq:id0} is equivalent to the expected pseudo-inverse formula~\eqref{eq:ExtendedExpectedpinv} announced in the introduction whereas the identities~\eqref{eq:id1} and \eqref{eq:id2} are incidental and are studied in more detail the the next section.
\begin{proof}
We begin by noticing that, thanks to Lemma~\ref{lem:pinv} in Appendix, it holds that
\[
(C \mathbb{P}_{V^\perp_\calC} K_{\calC\calC} \mathbb{P}_{V^\perp_\calC} C^\top)^+ = C( \mathbb{P}_{V^\perp_\calC} K_{\calC\calC} \mathbb{P}_{V^\perp_\calC})^+ C^\top \text{ and } (\widetildeK + \lambda\mathbb{P}_{V^\perp})^+ = \mathbb{P}_{V^\perp}(\widetildeK + \lambda \mathbb{I})^{-1}\mathbb{P}_{V^\perp}.
\]
Without loss of generality, we now take $\lambda = 1$, since the results can be recovered at the end for any $\lambda>0$ by a simple rescaling of $K$.
Let $B(\calC)$ be a matrix whose columns are an orthonormal basis of the column space of $(V_\calC)^\perp$. In this case, we have $P_{V_\calC^\perp} = B(\calC)B(\calC)^\top$. Then, we have the explicit expression 
\[
(\mathbb{P}_{V^\perp_\calC} K_{\calC\calC} \mathbb{P}_{V^\perp_\calC} )^+ = B(\calC)\left( B(\calC)^\top K_{\calC\calC}B(\calC)\right)^{-1}B(\calC),
\]
where we used once more Lemma~\ref{lem:pinv} (with $S = B(\calC)$ and $M = B(\calC)^\top K_{\calC\calC}B(\calC)$).
The identities~\eqref{eq:id0},~\eqref{eq:id1} and~\eqref{eq:id2} are obtained in what follows by merely calculating matrix inverse in the formula given in Theorem~\ref{thm:implicit_reg}.
For simplicity, define
\[
T(\calC) = 
\begin{pmatrix}
 C  &  0\\
0 & \mathbb{I}\\
\end{pmatrix}\begin{pmatrix}
 K_{\calC\calC} &  V_\calC\\
V^\top_\calC & 0\\
\end{pmatrix}^{-1}
\begin{pmatrix}
 C^\top  &  0\\
0 & \mathbb{I}\\
\end{pmatrix}, \text{ for } \calC\sim DPP(K,V).
\]
The above matrix inverse is now calculated by using Lemma~\ref{lemma:proj_inverse_blocks} in Appendix (with $A = K_{\calC\calC}$ and $W = V_\calC$).  Remark that $B^\top(\calC) K_{\calC\calC}B(\calC)$ is non-singular. Indeed, since $\calC\sim DPP(K,V)$,  we have
\[
0\neq \det \begin{pmatrix}
K_{\calC\calC}  & V_\calC\\
V^\top_\calC & 0\\
\end{pmatrix} = (-1)^p \det\left(V_\calC^\top V_\calC\right) \underbrace{\det\left(B^\top(\calC) K_{\calC\calC}B(\calC)\right)}_{\neq 0},
\]
where the determinant on the LHS is calculated thanks to Lemma~\ref{lemma:Qperp} in Appendix.
Thus, we find the following expression
\[
T(\calC)= \begin{pmatrix}
 C(\mathbb{P}_{V^\perp_\calC} K_{\calC\calC} \mathbb{P}_{V^\perp_\calC} )^+ C^\top   &  C(\mathbb{I}- (\mathbb{P}_{V^\perp_\calC} K_{\calC\calC} \mathbb{P}_{V^\perp_\calC} )^+ K_{\calC\calC}) V_{\calC}^{+\top} \\
V_{\calC}^{+}(\mathbb{I}- K_{\calC\calC}(\mathbb{P}_{V^\perp_\calC} K_{\calC\calC} \mathbb{P}_{V^\perp_\calC} )^+)C^\top   & -V_{\calC}^\top (K_{\calC\calC}-K_{\calC\calC}(\mathbb{P}_{V^\perp_\calC} K_{\calC\calC} \mathbb{P}_{V^\perp_\calC} )^+K_{\calC\calC})V_{\calC}^{+\top}\\
\end{pmatrix}.
\]
Next, by using Theorem~\ref{thm:implicit_reg}, it holds that
$
\mathbb{E}_{\calC}[T(\calC)] = \bigl(\begin{smallmatrix}
 K+\mathbb{I} &  V\\
V^\top & 0\\
\end{smallmatrix}\bigr)^{-1},
$
where the inverse matrix can be calculated thanks to  Lemma~\ref{lemma:proj_inverse_blocks}, as follows:
\[
\begin{pmatrix}
 K+\mathbb{I} &  V\\
V^\top & 0\\
\end{pmatrix}^{-1} = \begin{pmatrix}
 (\mathbb{P}_{V^\perp} (K+\mathbb{I)} \mathbb{P}_{V^\perp} )^+    &  (\mathbb{I}- (\mathbb{P}_{V^\perp} (K+\mathbb{I}) \mathbb{P}_{V^\perp} )^+ K) V^{+\top} \\
V^{+}(\mathbb{I}- K(\mathbb{P}_{V^\perp} (K+\mathbb{I}) \mathbb{P}_{V^\perp} )^+ )  & -V^+ (K+\mathbb{I}-K(\mathbb{P}_{V^\perp} (K+\mathbb{I}) \mathbb{P}_{V^\perp} )^+K)V^{+\top}\\
\end{pmatrix}.
\]
The desired result follows by identifying the different terms in the above block matrix, and by using the simply identity $(\mathbb{P}_{V^\perp} (K+\mathbb{I)} \mathbb{P}_{V^\perp} )^+ = (\widetildeK + \mathbb{P}_{V^\perp})^+.$
This completes the proof.
\end{proof}
The above cumbersome expressions in~\eqref{eq:id1} and~\eqref{eq:id2}  do not seem at first sight to have a straightforward interpretation. Nonetheless, for a special choice of partial projection DPP, the identities in~\eqref{eq:id1} and~\eqref{eq:id2} can be used in the context of random design regression in the same spirit as rescaled volume sampling~\cite{Derezinski2019UnbiasedEF}, as we discuss in the following interlude subsection.

\subsection{
Interlude: an unbiased estimator for random design regression with a DPP\label{sec:randomDesign}} 
In order to illustrate the interest of~\eqref{eq:id1} and~\eqref{eq:id2} in  Proposition~\ref{prop:ExtendedExpectedpinv}, we consider the connection between extended $L$-ensembles, volume sampling~\cite{VolumeRescaled} and projection DPPs (see e.g.,~\cite{Belhadji}). This yields a partial projection DPP extending the well-known volume sampling method, which has often been considered in the literature for example in randomized linear algebra~\cite{AvronBoutsidis}. Another related approach called `proportional volume sampling' has been discussed in~\cite{Poinas} for general Polish spaces, while a related Bayesian approach is given in~\cite{derezinski2020bayesian}.  The  generalization presented here can be used in order to find discrete random designs for parametric regression problems, as we explain below.
\paragraph{A random-size volume sampling} First, we recall the definition of volume sampling.
\begin{definition}[volume sampling]
Let $V\in \mathbb{R}^{n\times p}$ be a matrix with full column rank. The probability to volume sample a subset $\calC_0\sim \Vol_k(V)$ of size $k\geq p$ is
\[
\Pr(\calC_0) =\frac{\det(V_{\calC_0}^\top V_{\calC_0})}{{n-p\choose k-p}\det(V^\top V)}.
\]
For $k=p$, volume sampling is a projection DPP with marginal kernel $V(V^\top V)^{-1}V^\top$.
\end{definition}
Next, consider the partial projection $DPP(t\mathbb{I}, V)$ for $t>0$ corresponding to the marginal kernel 
\[
P = q \mathbb{P}_{V^\perp} + \mathbb{P}_V \text{ where } q=\frac{t}{1+t}.
\]
This process is in fact a rescaling of volume sampling, that is,
\begin{align}
    \Pr(\mathcal{Y} =\calC) =\frac{\det\left( t\mathbb{I}_{|\calC|}\right)}{ \det\left( \mathbb{P}_{V} + (1+t)\mathbb{P}_{V_\perp}\right)}  \frac{\det(V_\calC^\top V_\calC)}{\det(V^\top V)} =  q^{|\calC|-p}\left(1-q\right)^{n-|\calC|}\times\frac{\det(V_\calC^\top V_\calC)}{\det(V^\top V)},\label{eq:BernRescaledVolumeSampling}
\end{align}
where the first equality  uses the formula~\eqref{eq:LensembleBis} for the probability mass function of $\calC$.
There is a clear analogy with the volume sampling method of~\cite[Eq. (1)]{VolumeRescaled}, although the sample size of~\eqref{eq:BernRescaledVolumeSampling} is here a random variable which satisfies
\begin{equation}
\frac{\mathbb{E}_\calC[|\calC|] -p}{n-p} = \frac{t}{1+t}, \text{ for } \calC\sim DPP(t\mathbb{I}, V)\label{eq:RandomSize}
\end{equation}
while volume sampling always returns subsets of a fixed size. In the same spirit as in~\cite{VolumeRescaled}, we propose below a sampling strategy based on a combination of volume sampling and a Bernouilli process, and whose correctness is discussed in Proposition~\ref{prop:correctness}. 
\begin{definition}[Bernouilli process]
Let $0\leq q\leq 1$ and $m$ be a positive integer. A Bernouilli$(q)$ process over $\{0,1\}^m$ is a sequence of i.i.d. Bernouilli random variables with success probability $q$, so that the probability to observe any sequence $(s_1, \dots, s_m)$ is $q^{k} (1-q)^{m-k}$ where $k$ is the number of successes in the sequence.
\end{definition}
Commonly, we define a Bernouilli process over an ordered sequence of $m$ items by selecting an item of its corresponding entry in $(s_1, \dots, s_m)\in \{0,1\}^m$ is equal to unity. A finite Bernouilli$(q)$ process is a DPP with marginal kernel $P=q \mathbb{I}$ for $0\leq q\leq 1$. Naturally, the expected sample size of Bernouilli$(q)$ over a set of $m$ items is $mq$.
\begin{proposition}[two-stage sampling method for~\eqref{eq:BernRescaledVolumeSampling}]\label{prop:correctness}
A sample distributed according to $DPP(t\mathbb{I}, V)$ can be obtained as follows:\begin{enumerate}
    \item draw a volume sample $\calC_0\sim \Vol_p(V)$ of the set $[n]$,
    \item draw a sample  $\mathcal{R}\subseteq [n]\setminus \calC_0$ according to Bernouilli$\left(\frac{t}{t+1}\right)$,
\end{enumerate} and return $\calC = \calC_0\cup \mathcal{R}$.
\end{proposition}
\begin{proof}
Let $\calC\subseteq [n]$ such that $|\calC|\geq p$. Consider all the decompositions  $\calC = \calR \cup \calC_0$ where $\calR = \calC\setminus \calC_0$, for all the subsets $\calC_0$ such that $|\calC_0| = p$. Notice that $|\mathcal{R}| = |\calC|-p$.
Then, the probability that $\calC$ is obtained by combining a volume sample $\calC_0$ with a sample $\mathcal{R}$ drawn by a Bernouilli process on $[n]\setminus \calC_0$ is
\[
\Pr(\calC) = \sum_{\calC_0\subseteq\calC :|\calC_0| = p} \underbrace{q^{|\calC|-p}\left(1-q\right)^{n-|\calC|}}_{\Pr(\calR = \calC\setminus \calC_0|\calC_0)} \underbrace{\det(V_{\calC_0}^\top V_{\calC_0})/\det(V^\top V)}_{\Pr(\calC_0) \text{(volume sampling)}} =  q^{|\calC|-p}\left(1-q\right)^{n-|\calC|}\det(V_{\calC}^\top V_{\calC}),
\]
where the second equality follows from the Cauchy-Binet identity $\sum_{\calC_0\subseteq\calC :|\calC_0| = p}\det(V_{\calC_0}^\top V_{\calC_0}) = \det(V_{\calC}^\top V_{\calC})$ (see, Lemma~\ref{lem:Cauchy-Binet} in Appendix).
\end{proof}
The following remarks are in order. For $\calC\sim DPP(t\mathbb{I}, V)$, Proposition~\ref{prop:ExtendedExpectedpinv} reduces to the simple identity
$
    \mathbb{E}_\calC[C \mathbb{P}_{V_\calC^\perp} C^\top]=\frac{t}{1+t}  \mathbb{P}_{V^\perp}.
$
 By using the expression of the marginal kernel of $DPP(t\mathbb{I}, V)$, an elementary manipulation yields the following expression
\begin{equation}
    \mathbb{E}_\calC[C V_\calC (CV_\calC)^+]= V V^+,\label{eq:ExpPproj}
\end{equation}
where $(CV_\calC)^+ = V_\calC^+ C^\top$.
This can be interpreted as follows: the expectation of the projector onto the column space of $C V_\calC$ is the projector onto the column space of $V$.
An identity analogous to~\eqref{eq:ExpPproj} has been obtained in the framework of matroids in~\cite{Lyons} which revisited a related identity in~\cite[Thm 1]{Maurer}.
In the same spirit as in~\eqref{eq:ExpPproj}, the identities in Proposition~\ref{prop:ExtendedExpectedpinv} simplify in the special case $\calC\sim DPP(t\mathbb{I}, V)$ to \begin{align}
      \mathbb{E}_\calC\left[(CV_\calC)^+\right] = V^+,\text{ and }\mathbb{E}_\calC[\underbrace{ (CV_\calC)^+ (CV_\calC)^{+\top}}_{(V_\calC^\top V_\calC)^{-1}}] = \frac{n-p}{\mathbb{E}_\calC[|\calC|] -p}\times \underbrace{V^+ V^{+\top}}_{(V^\top V)^{-1}},\label{eq:alphaExp}
\end{align}
where we used the expression relating $t>0$ with the expected subset size~\eqref{eq:RandomSize}. The above expressions in~\eqref{eq:alphaExp} can be used define an unbiased estimator for subsampled regression problems over a smaller set of inputs and outputs, as we explain below.
\paragraph{Discrete random design regression}  Let $\bm{y}\in \mathbb{R}^n$ and consider the (full-fledged) parametric regression problem 
\[
\bm{\beta}^\star = \arg\min_{\bm{\beta}}\|V \bm{\beta} - \bm{y}\|_2^2,
\]
with $V\in \mathbb{R}^{n\times p}$, and whose solution writes $\bm{w}^\star = V^{+}\bm{y}$. It is known that solving a smaller regression problem with sampled input-output pairs $(V_\calC, \bm{y}_\calC)$ regardless of the outputs for $\calC\subseteq [n]$ can lead to a biased estimator of the full regressor~\cite{JMLR:Warmuth}, for instance if $\calC$ is sampled uniformly at random. However, for $\calC\sim DPP(t\mathbb{I}, V)$  the estimator $\hat{\bm{\beta}}(\calC) = V_\calC^{+}\bm{y}_\calC$ of the quantity $V^{+}\bm{y}$ is unbiased, while the  matrix variance writes \begin{equation}
\mathbb{E}[(CV_\calC)^+ (CV_\calC)^{+\top}] -\mathbb{E}[(CV_\calC)^+]\mathbb{E}[(CV_\calC)^{+}]^\top= \frac{n-\mathbb{E}[|\calC|]}{\mathbb{E}[|\calC|]-p} \times V^+ V^{+\top},\label{eq:variance}
\end{equation} thanks to~\eqref{eq:alphaExp}.  A main difference with the designs of~\cite{JMLR:Warmuth,Derezinski2019UnbiasedEF} obtained with a fixed-size rescaled volume sampling is that here: (i) the subset size is random, and (ii) the above variance grows unbounded as $t\to 0$. On the contrary, a large value of $t$ promotes a large expected subset size and a small variance.

Let us recall some related results for volume sampling which yield formulae analogous to~\eqref{eq:alphaExp}.
On the one hand, the identity
$\mathbb{E}[(C V_\calC)^{+}] = V^{+}$ for $\calC\sim \Vol_k(V)$ has been obtained in the context of random design regression  in~\cite{JMLR:Warmuth} and in~\cite[Thm 2.10]{Derezinski2019UnbiasedEF}.
On the other hand,  the result of~\cite[Eq (1)]{NEURIPS2018_2ba8698b} yields the `second moment' $\mathbb{E}[(CV_\calC)^+ (CV_\calC)^{+\top}] = \frac{n-p+1}{k-p +1}V^+ V^{+\top}$ for $\calC\sim \Vol_k(V)$ which has a slightly different form in comparison with our result in~\eqref{eq:alphaExp}. Indeed, volume sampling for $k=p$ yields a finite variance estimator while the variance  of the estimator built with our partial projection DPP grows unbounded as the expected subset size goes to $p$.

As it is discussed in~\cite{Poinas}, it seems that random design methods only outperform existing optimal design methods~\cite{Pronzato} in very specific settings. We leave a more detailed empirical exploration of our random-size volume sampling for further work.
\subsection{Projected Nystr\"om approximation\label{sec:projNys}}
In the light of Proposition~\ref{prop:ExtendedExpectedpinv}, we now define the projected Nystr\"om approximation as a generalization of the common Nystr\"om approximation which was given in the introduction.
\begin{definition}[projected Nystr\"om approximation]\label{def:Nyst}
Let $(K,V)$ be a NNP and $\calC\subseteq [n]$. Let $\widetilde{K} = \mathbb{P}_{V^\perp}K\mathbb{P}_{V^\perp}$. The projected Nystr\"om approximation\footnote{The projected Nystr\"om approximation is denoted by a $\widetilde{L}$ (for Low rank matrix), although it is not used in the construction of an extended $L$-ensemble. The difference should be clear from the context.} of $\widetilde{K}$ is
\[
\widetilde{L(\mathcal{C})} = \mathbb{P}_{V^\perp}K_\calC^\top B(\calC)\left(B^\top(\calC) K_{\calC\calC}B(\calC)\right)^{+} B^\top(\calC) K_\calC \mathbb{P}_{V^\perp},
\]
where $B(\calC)\in \mathbb{R}^{k\times (k-p)}$ be a matrix whose columns are an orthonormal basis of $(V_\calC)^\perp$.
\end{definition}
Several remarks are in order. First, the pseudo-inverse in the above definition can be replaced by a matrix inverse almost surely if $\calC\sim DPP(K,V)$.
Second, it is worth emphasizing that, in Definition~\ref{def:Nyst}, the submatrices $K_\calC$ and $K_{\calC\calC}$ are sufficient to construct the projected Nystr\"om approximation while $\widetildeK$ should not be explicitly constructed. Therefore, the projected Nystr\"om approximation is also promissing in order to solve problems where the kernel matrix is too large compared to the computer memory.
Third, the projected Nystr\"om approximation given in Definition~\ref{def:Nyst} satisfies the following desirable property: the submatrices
$\widetilde{L(\mathcal{C})}_{\calC\calC}$ and $\widetildeK_{\calC \calC}$ match in the appropriate subspace, that is,
$
B^\top (\calC) \widetilde{L(\mathcal{C})}_{\calC\calC} B(\calC) = B^\top (\calC)  \widetildeK_{\calC \calC}  B(\calC),
$
for $\calC\subseteq [n]$. This is a consequence of Corollary~\ref{corol:Nyst_error} given below.
\begin{corollary}[Nystr\"om approximation error]\label{corol:Nyst_error}
Let $\calC\sim DPP(K/\lambda,V)$ with $\lambda>0$. Let $B(\calC)\in \mathbb{R}^{|\calC|\times (|\calC|-p)}$ be a matrix whose columns are an orthonormal basis of $(V_\calC)^\perp$. Then, we have the following identities
\begin{itemize}
    \item[(i)] $0\preceq \widetilde{L(\mathcal{C})} =\widetildeK_\calC^\top B(\calC)\left(B^\top(\calC) \widetilde{K}_{\calC\calC}B(\calC)\right)^{-1} B^\top(\calC) \widetildeK_\calC \preceq \widetildeK$,\\
    \item[(ii)]
    $\mathbb{E}[\widetildeK - \widetilde{L(\mathcal{C})}] = \lambda \widetildeK (\widetildeK + \lambda\mathbb{I})^{-1}\preceq \lambda \mathbb{I}.$
\end{itemize}
\end{corollary}
Notice that Corollary~\ref{corol:Nyst_error} shows that the expected error of the approximation naturally decreases if the expected number of sampled landmarks increases, that is, as $\lambda>0$ goes to zero. The proof of this result merely follows from Proposition~\ref{prop:ExtendedExpectedpinv}.
\begin{proof}[Proof of Corollary~\ref{corol:Nyst_error}]
(i) It is easy to check that $ \widetilde{L(\mathcal{C})}\succeq 0$. The first identity is simply obtained by using Lemma~\ref{lem:projC} in Appendix. To show that $ \widetilde{L(\mathcal{C})}\preceq \widetildeK$, it is sufficient to show the following fact: for all $\epsilon>0 $,
\begin{equation}
    \widetildeK_\calC^\top B(\calC)\left(B^\top(\calC) \widetilde{K}_{\calC\calC}B(\calC) + \epsilon \mathbb{I}\right)^{-1} B^\top(\calC) \widetildeK_\calC\preceq  \widetildeK,\label{eq:esp_Nyst}
\end{equation}
since by taking the limit $\epsilon\to 0$, we obtain $ \widetilde{L(\mathcal{C})}\preceq \widetildeK$.
To prove the inequality~\eqref{eq:esp_Nyst}, we define $\widetildeK = A A^\top$ and, thanks to the push-through identity (see Lemma~\ref{lemma:push} in Appendix), we show that
\[
A_\calC^\top B(\calC)\left(B^\top(\calC)A_\calC A_\calC^\top B(\calC) + \epsilon \mathbb{I}\right)^{-1} B^\top(\calC) A_\calC = \left(A_\calC^\top B(\calC)B^\top(\calC)A_\calC  + \epsilon \mathbb{I}\right)^{-1} A_\calC^\top B(\calC)B^\top(\calC) A_\calC \preceq \mathbb{I},
\]
where $A_\calC = C^\top A$.
(ii) The second identity follows from Proposition~\ref{prop:ExtendedExpectedpinv}. Consider first the case $\lambda = 1$ without loss of generality. 
The expectation of the projected Nystr\"om approximation reads 
\[
\mathbb{E}[\widetilde{L(\mathcal{C})}]  =\widetildeK \mathbb{E}_\calC\left[ C B(\calC)\left(B^\top(\calC) \widetilde{K}_{\calC\calC}B(\calC)\right)^{-1} B^\top(\calC)C^\top\right] \widetildeK = \widetildeK (\widetildeK + \mathbb{I})^{-1} \widetildeK,
\]
where  Proposition~\ref{prop:ExtendedExpectedpinv} was used for the last equality. This gives $\widetildeK - \mathbb{E}[\widetilde{L(\mathcal{C})}] = \widetildeK (\widetildeK + \mathbb{I})^{-1}$. The final result is obtained by replacing $\widetildeK$ by $\widetildeK/\lambda$.
\end{proof}
Corollary~\ref{corol:Nyst_error} only considers matrices which are only non trivial in the subspace $V_\perp$. 
Interestingly, the projected Nystr\"om approximation can be written as $\widetilde{L(\mathcal{C})} = \mathbb{P}_{V_\perp}L(\mathcal{C})\mathbb{P}_{V_\perp},
$
where, in the notations of Proposition~\ref{prop:ExtendedExpectedpinv}, the `unprojected' Nystr\"om approximation writes
\[
L(\mathcal{C}) = K I(\calC)K.
\]
The difference between $K$ and the latter matrix also satisfies simple identities given below.
\begin{corollary}
Let $\calC\sim DPP(K/\lambda,V)$ with $\lambda>0$. Then,  we have
\begin{align*}
    \mathbb{E}_\calC[ K-L(\calC)] &=  K - K(\widetildeK + \lambda\mathbb{P}_{V^\perp})^+K \\
    \mathbb{E}_\calC[(CV_\calC)^+ (K-L(\calC))] &= V^+ \mathbb{E}_\calC[ K-L(\calC)]\\
    \mathbb{E}_\calC[(CV_\calC)^+ \left(K-L(\calC)\right)(CV_\calC)^{+\top}] &= \lambda V^+ V^{+\top} + V^+\mathbb{E}_\calC[ K-L(\calC)] V^{+\top}.
\end{align*}
\end{corollary}
Additional properties of the projected Nystr\"om approximation error $\widetildeK - \widetilde{L(\mathcal{C})}$ can be obtained by merely replacing $K$ by $\widetildeK$ in the above expressions.
\begin{proof}
The results simply follow from the identities obtained respectively by: (i) multiplying~\eqref{eq:id0} by $K$ on the left and on the right, (ii) multiplying~\eqref{eq:id1} by $K$ on the right, (iii) reformulating~\eqref{eq:id2} by using~\eqref{eq:id0}.
\end{proof}

\paragraph{Empirical results for matrix Nystr\"om approximation}
We illustrate here the effect of extended $L$-ensemble sampling versus uniform sampling for Nystr\"om matrix approximation on the UCI benchmark data sets\footnote{\url{https://archive.ics.uci.edu/ml/index.php}}: \texttt{Breast Cancer}, \texttt{Mushroom} and \texttt{Wine Quality}.
\begin{figure}[t]
		\centering
	\begin{subfigure}[b]{0.32\textwidth}
			\includegraphics[width=\textwidth, height= 0.8\textwidth]{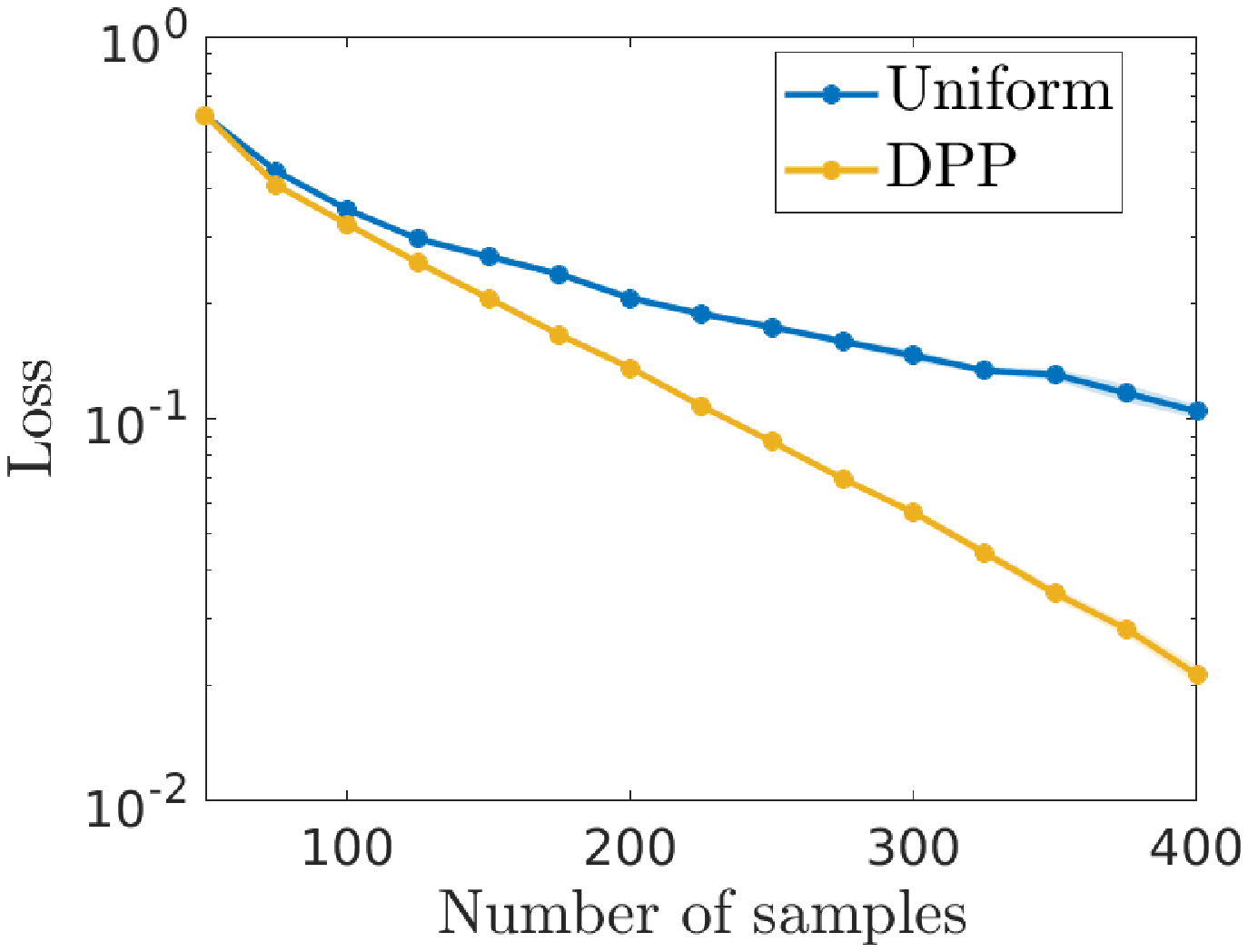}
			\caption{\texttt{Breast Cancer}}
			\label{fig:Nystrom_approx1}
		\end{subfigure}
		\begin{subfigure}[b]{0.32\textwidth}
			\includegraphics[width=\textwidth, height= 0.8\textwidth]{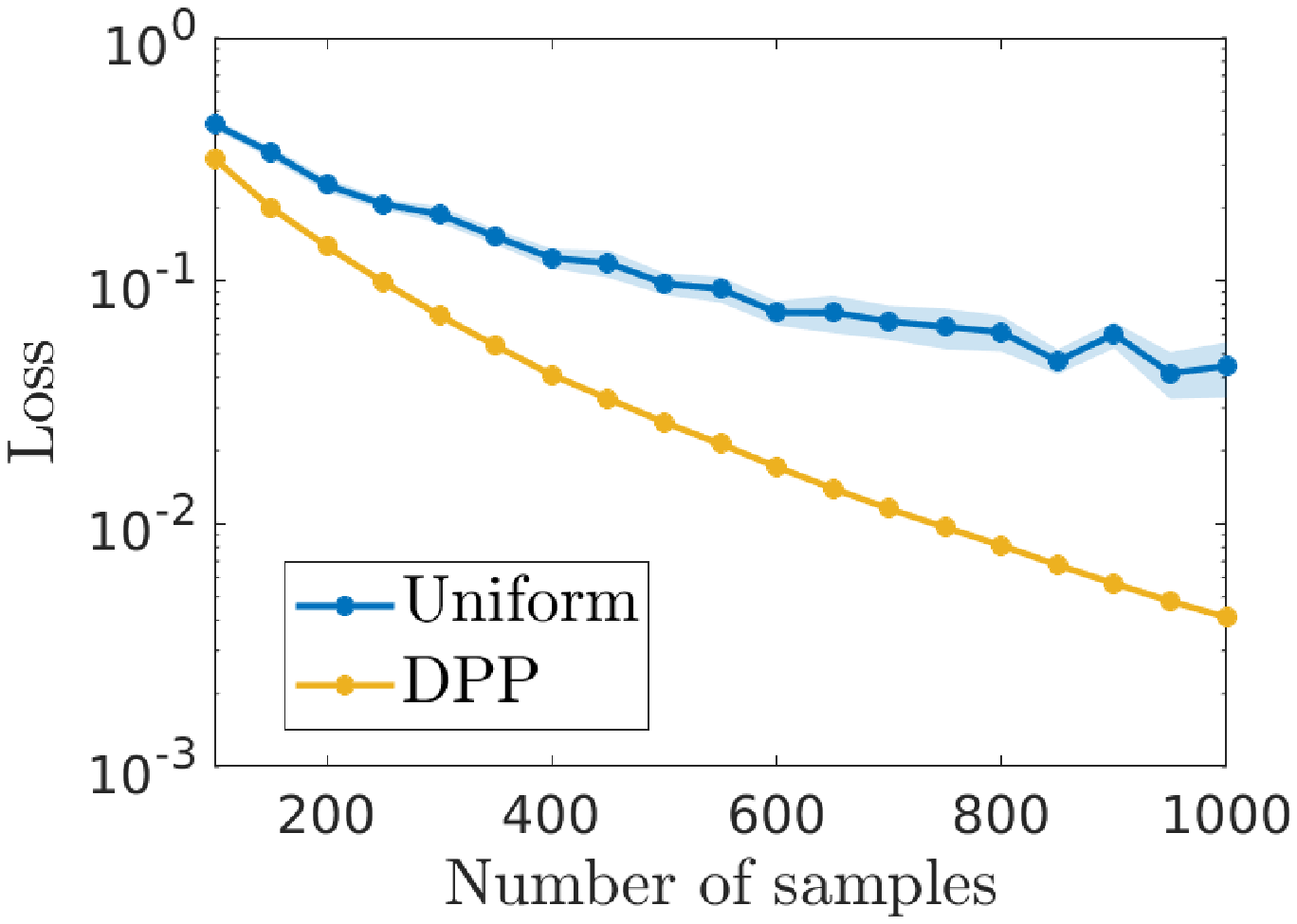}
			\caption{\texttt{Mushroom}}
			\label{fig:Nystrom_approx2}
		\end{subfigure}
		\begin{subfigure}[b]{0.32\textwidth}
			\includegraphics[width=\textwidth, height= 0.8\textwidth]{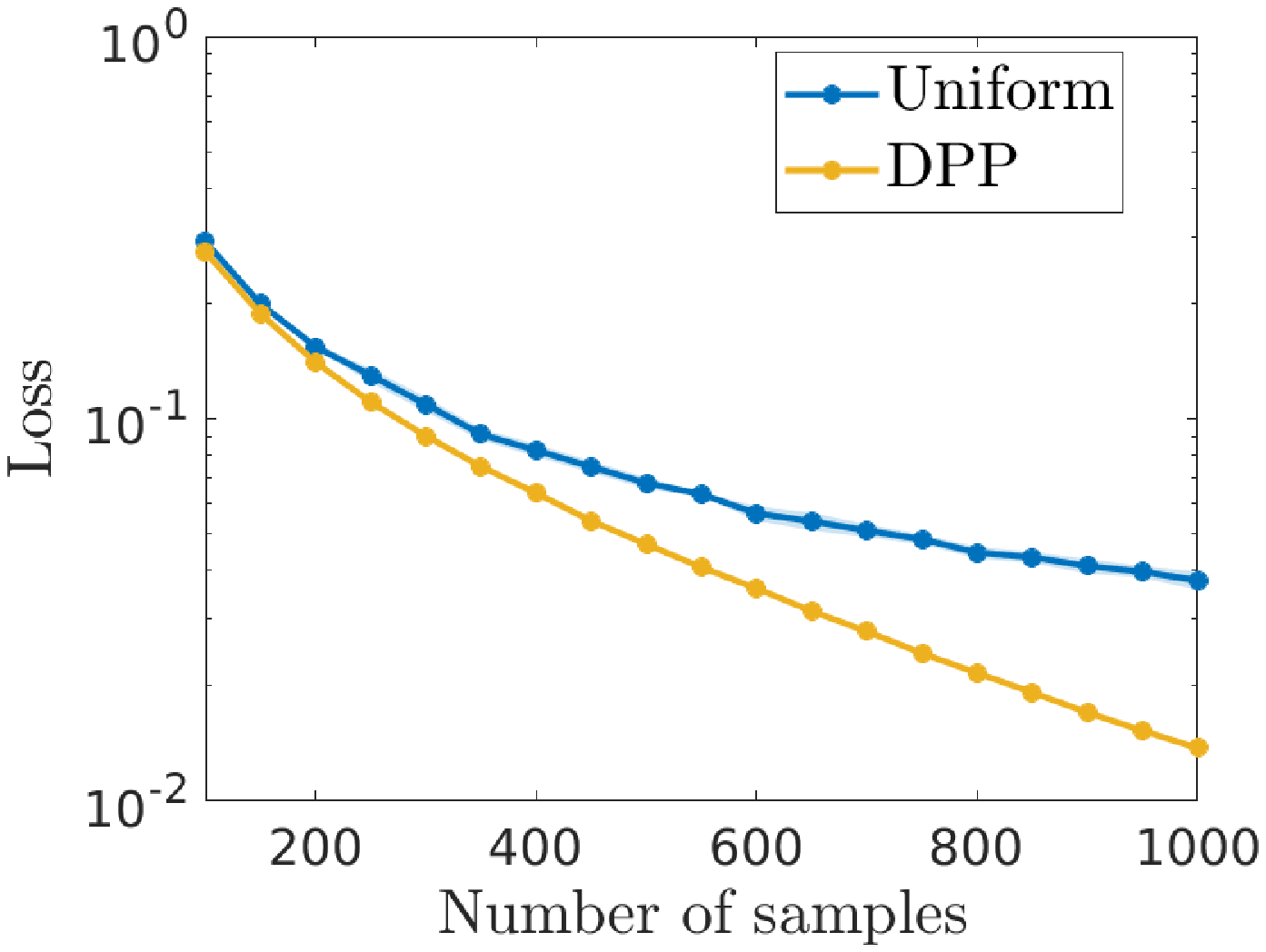}
			\caption{\texttt{Wine Quality}}
			\label{fig:Nystrom_approx3}
		\end{subfigure}
	\caption{The relative Nystr\"om approximation error~\eqref{eq:NystError} as a function of the number landmarks using uniform vs extended $L$-ensemble DPP sampling.}\label{fig:Nystrom_approx}
\end{figure} 
The data sets are standardized, and the performance is measured by the relative Frobenius norm of the error:  
\begin{equation}
    \|\widetilde{K} - \widetilde{L(\mathcal{C})}\|_F/\|\tilde{K}\|_F.\label{eq:NystError}
\end{equation}
We again use a Gaussian kernel $k(\bmx,\bmxp)  = \exp(-\|\bmx-\bmxp\|_2^2/\sigma^2)$ and linear regression component: $V = [X \enskip \bm{1}_n]$ where $X = [\bmx_1 \dots \bmx_n]^\top\in \mathbb{R}^{n\times d}$. The bandwidth is determined using the median heuristic~\cite{gretton2012kernel} defined in Section~\ref{sec:EmpGaussian}. The simulation is repeated 10 times and the error bars show the $97.5\%$ confidence interval. The results are displayed in Figure~\ref{fig:Nystrom_approx}. Extended DPP $L$-ensemble sampling gives a more accurate Nystr\"om approximation. We emphasize that these results are illustrative and that we do not claim that the aforementioned semi-parametric setting is the most suitable for these three data sets.
The Nystr\"om approximation of the penalized regression problem can now be discussed by using the matrix Nystr\"om approximation that we just introduced. 
\subsection{Nystr\"om approximation of regularized regression\label{sec:Nyst_penreg}}
Given a subset $\calC = \{{i_1}, \dots, {i_{k}}\}\subset [n]$, the Nystr\"om approximation allows to reduce the number of parameters from $n+p$ to $k + p$ without overlooking data points. To do so in the setting of this paper, we propose to solve a simplified problem which differs from~\eqref{eq:pen_reg} by the domain of the minimization, i.e., we introduce the following problem
\begin{equation}
    \min_{f\in\mathcal{H}_N}\frac{1}{n} \sum_{i=1}^n \left(y_i -f(\bmx_i)\right)^2 +\gamma J(f),\tag{$\text{NysPLS}$}  \label{eq:pen_reg_Nys}
\end{equation}
where the domain is defined by
\[
\mathcal{H}_N \triangleq \left\{f(\bmx) = \sum_{\ell =1}^{k} \alpha_{i_\ell} k(\bmx,\bmx_{i_\ell}) + \sum_{m=1}^p \beta_m p_m(\bmx) \text{ s.t. } \sum_{\ell=1}^{k}\alpha_{i_\ell} p_m(\bmx_{i_\ell}) = 0 \text{ for all } 1\leq m\leq p \right\}
\]
with $k= |\calC|$. The domain of the optimization problem~\eqref{eq:pen_reg_Nys} now includes only finite linear combinations of $k(\cdot, \bmx_{i_\ell})$ for $i_\ell\in \calC$, with a specific condition on the coefficients, whereas the domain of the `full' optimization problem~\eqref{eq:pen_reg} includes possibly infinite linear combinations. In analogy with~\eqref{eq:LinearSystem}, this condition yields afterwards the constraint $V^\top_{\calC} \bma' = 0$ where $\bma'= [\alpha_{i_1} \dots \alpha_{i_{k}}]^\top \in \mathbb{R}^{k}$. Here, $V =[ p_m(\bmx_i)]_{im}$ is a $n\times p$ matrix.

The solution of~\eqref{eq:pen_reg_Nys}  involves a $(k-p)\times (k-p)$ linear system that we write below, after introducing useful notations.
Let $B(\calC)\in \mathbb{R}^{k\times (k-p)}$ be a matrix whose columns are an orthonormal basis of $(V_\calC)^\perp$, and that is such that $\mathbb{P}_{V_\calC^\perp} =B(\calC) B^\top(\calC)$. Then, after elementary manipulations, the system yields
\begin{align}
\bma'^\star &= B(\calC)\left( B^\top(\calC)K_\calC \mathbb{P}_{V^\perp} K_\calC^\top B(\calC)  + n\gamma B^\top(\calC) K_{\calC\calC} B(\calC) \right)^{-1}B^\top(\calC) K_\calC\mathbb{P}_{V^\perp} \bm{y},\label{eq:LinSyst}\\
\bmb^\star &= (V^\top V)^{-1} V^\top \left( \bm{y} - K_\calC^\top B(\calC) B^\top(\calC)\bma'^\star\right).\nonumber
\end{align}
The details of this derivation are given in Appendix. The above linear system involves positive definite matrices and can be solved conveniently by the conjugate gradient algorithm, possibly after a preconditioning which is described in Appendix.

\paragraph{In-sample estimator}
Also, it is straightforward to determine the in-sample estimator $\hat{\bm{z}}_{N} = K^\top_\calC \bma'^\star + V^\top \bmb^\star$, which is then given, in terms of the projected Nystr\"om approximation, as follows
\[
\hat{\bm{z}}_{N} = \widetilde{L(\calC)} \left(\widetilde{L(\calC)}+n\gamma \mathbb{I}_n\right)^{-1}\mathbb{P}_{V^\perp}\bm{y} + \mathbb{P}_{V}\bm{y}.
\]
Importantly, this estimator can be formally obtained from the estimator of the `full' problem by replacing $\widetildeK$ by $\widetilde{L(\calC)}$. Notice that the computation of $\bm{\beta}^\star$ only requires solving a $p\times p$ linear system.




\subsection{Bound on the expected risk\label{sec:BoundRisk}}
Recall our data assumption $y_i = f(\bmx_i) + \epsilon_i$ where $\epsilon_i$ denotes i.i.d. $\mathcal{N}(0,\sigma^2)$ noise with $1\leq i\leq n$. For convenience, define $z_i = f(\bmx_i)$ for all $1\leq i\leq n$. 
The expected risk of the full regression problem is 
$
\mathcal{R}(\hat{\bm{z}}) = \mathbb{E}_\epsilon\|\hat{\bm{z}}-\bm{z}\|_2^2.
$
We find an upper bound for the expected risk obtained thanks to the projected Nystr\"om approximation, i.e., $\mathcal{R}(\hat{\bm{z}}_{N}) = \mathbb{E}_\epsilon\|\hat{\bm{z}}_{N}-\bm{z}\|_2^2$.
In the spirit of the kernel ridge regression and Theorem~2.5 in~\cite{fanuel2020diversity}, we can prove the following stability bound on expectation.
\begin{theorem}[stability of the expected risk]\label{thm:Bound}
Let $\calC\sim DPP(K/\lambda,V)$ with $\lambda >0$. Then, it holds that
\[
\mathbb{E}_{\calC}\left[\sqrt{\frac{\mathcal{R}(\hat{\bm{z}}_{N})}{\mathcal{R}(\hat{\bm{z}})}}\right]\leq 1+ \frac{\lambda}{n\gamma} d_{\rm eff}(\widetilde{K}/\lambda),
\text{ with } d_{\rm eff}(\widetilde{K}/\lambda) = \Tr\left(\widetilde{K} (\widetilde{K}+ \lambda\mathbb{I}_n)^{-1} \right).
\]
\end{theorem}
\begin{proof}
The proof follows exactly the same lines as in~\cite[Theorem 3]{fanuel2020diversity}, where $\widetildeK$ and $\widetilde{L(\calC)}$ replace the \emph{psd} kernel matrix and its common Nystr\"om approximation. This can be done since $\widetilde{L(\calC)}$ formally satisfies the same identity as the common Nystr\"om approximation in kernel ridge regression case when the sampling is done with a $L$-ensemble.
\end{proof}
This result indicates that using the Nystr\"om approximation cannot dramatically deteriorate the risk on expectation. An analogous result for leverage scores sampling holding with high probability can found in~\cite{MuscoMusco} for kernel ridge regression.
We remark that the effective dimension $d_{\rm eff}(\widetilde{K}/\lambda)$ is crucial in many sampling methods (see also, e.g.~\cite{ElAlaouiMahoney}). Typically, a small $\lambda>0$ yields a large expected sample $\mathbb{E}[|\calC|] = p + d_{\rm eff}(\widetilde{K}/\lambda)$ and therefore reduces the magnitude of the upper bound in Theorem~\ref{thm:Bound}.

\subsection{Application: non-linear time series using semi-parametric models}

A typical (embedded) application that requires a small number of parameters is non-linear time series estimation which is a problem of interest in engineering, for instance in the context of system identification, electromechanical systems or for the control of chemical processes. Within the framework of non-linear time series, a common approach consists in estimating a non-linear black-box model to produce accurate and fast forecasts starting from a set of observations. The user usually has some expert knowledge to incorporate into the estimation. This makes the use of semi-parametric regression models especially appealing for systems and control~\cite{EspinozaConf,EspinozaJournal}, while we refer to~\cite{Gao} for an overview of the various applications in finance, climate and environment sciences.
Empirically, it is common to transform time series estimation or system identification problems into regression problems, such as~\eqref{eq:pen_reg}, as we explain below. Therefore, we use this engineering application as a case study for the function estimation framework used in this paper. A recent theoretical analysis of this type of regression frameworks for time series can be found in~\cite{pmlr-v89-mariet19a}.

The time cruciality of industrial applications necessitates models with a small number of parameters, as these have a large impact on the memory requirements and prediction speed in real-time forecasting. We demonstrate the use of DPP sampling for Nystr\"om-based regression. We show that Nystr\"om-based regression~\eqref{eq:pen_reg_Nys} does not have a much lower performance compared to the full system, with a lower memory cost and prediction time. 

In this simulation, we compare the performance of solving~\eqref{eq:pen_reg} and~\eqref{eq:pen_reg_Nys} by using either uniform or DPP sampling. Each model contains a linear parametric part corresponding to the system to be estimated and non-parametric part based on the Gaussian kernel. The non-parametric part can be viewed as a misspecification error.
A simple observation given in Proposition~\ref{prop:Gauss} justifies a decomposition with a separation of variables between the linear and non-linear components.
\begin{proposition}\label{prop:Gauss}
Let $\bm{v}\in \mathbb{R}^d$ and let $\mathbb{P}_{\bm{v}^\perp}$ be the projector onto the orthogonal of $\bm{v}$. Then, the kernel $k(\mathbb{P}_{\bm{v}^\perp} \bmx, \mathbb{P}_{\bm{v}^\perp} \bmx) = \exp\left(-\| \mathbb{P}_{\bm{v}^\perp} (\bmx - \bmx')\|^2\right)$  is  positive semi-definite on $\mathbb{R}^d$.
\end{proposition}
\begin{proof}
This can be shown thanks to the following result of~\cite[Thm 2.2]{Berg1984}: $\exp(-g(\bmx, \bmx'))$ is positive semi-definite if and only if $g(\bmx, \bmx')$ is negative semi-definite with respect to $1$. Clearly, $g(\bmx, \bmx') = \| \mathbb{P}_{\bm{v}^\perp} (\bmx - \bmx')\|^2$ satisfies $\sum_{i,j=1}^m \alpha_i\alpha_j g(\bmx_i, \bmx_j) \leq 0$ for all finite set of $\bmx_i$ for $1\leq i\leq m$ and $\bm{\alpha}\in \mathbb{R}^m$ such that $\sum_{i=1}^m \alpha_i = 0$.
\end{proof}
\paragraph{Non-linear time series}
Then, in what follows, three systems are defined.\\
\noindent \texttt{System 1:} 
    The first model is a static toy example that is given at time step $t$ with $1\leq t\leq n$ by
    \[\sfy^t = a_2 \sfz_1^t + a_1 + \operatorname{sinc}(\sfx_1^t + \sfx_2^t) + \sfe^t,\]
    with $a_2 = 0.2$, $a_1 = 0.4$ and where superscript $t$ indicates a value obtained at time $t$. The real $\sfy^t$ is the output of the system at time step $t$, which is given by the combination of a linear combination of the real input $\sfz^t$ and a non-linear function of two other real inputs $\sfx_1^t$ and $\sfx_2^t$ at time $t$. The training set is obtained by considering a set of input-output pairs obtained for a sequence of integer time steps $1\leq t\leq n$. The inputs are sampled independently as follows: $\sfx_1$ and $\sfx_2$ are $\mathcal{N}(0,2)$ random variables, whereas $\sfz\sim \mathcal{N}(0,2.5)$, and $\sfe\sim \mathcal{N}(0,0.05)$ is the noise. The training data for the penalized regression problem are $(\bmx_i, y_i)$ with\footnote{Notice that a different font is used for the inputs and output, compared to the $(\bmx_i, y_i)$ pairs.}
    \[
    \bm{x}_i = [\sfz^i\ \sfx_1^i\ \sfx_2^i]^\top \in \mathbb{R}^3  \text{ and } y_i = \sfy^i \text{ for } 1\leq i\leq n.
    \]
     Remark that the time information is not considered here in order to transform the system identification problem into a regression problem, while non-static systems are given below. Let $\bm{x} = [\sfz\ \sfx_1\ \sfx_2]^\top$. The estimated function is of the form 
     \[
     f(\bmx) = \beta_1 p_1(\bmx) + \beta_2 p_2(\bmx) + \sum_{i=1}^n \alpha_i k(\bmx,\bmx_i),
     \]
     where $p_1(\bmx) = 1$ and $p_2(\bmx) = x_{1}$, while the kernel is $ k(\bmx,\bmx') = \exp\left(-\frac{1}{\sigma^2}\left((x_{2}-x'_{2})^2+(x_{3}-x'_{3})^2\right)\right)$, where $x_k$ denotes the $k$-th component of $\bmx$ with $1\leq k \leq 3$. The estimation problem is then cast into the form of~\eqref{eq:pen_reg} since Proposition~\ref{prop:Gauss} indicates that $ k(\bmx,\bmx') $ is \emph{psd}.\\
\noindent\texttt{System 2}: 
     The second model is not static: for integer time steps $1\leq t\leq n$, it reads 
     \[
     \sfy^{t}= a_{1} +  a_{2} \sfy^{t-1} + a_{3} \sfy^{t-2} + 2\operatorname{sinc}\left( \sfx^{t}_{1}+ \sfx^{t}_{2}\right)+\sfe^{t},
     \]
    where $a_1 = 0.3$, $a_2 = 0.2$, $a_3 = 0.1$. It is common to define $\sfy^{0} = \sfy^{-1} = 0$. Also, we consider independent random variables $\sfx_1\sim \mathcal{N}(0,2)$, $\sfx_2\sim \mathcal{N}(0,2)$, and a noise $\sfe\sim \mathcal{N}(0,0.05)$.  The training data for the penalized regression problem are $(\bmx_i, y_i)$ for $1\leq i\leq n$ with
    \[
    \bm{x}_i = [\sfx_1^i\ \sfx_2^i\  \sfy^{i-1}\  \sfy^{i-2}]^\top \in \mathbb{R}^{4} \text{ and } y_i = \sfy^i. 
    \]
    Here, the time series is encoded into the regression problem in such a way that $\bm{x}_i$ contains the previous two time steps.  Let $\bm{x}\in \mathbb{R}^4$. The estimated function is of the form 
    \[
    f(\bmx) = \beta_1 p_1(\bmx) + \beta_2 p_2(\bmx) + \beta_3 p_3(\bmx) + \sum_{i=1}^n \alpha_i k(\bmx,\bmx_i),
    \]
    where $p_1(\bmx) = 1$, $p_2(\bmx) = x_{3}$, $p_3(\bmx) = x_{4}$ and $ k(\bmx,\bmx') = \exp\left(-\frac{1}{\sigma^2}\left((x_{1}-x'_{1})^2+(x_{2}-x'_{2})^2\right)\right).$\\
\noindent\texttt{System 3}: 
    The third model is of the form: 
    \[
    \sfy^{t}=  a_{1} + a_{2} \sfy^{t-1} +  a_{3} \sfy^{t-2} + b_1\operatorname{sinc}\left(\sfu^{t-1}\right) + + b_2\operatorname{sinc}\left(\sfu^{t-2}\right)+\sfe^{t},
    \]
   with $a_1 =0.6$, $a_2 = 0.4$, $a_3 = 0.2$,  $b_1 = 0.7$, $b_2 = 0.6$, $\sfu \sim \mathcal{N}(0,4)$, and the noise $\sfe\sim \mathcal{N}(0, 0.05)$. We also have $\sfy^{0} = \sfy^{-1} = 0$ and $\sfu^{0} = \sfu^{-1} = 0$ by definition. The training data for the penalized regression problem are $(\bmx_i, y_i)$ for $1\leq i\leq n$ with
    \[
    \bm{x}_i = [\sfu^{i-1}\ \sfu^{i-2}\  \sfy^{i-1}\  \sfy^{i-2}]^\top \in \mathbb{R}^{4}
    \text{ and } y_i = \sfy^i.
    \]
    Let $\bm{x}\in \mathbb{R}^4$. The estimated function is of the form 
    \[
    f(\bmx) = \beta_1 p_1(\bmx) + \beta_2 p_2(\bmx) + \beta_3 p_3(\bmx) + \sum_{i=1}^n \alpha_i k(\bmx,\bmx_i),
    \]
    where $p_1(\bmx) = 1$, $p_2(\bmx) = x_{3}$, $p_3(\bmx) = x_{4}$ and $ k(\bmx,\bmx') = \exp\left(-\frac{1}{\sigma^2}\left((x_{1}-x'_{1})^2+(x_{2}-x'_{2})^2\right)\right).$ 

\paragraph{Simulation setting} We take $n=1000$ time steps. The data set is split into a 50/25/25 train, validation and test set. The validation set is used to determine the regularization parameter $\gamma$ and bandwidth $\sigma$. For each model, we measure the parameter identification error, i.e., the mean squared error between the true coefficients $a_1, a_2, a_3$ of the parametric component and their estimates $\beta_1, \beta_2, \beta_3$. 
More importantly, we calculate the prediction error on the test set: $ (1/n_{\mathrm{test}}) \sum_{i=1}^{n_{\mathrm{test}}} (y_i - \hat{f}(\bm{x}_i))^2$. The simulation is repeated 10 times. The results are visualized in Figures \ref{fig:sys_id} and \ref{fig:sys_pred}  where the error bars show the $97.5\%$ confidence interval. Both sampling algorithms are capable of correctly identifying the linear part of the model. Given a number of landmark points, DPP sampling shows better performance than uniform sampling for the prediction error which is the task of practical interest. 

\begin{figure}[ht]
		\centering
	\begin{subfigure}[b]{0.32\textwidth}
			\includegraphics[width=\textwidth, height= 0.8\textwidth]{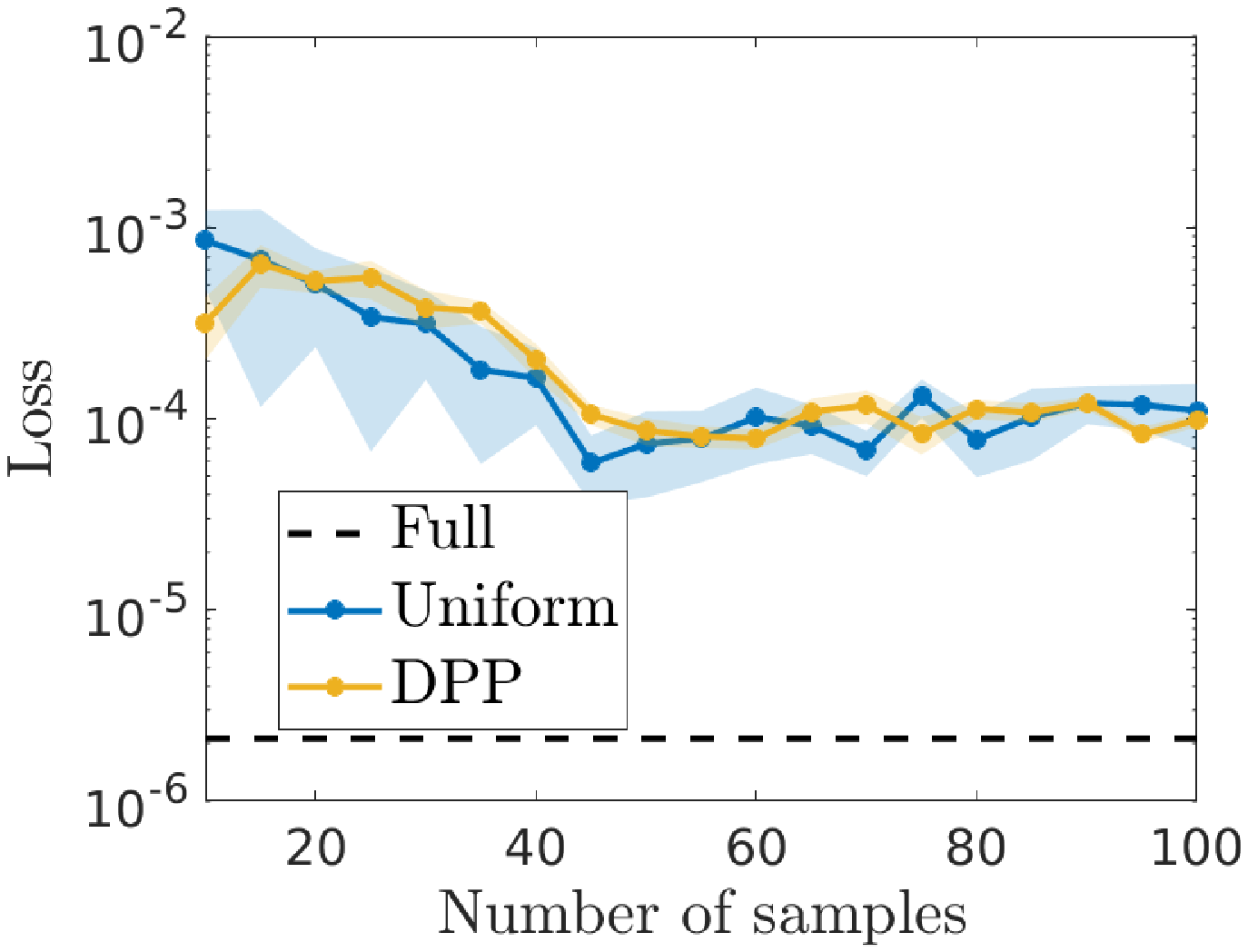}
			\caption{\texttt{System 1}}
			\label{fig:sys_id_1}
		\end{subfigure}
		\begin{subfigure}[b]{0.32\textwidth}
			\includegraphics[width=\textwidth, height= 0.8\textwidth]{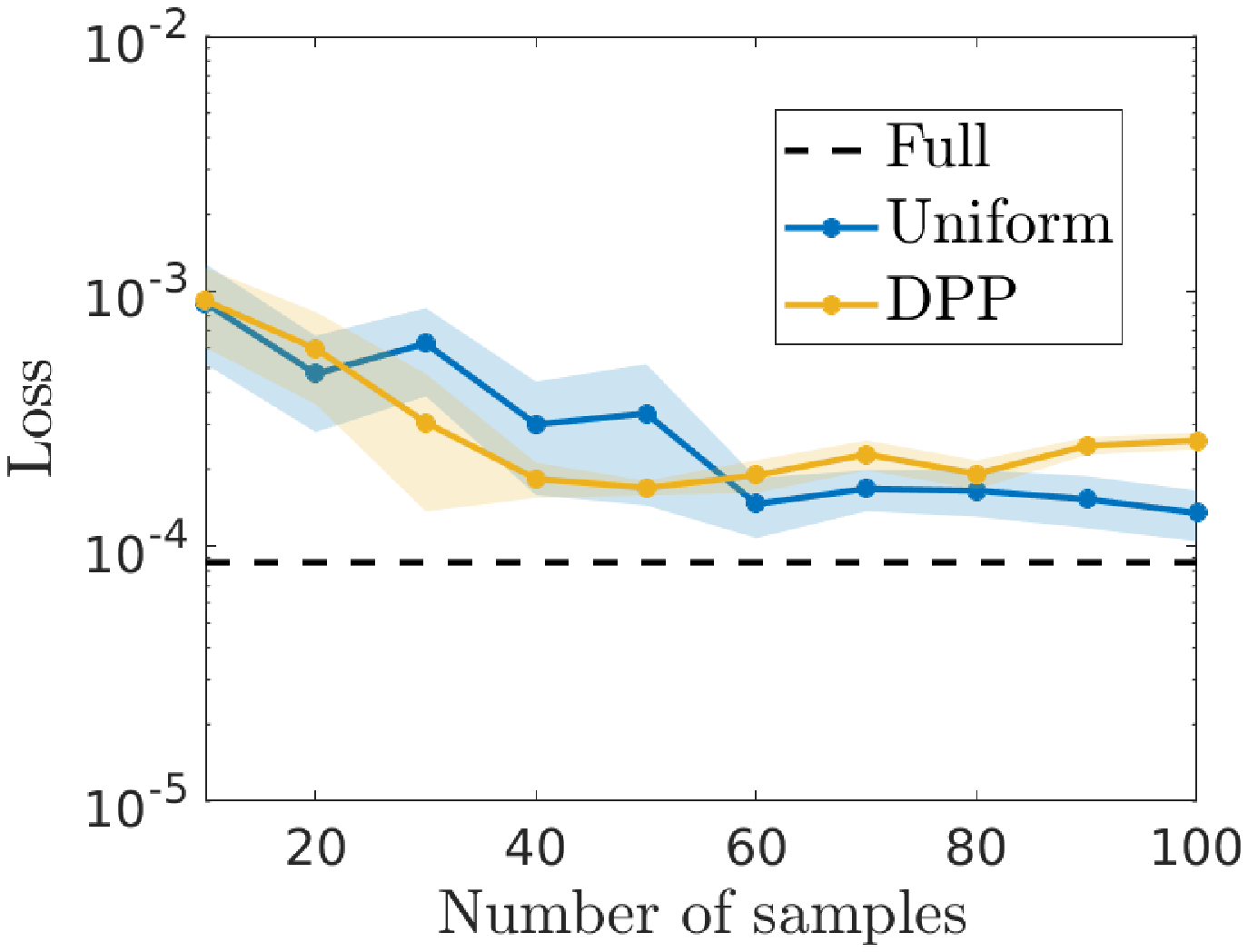}
			\caption{\texttt{System 2}}
			\label{fig:sys_id_2}
		\end{subfigure}
		\begin{subfigure}[b]{0.32\textwidth}
			\includegraphics[width=\textwidth, height= 0.8\textwidth]{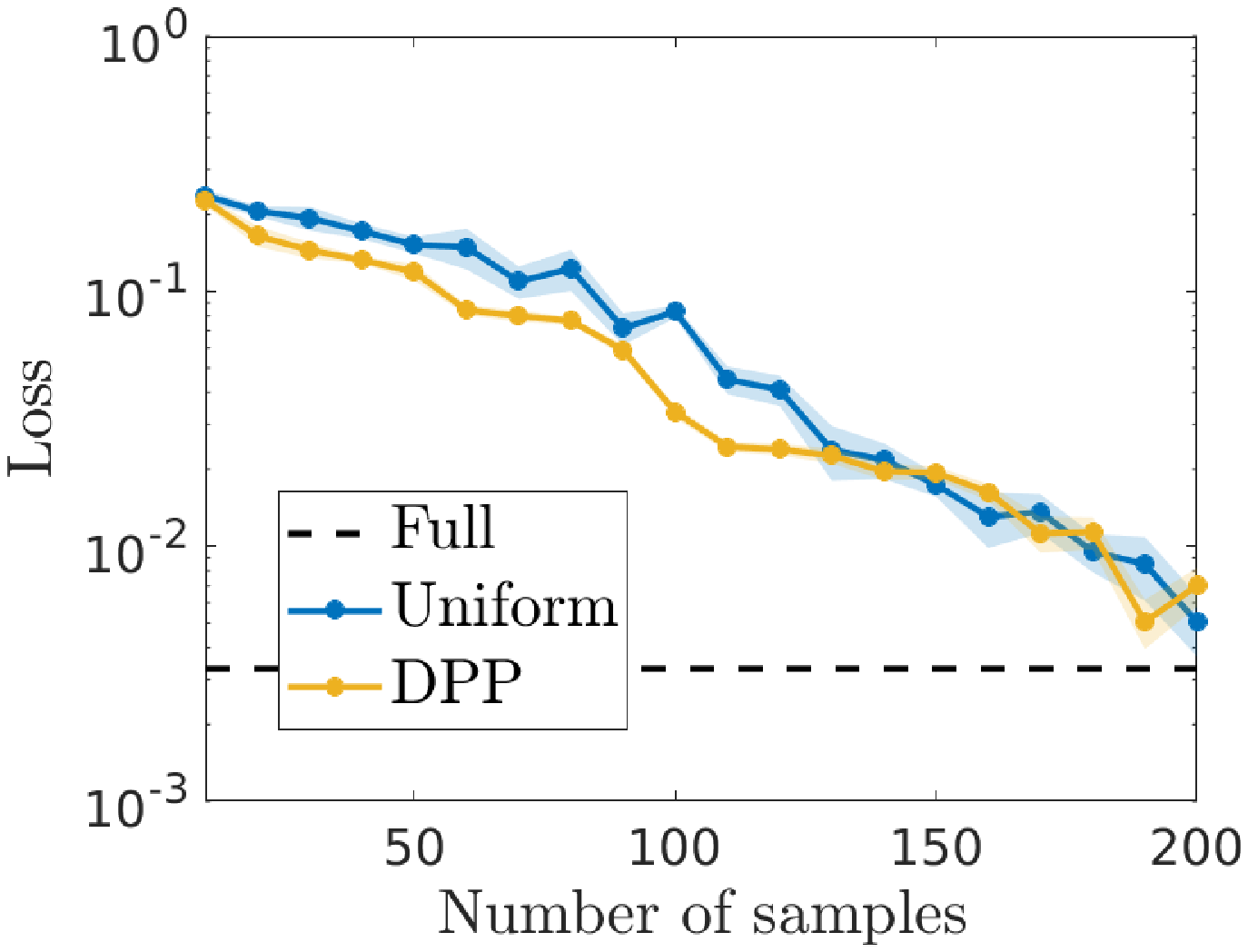}
			\caption{\texttt{System 3}}
			\label{fig:sys_id_3}
		\end{subfigure}
	\caption{The parameter identification error as a function of the number landmarks $|\calC|$ using uniform vs extended $L$-ensemble sampling. Here, the total number of training points is $n=500$.}\label{fig:sys_id}
\end{figure}
\begin{figure}[ht]
		\centering
	\begin{subfigure}[b]{0.32\textwidth}
			\includegraphics[width=\textwidth, height= 0.8\textwidth]{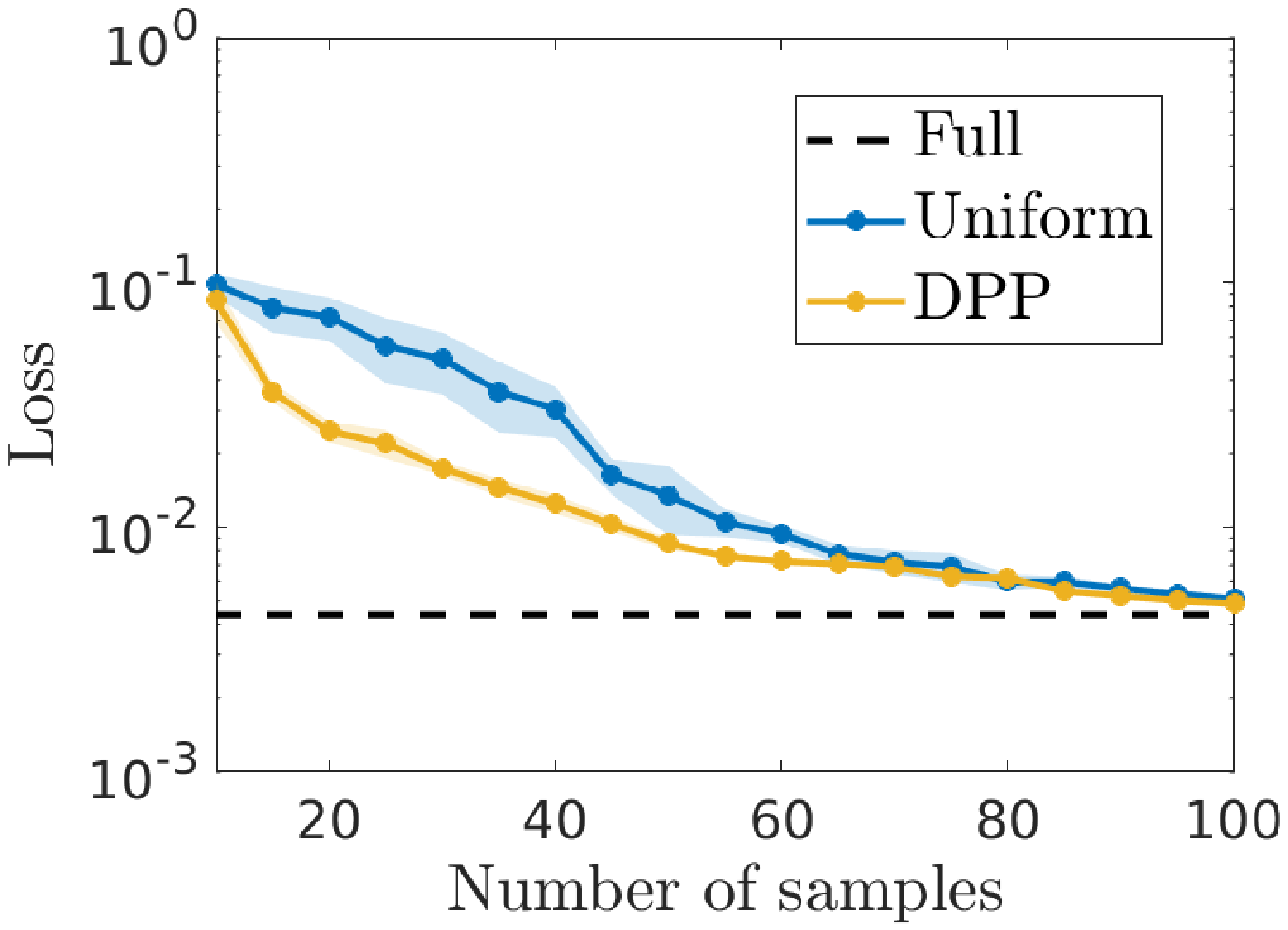}
			\caption{\texttt{System 1}}
			\label{fig:sys_pred_1}
		\end{subfigure}
		\begin{subfigure}[b]{0.32\textwidth}
			\includegraphics[width=\textwidth, height= 0.8\textwidth]{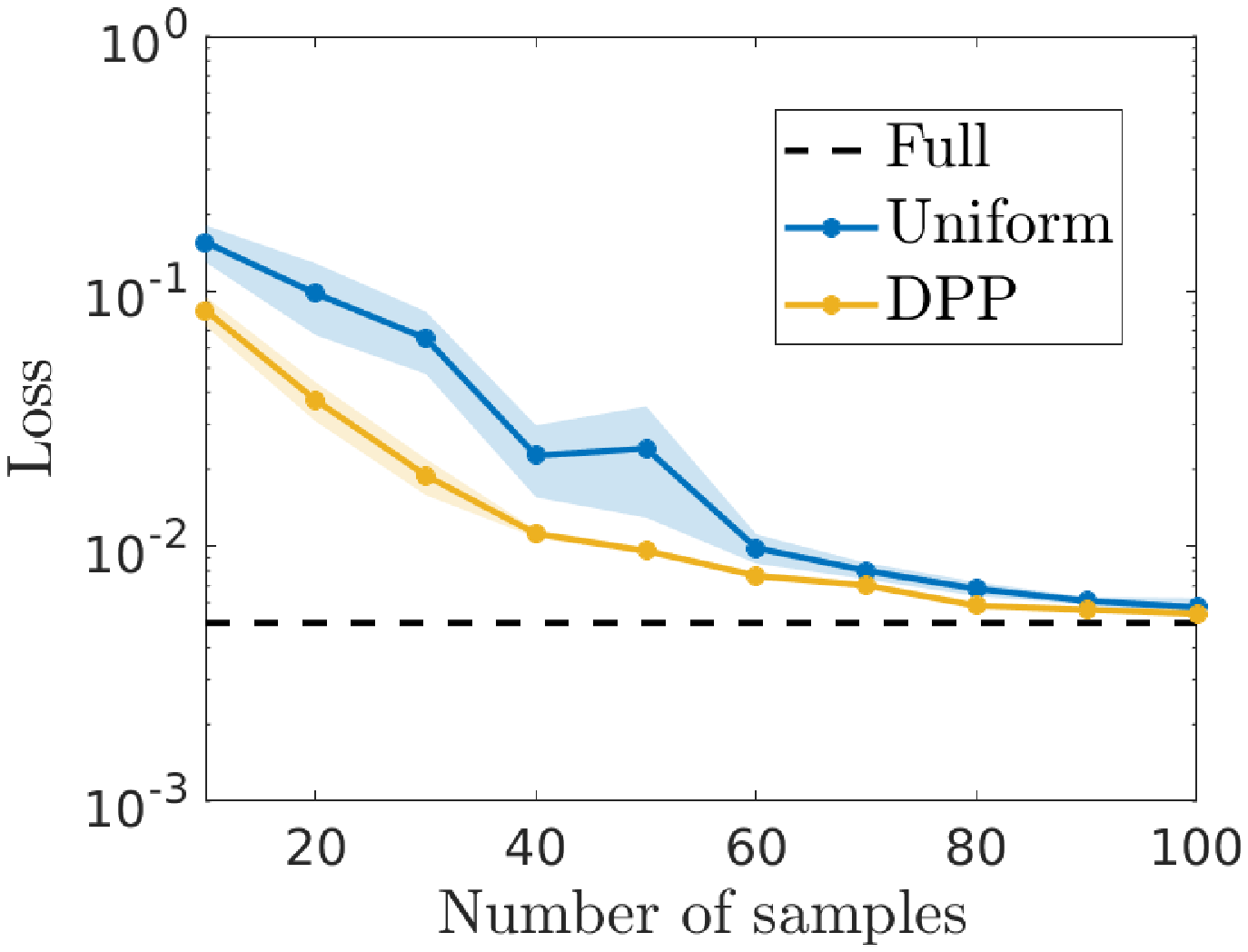}
			\caption{\texttt{System 2}}
			\label{fig:sys_pred_2}
		\end{subfigure}
		\begin{subfigure}[b]{0.32\textwidth}
			\includegraphics[width=\textwidth, height= 0.8\textwidth]{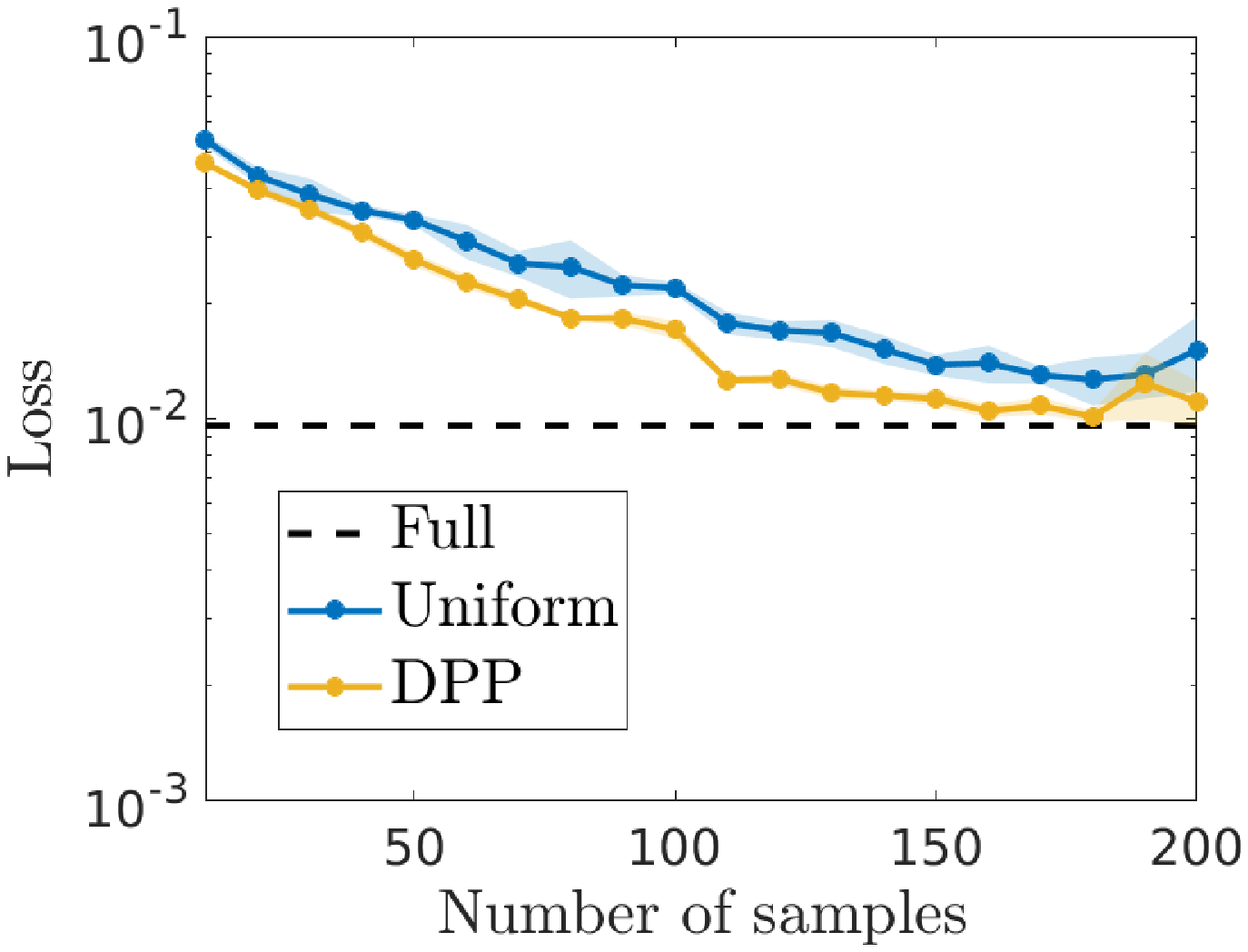}
			\caption{\texttt{System 3}}
			\label{fig:sys_pred_3}
		\end{subfigure}
	\caption{The prediction error as a function of the number landmarks $|\calC|$ using uniform vs extended $L$-ensemble sampling. The total number of training points is $n=500$. The loss is the MSE on the test set.}\label{fig:sys_pred}
\end{figure}
\FloatBarrier
\section{Discussion and conclusion}

Although the extended $L$-ensembles are particularly suited for sampling in the context of semi-parametric regression, the sampling cost in practice might be high if an exact DPP sampling algorithm is used. We acknowledge this difficulty and point the reader towards the recent advances in DPP sampling algorithm which have achieved an improved scalability, especially, if the number of sampled landmarks is not large~\cite{DPP-FX}. Sampling a fixed-size DPP without looking at all items was also studied in~\cite{DPPwithoutLooking}, which provides theoretical guarantees. We expect the methods of~\cite{DPP-FX,DPPwithoutLooking} to be also applicable for partial-projection DPPs, while further approximate DPP sampling algorithms can be developed in the future.
\section*{Acknowledgments}
We are grateful to {Micha\l}  Derezi\'nski for his insightful correspondence about DPPs and implicit regularization. We thank Simon Barthelm\'e for pointing out the references~\cite{Maurer,Lyons}.

\section*{Appendix}
\subsection{Solution of the penalized least-squares regression}
\begin{theorem}[existence, Thm 2.9 in \cite{GuBook}]\label{thm:existence} Suppose $L(f)$ is a continuous and convex functional in a Hilbert space $H$ and $J(f)$ is a square (semi) norm in $H$ with a null space $\mathcal{N}_J$, of finite dimension. If $L(f)$ has a unique minimizer in $\mathcal{N}_J$, then $L(f) + \gamma J(f)$ has a minimizer in $H$.
\end{theorem}
%

\subsection{Useful  results}
\subsubsection{Classical identities}
First, we mention an instrumental result given in~\cite{KuleszaTaskar}. Next, we list a few well-known lemmata.
\begin{lemma}[theorem 2.1 in \cite{KuleszaTaskar}]\label{lemma:KT}
Let $\mathcal{A}\subseteq [n]$ and $M\in\mathbb{R}^{n\times n}$. Let $\mathds{1}_{\bar{\mathcal{A}}}$ be the diagonal matrix with ones in the diagonal positions corresponding to elements of $\bar{\mathcal{A}} = [n]\setminus \mathcal{A}$, and zeros otherwise. Then, it holds that
$
\sum_{\mathcal{C}: \mathcal{A}\subseteq \mathcal{C}}\det M_{\mathcal{C}\mathcal{C}} = \det(M+ \mathds{1}_{\bar{\mathcal{A}}}).
$
\end{lemma}

\begin{lemma}[push-through]\label{lemma:push}
Let $X\in \mathbb{R}^{m\times k}$ and $Y\in \mathbb{R}^{k\times m}$. Then, it holds that
\[
(XY+\mathbb{I}_m)^{-1}X = X(YX + \mathbb{I}_k)^{-1}.
\]
\end{lemma}
\begin{lemma}[Cauchy-Binet]\label{lem:Cauchy-Binet}
Let $X\in \mathbb{R}^{n\times p}$ and $Y\in \mathbb{R}^{n\times p}$ with $n> p$. Then, it holds that 
\[
\det(X^\top Y) = \sum_{\calC\subset [n] : |\calC| = p}  \det(X_\calC^\top Y_\calC),
\]
where $X_\calC$ is the $p\times p$ matrix obtained from $X$ by selecting the rows indexed by $\calC$.
\end{lemma}
\begin{lemma}\label{lem:appli_inverse_block}
Let $A\in \mathbb{R}^n$ and $W\in \mathbb{R}^{n\times p}$ such that $A$ is invertible as well as  $W^\top A^{-1} W$. Then, it holds that
\[
\begin{pmatrix}
A & W\\
W^\top & 0
\end{pmatrix}^{-1} = \begin{pmatrix}
A^{-1} - A^{-1}W(W^\top A^{-1} W)^{-1} W^\top A^{-1} & A^{-1}W(W^\top A^{-1} W)^{-1}\\
(W^\top A^{-1}W)^{-1} W^\top A^{-1} & -(W^\top A^{-1}W)^{-1}
\end{pmatrix}.
\]
\end{lemma}
We also use a slightly different result.
\begin{lemma} \label{lemma:proj_inverse_blocks}
Let $A\in \mathbb{R}^{n\times n}$. Let $W\in \mathbb{R}^{n\times p}$ be a matrix with full column rank. Let $B$ be a matrix with orthonormal columns so that $\mathbb{P}_{W^\perp} = BB^\top$. If $B^\top A B$ is invertible, we have
\[
\begin{pmatrix}
A & W\\
W^\top & 0
\end{pmatrix}^{-1} = \begin{pmatrix}
(\mathbb{P}_{W^\perp}A \mathbb{P}_{W^\perp})^+ & \left(\mathbb{I} -(\mathbb{P}_{W^\perp}A \mathbb{P}_{W^\perp})^+ A\right) W^{+\top} \\
W^{+} \left(\mathbb{I}- A (\mathbb{P}_{W^\perp}A \mathbb{P}_{W^\perp})^+\right)  & - W^{+}\left( A-A (\mathbb{P}_{W^\perp}A \mathbb{P}_{W^\perp})^+ A\right) W^{+\top}
\end{pmatrix},
\]
where $W^+ = (W^\top W)^{-1} W^\top$ and $(\mathbb{P}_{W^\perp}A \mathbb{P}_{W^\perp})^+ = B(B^\top A B)^{-1} B^\top$ .
\end{lemma}
\begin{lemma}\label{lem:pinv}
Let $M$ be a $k\times k$ matrix and let $S$ be a $n\times k$ matrix such that $S^\top S = \mathbb{I}_{k}.$ Then, we have $(S M S^\top)^+ = S M^+ S^\top$.
\end{lemma}
\begin{proof}
The criteria satisfied by the Moore-Penrose pseudo-inverse are readily checked. It holds that
\begin{align*}
    (S M S^\top)(S M S^\top)^+(S M S^\top) &= S M^+ M M^+ S^\top = S M S^\top\\
    (S M S^\top)^+(S M S^\top)(S M S^\top)^+ &= S M M^+ M S^\top = S M^+ S^\top,
\end{align*}
while  the following matrices are Hermitian
\begin{align*}
 (S M S^\top)(S M S^\top)^+ & = S M^+ M S^\top\\
 (S M S^\top)^+(S M S^\top) & = S M M^+ S^\top,
\end{align*}
since $M^+ M$ and $M M^+$ are Hermitian.
\end{proof}
\begin{lemma}[matrix determinant lemma]\label{lem:matrixDetLemma}
Let $A\in \mathbb{R}^{n\times n}$ be an invertible matrix and $\bmu$, $\bm{v}\in \mathbb{R}^n$. Then, it holds that $\det(A+ \bm{u}\bm{v}^\top) = (1+\bm{v}^\top A^{-1} \bmu) \det(A)$.
\end{lemma}
\subsubsection{Lemmata related to the extended $L$-ensemble formalism}

We below provide  helpful formulae to calculate the determinant of matrices with the special structure of extended $L$-ensembles.
\begin{lemma}[Lemma 3.12 in \cite{barthelm2019spectral}]\label{lemma:Bart}
Let  $(M,V)$ is a NNP such as in Definition~\ref{def:NNP}. Let $\widetilde{M} = \mathbb{P}_{V^\perp}M\mathbb{P}_{V^\perp}\in \mathbb{R}^{n\times n}$ and $V = QR\in \mathbb{R}^{n\times p}$ where $Q\in \mathbb{R}^{n\times p}$ has orthonormal columns and $R$ is upper triangular. Then, we have
\[
\det \begin{pmatrix}
\widetilde{M} &  Q\\
Q^\top & 0\\
\end{pmatrix} = (-1)^p  [t^p] \det\left( \widetilde{M}+t Q Q^\top\right),
\]
where $ [t^p]$ denotes the coefficient of the term $t^p$. 
\end{lemma}
\begin{lemma}[lemma 3.11 \cite{barthelm2019spectral}]\label{lemma:Qperp}
Let $M\in\mathbb{R}^{n\times n}$ and $V\in \mathbb{R}^{n\times p}$ such that $(M,V)$ is a NNP. Then, we have
\[
\det \begin{pmatrix}
M &  V\\
V^\top & 0\\
\end{pmatrix} = (-1)^p \det(V^\top V) \det(Q^\top_\perp M Q_\perp),
\]
where $Q_\perp\in \mathbb{R}^{n\times (n-p)}$ has orthonormal columns and is such that $V^\top Q_\perp = 0$.
\end{lemma}
For the sake of completeness, we derive an equivalent expression for the normalization appearing in Definition~\ref{def:ExtendedL}, by extensively using proof techniques of~\cite{barthelm2019spectral,Barthelme2020}.
\begin{proposition}[normalization factors]\label{prop:norm}
Let  $(L,V)$ is a NNP such as in Definition~\ref{def:NNP}. We have the following identity
\[
(-1)^p\det(\mathbb{I}+ \widetilde{L}) \det(V^\top V) = \det \begin{pmatrix}
L + \mathbb{I} &  V\\
V^\top & 0\\
\end{pmatrix}. 
\]
\end{proposition}
\begin{proof}
Let $V = QR$ thanks to the QR-decomposition of $V$, with $Q\in \mathbb{R}^{n\times p}$. Let $Q_\perp\in \mathbb{R}^{n\times (n-p)}$ be a matrix with orthonormal columns so that $\mathbb{P}_{V^\perp} = Q_{\perp}Q_{\perp}^\top$.  Then, $\widetilde{L} = Q_{\perp}Q_{\perp}^\top L Q_{\perp}Q_{\perp}^\top$. This gives the following decomposition
\begin{equation}
     \mathbb{I}+ \widetilde{L} = Q Q^\top + Q_{\perp}\left( \mathbb{I}_{n-p} + Q_{\perp}^\top L Q_{\perp}\right)
Q_{\perp}^\top,\label{eq:1+tildeL}
\end{equation} 
where we used $\mathbb{I} = Q_{\perp}Q_{\perp}^\top +Q Q^\top$.
Define the eigendecomposition 
$
Q_{\perp}^\top L Q_{\perp}+ \mathbb{I}_{n-p} = U \Lambda U^\top,
$
with $U\in \mathbb{R}^{(n-p)\times (n-p)}$ an orthogonal matrix and $\Lambda\in \mathbb{R}^{(n-p)\times (n-p)}$ a diagonal matrix. Therefore, in the light of~\eqref{eq:1+tildeL}, we find 
$\det(\mathbb{I}+ \widetilde{L}) = \det(\Lambda)$. Now, we use the following identity, which results from Lemma 2.6 in~\cite{Barthelme2020},
\begin{equation}
\det \begin{pmatrix}
L + \mathbb{I} &  V\\
V^\top & 0\\
\end{pmatrix} = \det \begin{pmatrix}
\widetilde{L} + \widetilde{\mathbb{I}} &  V\\
V^\top & 0\\
\end{pmatrix},\label{eq:det_tilde}
\end{equation}
with $\widetilde{\mathbb{I}} = \mathbb{P}_{V_\perp}$ and $\widetilde{L} = \mathbb{P}_{V_\perp}L\mathbb{P}_{V_\perp}$.
Next, by using the decomposition of $V= QR$, we find
\begin{equation}
\det \begin{pmatrix}
\widetilde{L} + \widetilde{\mathbb{I}} &  V\\
V^\top & 0\\
\end{pmatrix}= \det(R^\top R)\det \begin{pmatrix}
\widetilde{L} + \widetilde{\mathbb{I}} &  Q\\
Q^\top & 0\\
\end{pmatrix} =  \det(V^\top V)\det \begin{pmatrix}
\widetilde{L} + \widetilde{\mathbb{I}} &  Q\\
Q^\top & 0\\
\end{pmatrix}\label{eq:det_RtR},
\end{equation}
where the first equality uses Lemma~\ref{lemma:Qperp}.
Then, we use Lemma~\ref{lemma:Bart} to express the last factor of the RHS of \eqref{eq:det_RtR},
\begin{equation}
\det \begin{pmatrix}
\widetilde{L} + \widetilde{\mathbb{I}} &  Q\\
Q^\top & 0\\
\end{pmatrix} = (-1)^p  [t^p] \det\left( \widetilde{L} + \widetilde{\mathbb{I}}+t Q Q^\top\right),\label{eq:lemma2}
\end{equation}
where $ [t^p] p(t)$ denotes the coefficient of the term $t^p$ in the polynomial $p(t)$. 
Next, by using the above eigendecomposition, we find
\begin{equation}
[t^p] \det\left( \widetilde{L} + \widetilde{\mathbb{I}}+t Q Q^\top\right) = [t^p] \det\left[
\begin{pmatrix}
Q & U
\end{pmatrix}
\begin{pmatrix}
t\mathbb{I}_{p\times p}  & 0\\
0 & \Lambda\end{pmatrix}
\begin{pmatrix}
Q^\top\\
U^\top
\end{pmatrix}
\right]= \det(\Lambda) = \det(\mathbb{I}+ \widetilde{L}),\label{eq:detI+L}
\end{equation}
where we used that $\begin{pmatrix}
Q & U
\end{pmatrix}\in \mathbb{R}^{n\times n}$ is an orthogonal matrix. Finally, we combine \eqref{eq:det_tilde}, \eqref{eq:det_RtR}, \eqref{eq:lemma2} and \eqref{eq:detI+L} to give  
\[
\det \begin{pmatrix}
L + \mathbb{I} &  V\\
V^\top & 0\\
\end{pmatrix} = \det(V^\top V) (-1)^p \det(\mathbb{I}+ \widetilde{L}),
\]
which is the desired result.
\end{proof}
Below, we give a lemma allowing to simplify several expressions in the paper.
\begin{lemma}[\cite{Barthelme2020}]\label{lem:projC}
Let $(K,V)$ be a NNP. Let $B(\calC)\in \mathbb{R}^{k\times (k-p)}$ be a matrix whose columns are an orthonormal basis of $(V_\calC)^\perp$. Let $\widetildeK = \mathbb{P}_{V_\perp}K\mathbb{P}_{V^\perp}$. Then, it holds that $\mathbb{P}_V C B(\calC) = 0$ and
\[
B^\top(\calC) K_{\calC\calC}B(\calC)= B^\top(\calC) \widetildeK_{\calC\calC}B(\calC).
\]
\end{lemma}
\begin{proof}
By definition, $\mathbb{P}_{V^\perp} = \mathbb{I}- QQ^\top$ where $Q$ has orthonormal columns and is obtained thanks to the QR-decomposition  $V = Q R$. Then, we find $B^\top(\calC)V_\calC = 0 = B^\top(\calC)Q_\calC$. Therefore, we obtain $B^\top(\calC) \widetilde{K}_{\calC\calC}B(\calC) = B^\top(\calC) K_{\calC\calC}B(\calC)$ by using
$
\widetilde{K}_{\calC\calC} = (C^\top - Q_\calC Q^\top ) K (C-Q Q^\top_\calC),
$
and $B^\top(\calC)Q_\calC = 0$.
\end{proof}
\subsection{Proof of Theorem~\ref{thm:implicit_reg} \label{sec:deferred}}
 Let $\bm{u},\bm{v}\in \mathbb{R}^{n+p}$. 
First, we first notice that
 \[
 \begin{pmatrix}
C^\top K C & C^\top V\\
V^\top C & 0\\
\end{pmatrix} =
\begin{pmatrix}
C^\top  & 0\\
0 & \mathbb{I}\\
\end{pmatrix}
\begin{pmatrix}
 K  &  V\\
V^\top & 0\\
\end{pmatrix}
\begin{pmatrix}
C  & 0\\
0 & \mathbb{I}\\
\end{pmatrix}= \widetilde{C}^\top \begin{pmatrix}
 K  &  V\\
V^\top & 0\\
\end{pmatrix} \widetilde{C},
\]
where $\widetilde{C}$ is a sampling matrix corresponding to a subset of $\{1,\dots,n+p\}$, i.e., $\widetilde{C}$ is associated to the set $\widetilde{\mathcal{C}} = \mathcal{C}\cup \mathcal{A}$ with $\mathcal{A} = \{n+1,\dots, n+p\}$. Therefore, by using Lemma~\ref{lemma:KT}, we have the following identity
$\sum_{\widetilde{\mathcal{C}}: \mathcal{A}\subseteq \widetilde{\mathcal{C}}}\det Q_{\widetilde{\mathcal{C}}\widetilde{\mathcal{C}}} = \det(Q+ \mathds{1}_{\bar{\mathcal{A}}})$ for all  $Q\in\mathbb{R}^{(n+p)\times (n+p)}$. In particular, for $ Q = \bigl(\begin{smallmatrix}
K  &  V\\
V^\top & 0\\
\end{smallmatrix}\bigr)$, this gives 
\begin{equation}
 \sum_{\mathcal{C}\subseteq [n]} \det \begin{pmatrix}
C^\top K C & C^\top V\\
V^\top C & 0\\
\end{pmatrix} = \det \begin{pmatrix}
K + \mathbb{I} &  V\\
V^\top & 0\\
\end{pmatrix}.\label{eq:reg}
\end{equation}
Then, the desired expectation can be written as follows
\[
\frac{1}{N}\sum_{\calC \subseteq [n]}\det (\widetildeC^\top Q \widetildeC) \times \bm{u}^\top \widetildeC \left(\widetildeC^\top Q \widetildeC\right)^{-1} \widetildeC^\top \bm{v}  = \frac{1}{N} \sum_{\calC\subseteq [n]}\det(\widetildeC^\top Q \widetildeC) - \frac{1}{N}\sum_{\calC \subseteq [n]}\det(\widetildeC^\top (Q- \bm{v} \bm{u}^\top) \widetildeC),
\]
thanks to the matrix determinant lemma (Lemma~\ref{lem:matrixDetLemma}), and where the normalization $N$ is given by~\eqref{eq:reg}. Define for simplicity $T_1 \triangleq \sum_{\calC\subseteq [n]}\det(\widetildeC^\top Q \widetildeC)$ and $T_2 \triangleq \sum_{\calC \subseteq [n]}\det(\widetildeC^\top (Q- \bm{v} \bm{u}^\top) \widetildeC)$.
Then, we find 
\[
T_1 =\det \begin{pmatrix}
K + \mathbb{I} &  V\\
V^\top & 0\\
\end{pmatrix} = \det(Q + \mathds{1}_{[n]}) = N,
\]
where $\mathds{1}_{\mathcal{B}}$ the diagonal matrix with ones in the diagonal positions corresponding to elements of the set $\mathcal{B}$, and zeros otherwise.
Let $\bm{u} = [\bm{u}_0,  \bm{u}_1]^\top$ and $\bm{v} = [\bm{v}_0,  \bm{v}_1]^\top$, with $\bm{u}_0, \bm{v}_0 \in \mathbb{R}^n$ and $\bm{u}_1, \bm{v}_1 \in \mathbb{R}^p$. Similarly, we have also
\begin{align*}
T_2 = \sum_{\calC \subseteq [n]}\det\begin{pmatrix}
C^\top (K-\bm{v}_0\bm{u}_0^\top) C & C^\top (V-\bm{v}_0\bm{u}_1^\top)\\
(V^\top-\bm{v}_1\bm{u}_0^\top) C & -\bm{v}_1\bm{u}_1^\top\\
\end{pmatrix} &= \det\begin{pmatrix}
(K + \mathbb{I}-\bm{v}_0\bm{u}_0^\top)  &  V-\bm{v}_0\bm{u}_1^\top\\
V^\top-\bm{v}_1\bm{u}_0^\top & -\bm{v}_1\bm{u}_1^\top\\
\end{pmatrix}\\
& = \det(Q + \mathds{1}_{[n]} - \bm{v} \bm{u}^\top).
\end{align*}

Hence, we obtain another expression for the desired expectation
\begin{align*}
\mathbb{E}_{\calC\sim DPP(K,V)}\left[ \begin{pmatrix}
\bm{u}_{0,\calC}\\
\bm{u}_1
\end{pmatrix}^\top\begin{pmatrix}
C^\top K C & C^\top V\\
V^\top C & 0\\
\end{pmatrix}^{-1} \begin{pmatrix}
\bm{v}_{0,\calC}\\
\bm{v}_1
\end{pmatrix}\right]&=\frac{\det(Q+ \mathds{1}_{[n]})-\det(Q + \mathds{1}_{[n]} - \bm{v} \bm{u}^\top)}{\det(Q + \mathds{1}_{[n]})} \\
&= \bm{u}^\top (Q+ \mathds{1}_{[n]})^{-1} \bm{v} ,
\end{align*}
where the last equality uses the matrix determinant lemma. This completes the proof.

\subsection{Details of the derivation of the large scale system of Section~\ref{sec:Nyst_penreg}}
After elementary manipulations, we obtain the following system
\begin{align*}
    &\left(\mathbb{P}_{V_\calC^\perp}K_\calC \mathbb{P}_{V^\perp} K_\calC^\top \mathbb{P}_{V_\calC^\perp} + n\gamma \mathbb{P}_{V_\calC^\perp} K_{\calC\calC} \mathbb{P}_{V_\calC^\perp}\right)\bma' = \mathbb{P}_{V_\calC^\perp}K_\calC \mathbb{P}_{V^\perp} \bm{y}\label{eq:TBC}\\
    &\bmb' = (V^\top V)^{-1} V^\top \left( \bm{y} - K_\calC^\top \mathbb{P}_{V_\calC^\perp}\bma'\right)\nonumber\\
    &\mathbb{P}_{V_\calC}\bma' = 0,\nonumber
\end{align*}
which yields the system in~\eqref{eq:LinSyst}, as we show below.
Let $B(\calC)\in \mathbb{R}^{k\times (k-p)}$ be a matrix whose columns are an orthonormal basis of $(V_\calC)^\perp$, and that is such that $\mathbb{P}_{V_\calC^\perp} =B(\calC) B^\top(\calC)$. Define $\Xi = B^\top(\calC) K_{\calC\calC} B(\calC)$, which is non-singular almost surely, as shown in the proof of Proposition~\ref{prop:ExtendedExpectedpinv}. Then, the first equation of the system yields
\begin{align*}
\bma' &= B(\calC)\left( B^\top(\calC)K_\calC \mathbb{P}_{V^\perp} K_\calC^\top B(\calC)  + n\gamma B^\top(\calC) K_{\calC\calC} B(\calC) \right)^{-1}B^\top(\calC) K_\calC\mathbb{P}_{V^\perp} \bm{y}\\
& = B(\calC)\Xi^{-1/2}\left( \Xi^{-1/2}B^\top(\calC)K_\calC \mathbb{P}_{V^\perp} K_\calC^\top B(\calC)\Xi^{-1/2}  + n\gamma \mathbb{I}_{k-p} \right)^{-1}\Xi^{-1/2}B^\top(\calC) K_\calC\mathbb{P}_{V^\perp} \bm{y}\\
&= B(\calC)\Xi^{-1}B^\top(\calC) K_\calC\mathbb{P}_{V^\perp} \left(\mathbb{P}_{V^\perp}K_\calC^\top B(\calC)\Xi^{-1}B^\top(\calC)K_\calC\mathbb{P}_{V^\perp} + n\gamma \mathbb{I}_n\right)^{-1} \mathbb{P}_{V^\perp}\bm{y},
\end{align*}
where we used the push-through identity for the last equality (Lemma~\ref{lemma:push}) $(XY+\mathbb{I})^{-1}X = X(YX + \mathbb{I})^{-1}$ with $X = \Xi^{-1/2}B^\top(\calC) K_\calC\mathbb{P}_{V^\perp}$, $Y = X^\top$ and $\mathbb{P}_{V^\perp}^2 = \mathbb{P}_{V^\perp}$.
As detailed above in Lemma~\ref{lem:projC}, an equivalent expression for $\Xi = B^\top(\calC) K_{\calC\calC} B(\calC)$ is $\Xi = B^\top(\calC) \widetilde{K}_{\calC\calC} B(\calC)$.
The in-sample estimator $\hat{\bm{z}}_{N} = K^\top_\calC \bma'^\star + V^\top \bmb'^\star$ is then given, in terms of the projected Nystr\"om approximation, as follows
\[
\hat{\bm{z}}_{N} = \widetilde{L(\calC)} \left(\widetilde{L(\calC)}+n\gamma \mathbb{I}_n\right)^{-1}\mathbb{P}_{V^\perp}\bm{y} + \mathbb{P}_{V}\bm{y},
\]
which is the result stated in Section~\ref{sec:Nyst_penreg}.
\paragraph{Preconditioning}
Consider the linear system in~\eqref{eq:LinSyst} and notice that the largest eigenvalue of the matrix 
\[
B^\top(\calC)K_\calC \mathbb{P}_{V^\perp} K_\calC^\top B(\calC) = B^\top(\calC)C^\top \widetildeK^2 C B(\calC),
\]
in the first term can be possibly numerically large, since it involves $\widetildeK^2$.  Therefore, the linear system might be ill-conditioned. A preconditioning  may improve the convergence of a linear solver such as the conjugate gradient method. We define the preconditioner as follows.
Denote the marginal probabilities of $DPP(K/\lambda,V)$ by 
\begin{equation}
    \label{eq:weightsLS}
    \bm{\ell}= \diag\left(\mathbb{P}_V + \widetildeK (\widetildeK + \lambda\mathbb{I})^{-1}\right),
\end{equation}
as given by~\eqref{eq:marginalKernel}, and define the diagonal matrix $D = \Diag(\bm{\ell})^{-1}$.
Then, it simply holds that 
$
\mathbb{E}_\calC[CD_{\calC\calC}C^\top] = \mathbb{I},
$
since the marginal probability is $\bm{\ell}_i = \Pr(i\in \mathcal{Y})$ for $\mathcal{Y}\sim DPP(K/\lambda,L)$. This remark motivates a preconditioning of the above linear system by approximating $\widetildeK^2$ by $\widetildeK C D_{\calC\calC}C^\top\widetildeK $. Let $H$ be a matrix obtained by the following Cholevski decomposition
\[
HH^\top = \left( B^\top(\calC)\widetildeK_{\calC\calC} D_{\calC\calC}  \widetildeK_{\calC\calC} B(\calC)  + n\gamma B^\top(\calC) \widetildeK_{\calC\calC} B(\calC) \right)^{-1}.
\]
The equivalent resulting linear system
\begin{align}
\label{eq:precSystem}
\bma'^\star &=  B(\calC) H \Big(H^\top B^\top(\calC)K_\calC \mathbb{P}_{V^\perp} K_\calC^\top B(\calC)H  + n\gamma H^\top B^\top(\calC) K_{\calC\calC} B(\calC) H\Big)^{-1}H^\top B^\top(\calC) K_\calC\mathbb{P}_{V^\perp} \bm{y},
\end{align}
 is likely to have a smaller condition number. This type of conditioning was studied in the case of kernel ridge regression in~\cite{BLESS}. 
  \begin{figure}[ht]
		\centering
		\begin{subfigure}[t]{0.24\textwidth}
			\includegraphics[width=\textwidth, height= 0.91\textwidth]{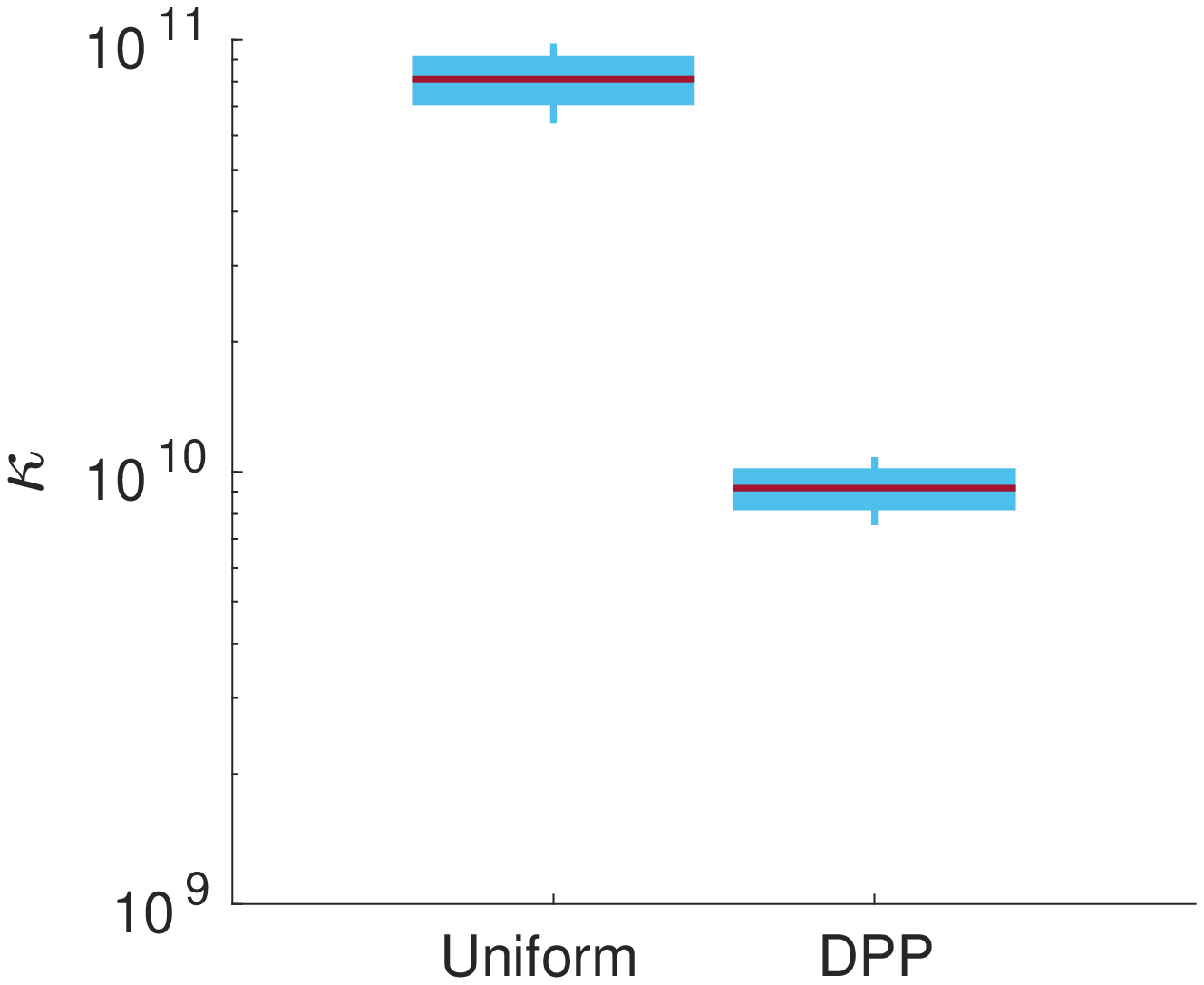}
			\caption{\texttt{Parkinson}}
		\end{subfigure}
		\begin{subfigure}[t]{0.24\textwidth}
			\includegraphics[width=\textwidth, height= 0.91\textwidth]{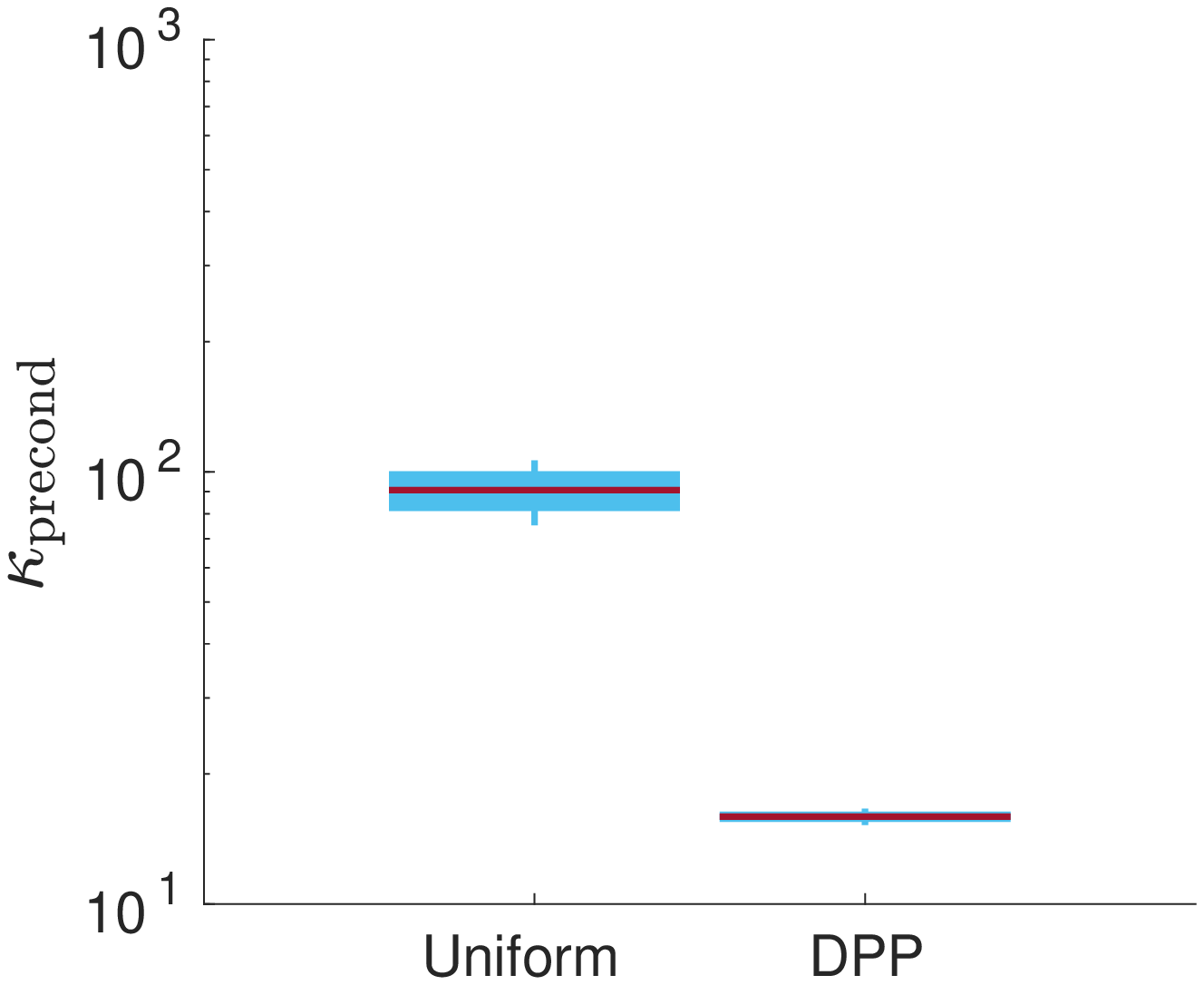}
			\caption{\texttt{Parkinson}}
		\end{subfigure}
		\begin{subfigure}[t]{0.24\textwidth}
			\includegraphics[width=\textwidth, height= 0.91\textwidth]{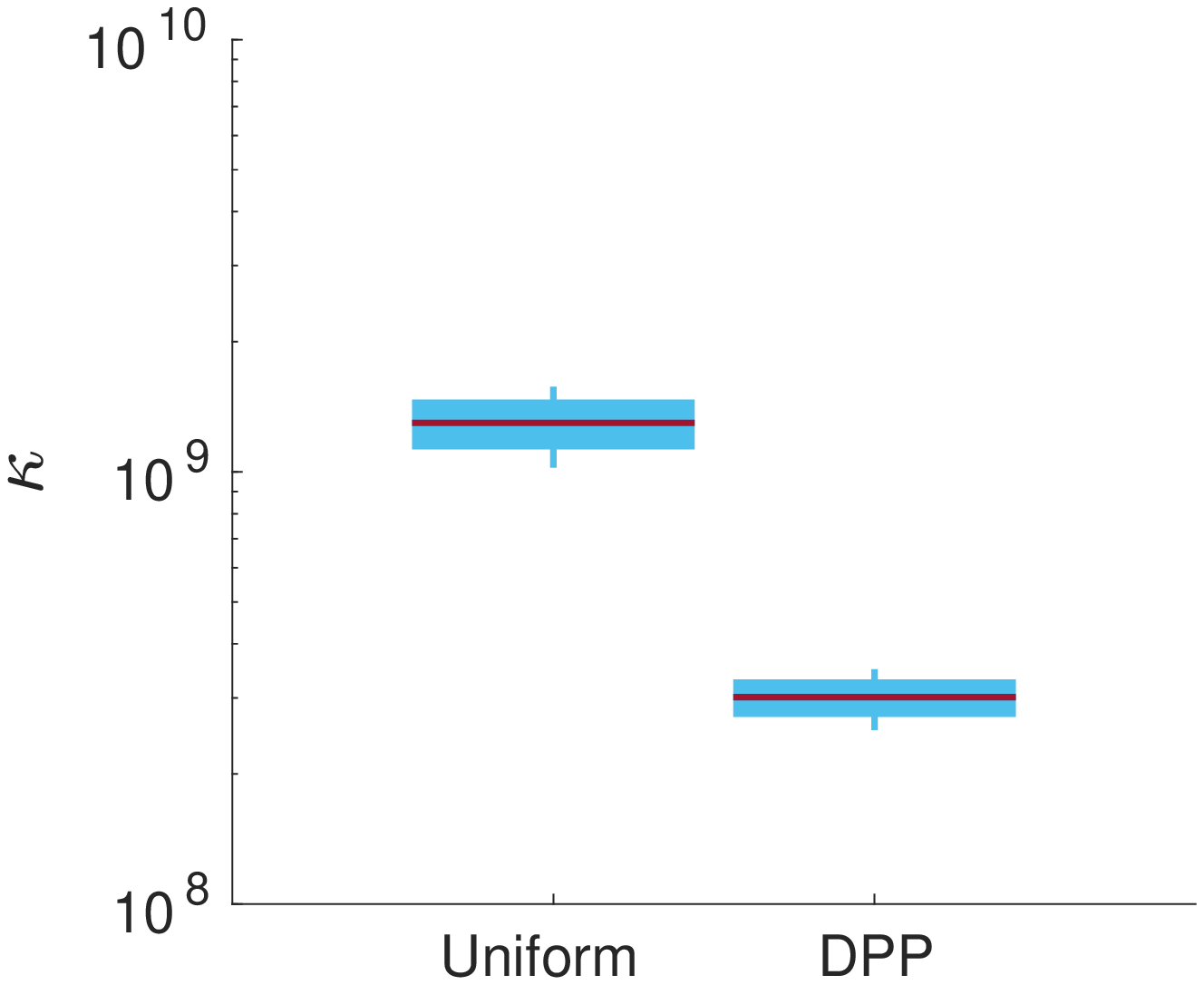}
			\caption{\texttt{Pumadyn8FM}}
		\end{subfigure}	
		\begin{subfigure}[t]{0.24\textwidth}
			\includegraphics[width=\textwidth, height= 0.91\textwidth]{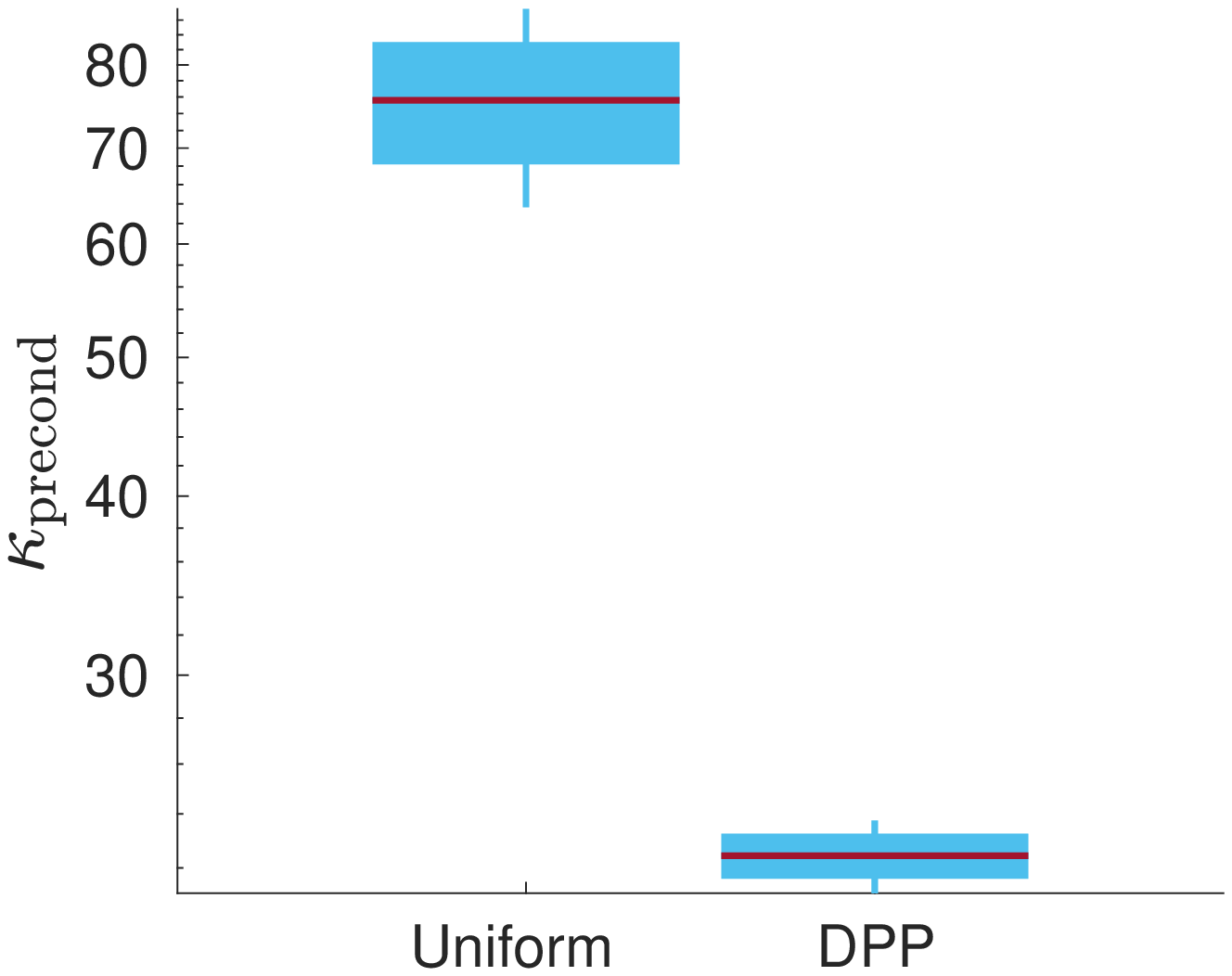}
			\caption{\texttt{Pumadyn8FM}}
		\end{subfigure}			
		\caption{Preconditioning results. The condition number of the linear system before ($\kappa$) and after the preconditioning ($\kappa_{\text{precond}}$) is plotted for uniform  and DPP sampling. From left to right, the condition number before and after preconditioning, for \texttt{Parkinson} and \texttt{Pumadyn8FM} data sets, respectively.}\label{fig:Preconditioner}
\end{figure}

 For convenience, we illustrate the use of the preconditioner on the UCI benchmark data sets \texttt{Parkinson}, and \texttt{Pumadyn8FM}. A Gaussian kernel with $\sigma = 5$ and linear regression component is used after standardizing the data sets: $V = [X \enskip \bm{1}_n]$ where $X = [\bmx_1 \dots \bmx_n]^\top\in \mathbb{R}^{n\times d}$. We compare the condition number of the linear system in~\eqref{eq:LinSyst} with the preconditioned system using the preconditioner given in~\eqref{eq:precSystem}.  For the uniform sampling method, we use  $D = (n/|\calC|)\diag(\bm{1}_n)$ in the preconditioner formula as in~\cite{BLESS}.  The ridge regularization parameter for the linear system  as well as the regularization parameter of the DPP are equal to $\lambda = \gamma = 10^{-6}$ for simplicity. The number of samples is equal to the effective dimensionality: $\sum_i \bm{\ell}_i$, which is $1325$ and $636$ for the \texttt{Parkinson}, and  \texttt{Pumadyn8FM} data sets respectively. The experiment is repeated 10 times. From the results in Figure~\ref{fig:Preconditioner}, we empirically see that using the proposed preconditioner in combination with the DPP sampling procedure, results in a smaller condition number of the linear system obtained from~\eqref{eq:pen_reg_Nys}.

\FloatBarrier

\bibliographystyle{siamplain}
\bibliography{References}
\end{nolinenumbers}
\end{document}